\documentclass[12pt]{article}
\usepackage[margin=2cm]{geometry}
\usepackage{amsmath, amssymb,amsfonts,bbm,verbatim,multirow,color,enumerate,url}
\usepackage{epic,eepic,psfrag,epsfig}
\usepackage[sf,bf,SF,footnotesize]{subfigure}
\usepackage{graphicx}
\usepackage{amsthm,breakcites}
\usepackage{amscd}
\usepackage{epsfig}
\usepackage{natbib}
\usepackage{array}
\usepackage[small]{caption}
\RequirePackage[colorlinks,citecolor=blue,urlcolor=blue,breaklinks]{hyperref}

\usepackage{xr}

\newcommand{\totalVariation}{\mathrm{TV}}
\newcommand{\kullbackLeiblerDivergence}{\mathrm{KL}}

\newcommand{\N}{\mathbb{N}}

\newcommand{\R}{\mathbb{R}}

\newcommand{\E}{\mathbb{E}}

\DeclareMathOperator\tr{tr}

\newcommand{\ambientDimension}{d}
\newcommand{\marginExponent}{\alpha}
\newcommand{\marginConstant}{C_{\mathrm{M}}}
\newcommand{\holderExponent}{\beta}
\newcommand{\holderConstant}{C_{\mathrm{S}}}

\newcommand{\Lebesgue}{\mathcal{L}_d}

\newcommand{\sourceDimension}{d_{\sourceDistribution}}
\newcommand{\targetDimension}{d_{\targetDistribution}}
\newcommand{\sourceTailExponent}{\gamma_{\sourceDistribution}}
\newcommand{\targetTailExponent}{\gamma_{\targetDistribution}}
\newcommand{\TailConstant}{C_{\sourceDistribution,\targetDistribution}}

\newcommand{\sourceDensityFunction}{\omega_{\sourceMarginalDistribution,\sourceDimension}}

\newcommand{\thresholdPartitions}{\mathbb{T}}

\newcommand{\expansivityConstant}{\phi}

\newcommand{\J}{\mathcal{J}}
\newcommand{\diam}{\mathrm{diam}}
\newcommand{\distX}{\mathrm{dist}_{\X}}
\newcommand{\firstTerm}{A}
\newcommand{\secondTerm}{B}

\newcommand{\one}{\mathbbm{1}}

\theoremstyle{plain} 
\newtheorem{theorem}{Theorem} 
\newtheorem{proposition}[theorem]{Proposition}
\newtheorem{assumption}{Assumption}
\newtheorem*{proposition*}{Proposition}
\newtheorem{lemma}[theorem]{Lemma}
\newtheorem*{lemma*}{Lemma}
\newtheorem{corollary}[theorem]{Corollary}

\newtheorem{definition}{Definition}
\newtheorem*{definition*}{Definition}

\theoremstyle{definition} 
\newtheorem{remark}{Remark}
\newtheorem*{remark*}{Remark}

\newtheorem*{aside*}{Aside}
\newtheorem{example}{Example}
\newtheorem*{example*}{Example}

\newcommand{\vertiii}[1]{{\left\vert\kern-0.25ex\left\vert\kern-0.25ex\left\vert #1 
    \right\vert\kern-0.25ex\right\vert\kern-0.25ex\right\vert}}
\newcommand{\risk}{\mathcal{R}}
\newcommand{\excessRisk}{\mathcal{E}}

\newcommand{\classifier}{f}

\newcommand{\X}{\mathcal{X}}

\newcommand{\sample}{\mathcal{D}}
\newcommand{\sourceSample}{\sample_{\sourceDistribution}}
\newcommand{\targetSample}{\sample_{\targetDistribution}}

\newcommand{\Prob}{\mathbb{P}}

\newcommand{\sourceDistribution}{P}
\newcommand{\targetDistribution}{Q}

\newcommand{\regressionFunction}{\eta} 
\newcommand{\sourceRegressionFunction}{\regressionFunction_{\sourceDistribution}}
\newcommand{\targetRegressionFunction}{\regressionFunction_{\targetDistribution}}

\newcommand{\marginalDistribution}{\mu}

\newcommand{\sourceMarginalDistribution}{\marginalDistribution_{\sourceDistribution}}
\newcommand{\targetMarginalDistribution}{\marginalDistribution_{\targetDistribution}}

\newcommand{\numSource}{n_{\sourceDistribution}}
\newcommand{\numTarget}{n_{\targetDistribution}}

\newcommand{\numTargetHalf}{\lfloor \numTarget/2 \rfloor}

\newcommand{\setOfDecisionTreeFunctionsL}{\mathcal{H}_L}
\newcommand{\setOfDecisionTreeFunctionsLStar}{\mathcal{H}_{L^*}}
\newcommand{\decisionTreeFunction}{h}
\newcommand{\empiricalMargin}{\hat{m}}

\newcommand{\robustness}{\sigma}

\newcommand{\lepskiChoiceOfKSource}{\hat{k}_{\robustness,\decisionTreeFunction}^{\sourceDistribution}}

\newcommand{\lepskiChoiceOfKTarget}{\tilde{k}_{\robustness}^{\targetDistribution}}
\newcommand{\empiricalMarginSourceDecisionTreeK}{\empiricalMargin_{k,\decisionTreeFunction}^{\sourceDistribution}}

\newcommand{\empiricalMarginTargetK}{\empiricalMargin^{\targetDistribution}_{k}}

\newcommand{\potentialFunctionSet}{\hat{\mathcal{F}}}

\newcommand{\posteriorDriftFunction}{g}

\DeclareMathOperator*{\argmin}{argmin}

\newcommand{\dist}{\rho}

\newcommand{\classifierEst}{\hat{\classifier}}
\newcommand{\bayesClassifier}{\classifier_{\targetDistribution}^*}

\newcommand{\supp}{\mathrm{supp}}

\newcommand{\setOfClassifiers}{\mathcal{F}}
\newcommand{\setofDDClassifiers}{\hat{\setOfClassifiers}_{\numSource,\numTarget}}

\newcommand{\neighboursAreCloseEventNoDelta}{E_1}
\newcommand{\neighboursAreCloseEventDelta}{\neighboursAreCloseEventNoDelta^{\delta}(x)}
\newcommand{\hoeffdingTypeEventNoDelta}{E_2}
\newcommand{\hoeffdingTypeEventDelta}{\hoeffdingTypeEventNoDelta^{\delta}(x)}

\title{Adaptive transfer learning}
\author{Henry W. J. Reeve, Timothy I. Cannings and Richard J. Samworth\\
University of Bristol, University of Edinburgh \\ 
and University of Cambridge}
\date{}

\begin{document}

\maketitle

\begin{abstract}
In transfer learning, we wish to make inference about a target population when we have access to data both from the distribution itself, and from a different but related source distribution.  We introduce a flexible framework for transfer learning in the context of binary classification, allowing for covariate-dependent relationships between the source and target distributions that are not required to preserve the Bayes decision boundary.   Our main contributions are to derive the minimax optimal rates of convergence (up to poly-logarithmic factors) in this problem, and show that the optimal rate can be achieved by an algorithm that adapts to key aspects of the unknown transfer relationship, as well as the smoothness and tail parameters of our distributional classes.  This optimal rate turns out to have several regimes, depending on the interplay between the relative sample sizes and the strength of the transfer relationship, and our algorithm achieves optimality by careful, decision tree-based calibration of local nearest-neighbour procedures.
\end{abstract}

\section{Introduction}

Transfer learning refers to statistical problems in which we wish to make inference about a test data population, but where some (typically, the large majority) of our training data come from a related but distinct distribution.   Such problems arise in many natural, practical settings: for instance, we may wish to understand the effectiveness of a treatment on a particular subgroup of a population, but still wish to exploit information about its efficacy on the wider population under study.  In medical applications, we may be interested in making predictions in a given experimental setting, or using a particular piece of equipment, but also have data obtained under different scenarios or measured with different devices.  Closely related problems have recently been of interest to many communities, sometimes studied under the banner of label noise \citep{Frenay1,blanchard2016classification,cannings2020}, multi-task learning \citep{caruana1997multitask,maurer2016benefit} or distributional robustness \citep{certifying2018sinha,weichwald2020distributional,christiansen2020difficult}.  For recent survey papers on transfer learning, see \citet{pan2009survey}, \citet{storkey2009training} and \citet{weiss2016survey}.

We focus here on transfer learning in the context of binary classification, both due to the latter's fundamental importance as a canonical problem in modern statistics and machine learning, and because, as we shall see, its structure is particularly amenable to algorithms that seek to exploit relationships between the training and test distributions.  To set the scene for our contributions, let $\sourceDistribution$ and $\targetDistribution$ denote two distributions on $\mathbb{R}^d \times \{0,1\}$, with corresponding generic random pairs $(X^{\sourceDistribution},  Y^{\sourceDistribution})$ and $(X^{\targetDistribution},  Y^{\targetDistribution})$ respectively.  We think of $\sourceDistribution$ as a source distribution, from which most of our training data are generated, and $\targetDistribution$ as a target distribution, from which we may have some training data, and about which we wish to make inference.  Let $\sourceRegressionFunction, \targetRegressionFunction:\mathbb{R}^d\rightarrow [0,1]$ denote the source and target regression functions respectively, defined by 
\begin{equation}
\label{Eq:RegressionFunctions}
\sourceRegressionFunction(x) :=\Prob(Y^{\sourceDistribution}=1|X^{\sourceDistribution}=x) \quad \text{and} \quad\targetRegressionFunction(x) :=\Prob(Y^{\targetDistribution}=1|X^{\targetDistribution}=x).
\end{equation}
Our main working assumption on the relationship between $\sourceDistribution$ and $\targetDistribution$ will be that our feature space $\R^d$ can be partitioned into finitely many cells $\X_\ell^*$, and for each cell there exists a \emph{transfer function} $g_\ell:[0,1] \rightarrow [0,1]$ such that $\sourceRegressionFunction$ can be approximated by $\posteriorDriftFunction_\ell \circ \targetRegressionFunction$ on $\X_\ell^*$.  We will further assume that the cells arise from a decision tree partition \citep{breiman1984classification}, and that each transfer function satisfies
\begin{equation}
\label{Eq:gCond}
\frac{\posteriorDriftFunction_\ell(z) - \posteriorDriftFunction_\ell(1/2)}{z - 1/2} \geq \expansivityConstant
\end{equation}
for some $\expansivityConstant > 0$, and all $z \in [0,1/2) \cup (1/2,1]$.  

Thus, in the simplest case where we have just a single cell, $\sourceDistribution$ and $\targetDistribution$ are connected via the fact that the propensity under the source distribution to have a Class 1 label at $x \in \mathbb{R}^d$ only depends on $x$ through $\targetRegressionFunction(x)$, which reflects the propensity under the target distribution to have a Class~1 label at $x$.  Our condition~\eqref{Eq:gCond} is of course satisfied if each $\posteriorDriftFunction_\ell$ is differentiable with $\posteriorDriftFunction_\ell'(z) \geq \expansivityConstant$ for $z \in [0,1]$, and ensures in particular that when $\sourceRegressionFunction = \posteriorDriftFunction_\ell \circ \targetRegressionFunction$ holds exactly on $\X_\ell^*$, we have $\mathrm{sgn}\bigl\{\sourceRegressionFunction(x) - \posteriorDriftFunction_\ell(1/2)\bigr\} = \mathrm{sgn}\bigl\{\targetRegressionFunction(x) - 1/2\bigr\}$ on that cell.  Importantly, though, condition~\eqref{Eq:gCond} does not require that $\posteriorDriftFunction_\ell(1/2) = 1/2$.  

To give an example where such a relationship between $\sourceDistribution$ and $\targetDistribution$ might be expected, suppose that we wish to predict At Risk individuals for a disease (e.g.~breast cancer), on the basis of a set of covariates $x$.  Due to the difficulties and expense of large-scale testing, only a small number of individuals in the general population are assessed (e.g.~via a mammogram), but those displaying symptoms have a much greater propensity to be tested.  In this example, we think of our (large) data set from $\sourceDistribution$ as being a set of individuals for whom we have recorded relevant covariates, and for whom we record a label $Y^\sourceDistribution = 1$ if and only if the individual has both been tested, and has been assessed to be At Risk as a result.  On the other hand, our main interest is in whether individuals are At Risk, regardless of whether or not they have been tested.  Our (small) data set from~$\targetDistribution$, then, is obtained by testing a number of uniformly randomly-chosen individuals from the general population, and we record $Y^\targetDistribution = 1$ if and only if the individual is assessed to be At Risk.  We can think of our training data from both $P$ and $Q$ as being generated from independent and identically distributed triples $(T,X,Y)$, where $T$ is a binary indicator of whether or not a test has been conducted before the start of the study, where $X$ encodes covariates, and where~$Y$ indicates whether or not an individual is At Risk.  However, in our source sample, we only observe $X^\sourceDistribution = X$ and $Y^\sourceDistribution = TY$, while in our target sample, we see $X^\targetDistribution = X$ and $Y^\targetDistribution = Y$.  Thus, in this example, the marginal distributions of $X^\sourceDistribution$ and $X^\targetDistribution$ are the same, while the regression functions $\sourceRegressionFunction$ and $\targetRegressionFunction$ satisfy $\targetRegressionFunction \geq \sourceRegressionFunction$.  In fact, in this formulation, we have
\[
\sourceRegressionFunction(x) = \mathbb{P}(TY = 1|X = x) = \mathbb{P}(T = 1|X = x,Y=1)\targetRegressionFunction(x).
\]
The relationship $\eta_P = g \circ \eta_Q$ then holds if $T$ and $(X,Y)$ are conditionally independent given $\eta_Q(X)$.  More generally in this example, we might construct a decision tree partition based on geographical location and income, for instance, and ask only that this relationship hold approximately for each cell of the partition.  

As another example, in tax fraud detection, most individuals can only be subjected to a simple screening procedure due to the administrative burden.  Hence, in order to assess the reliability of their detection algorithms, a government agency might draw a separate, smaller sample of individuals, chosen uniformly at random from the population, for a more formal audit.  Here, $Y^{\sourceDistribution} = 1$ if the screening flags a potentially fraudulent return,  $Y^{\targetDistribution}=1$ if the audit detects fraud, and $X^{\sourceDistribution} = X^{\targetDistribution} = X$ encodes covariates.  Since
\[
\sourceRegressionFunction(x) = \mathbb{P}(Y^{\sourceDistribution} = 1|X=x,Y^{\targetDistribution}=1)\targetRegressionFunction(x) + \mathbb{P}(Y^{\sourceDistribution} = 1|X=x,Y^{\targetDistribution}=0)\bigl\{1 - \targetRegressionFunction(x)\bigr\},
\]
the modelling assumption $\eta_P = g \circ \eta_Q$ holds if the conditional probabilities above only depend on $x$ through $\targetRegressionFunction(x)$.  In practice, there may be additional dependencies, e.g.~based on profession, income bracket and domicile status, but the modelling relationship may still hold approximately on the cells of a suitable decision tree partition.  Further examples may be found in computer vision, precision medicine, natural language processing and many other areas.  

In line with the above examples, then, we will assume a transfer learning setting with independent data $\mathcal{D}_\sourceDistribution := \bigl((X_1^{\sourceDistribution},Y_1^{\sourceDistribution}),\ldots,(X_{\numSource}^{\sourceDistribution},Y_{\numSource}^{\sourceDistribution})\bigr)$ from~$\sourceDistribution$ and
$\mathcal{D}_\targetDistribution := \bigl((X_1^{\targetDistribution},Y_1^{\targetDistribution}),\ldots,(X_{\numTarget}^{\targetDistribution},Y_{\numTarget}^{\targetDistribution})\bigr)$ from~$\targetDistribution$, and wish to classify a new observation $(X^{\targetDistribution},Y^{\targetDistribution}) \sim \targetDistribution$.  Our first contribution is to formalise the new, decision tree-based transfer framework to incorporate the broad range of relationships between source and target distributions seen in practical applications such as those mentioned above.  In particular, in contrast to most other work in this area, 
our highly flexible form of relationship between $\sourceRegressionFunction$ and $\targetRegressionFunction$ does not require that the Bayes decision boundaries agree for the two populations; we also allow the marginal distributions of $X^\sourceDistribution$ and $X^\targetDistribution$ to differ, and do not assume that these distributions have densities that are bounded away from zero on their respective supports.  The classes of distributions we consider, then, combine local smoothness assumptions on $\sourceRegressionFunction$ and $\targetRegressionFunction$ with tail assumptions on the marginal distribution of $X^\targetDistribution$ and the marginal distribution of $X^\sourceDistribution$.  To understand the fundamental difficulty of the transfer learning problem, we derive a minimax lower bound that comprises several regimes, according to the relative sample sizes and the strength of the transfer relationship, as measured by the distributional parameters of our classes.  The next challenge is to introduce a new method for the transfer learning task; our basic idea is to use $\mathcal{D}_\sourceDistribution$ to construct a local nearest neighbour-based estimate of $\sourceRegressionFunction$, and then perform empirical risk minimisation with $\mathcal{D}_\targetDistribution$  to estimate the underlying decision tree partition and the values of the transfer functions at $1/2$.  We derive a high-probability upper bound for the excess test error of our procedure, which, together with our lower bound, reveals that our algorithm attains the minimax optimal rate, up to a poly-logarithmic factor.  A notable feature of our methodology is that the only inputs required are $\mathcal{D}_\sourceDistribution$ and $\mathcal{D}_\targetDistribution$; in particular, it is adaptive to the unknown transfer relationship in the primary regime of interest, as well as the smoothness and tail parameters of our distributional classes, and the confidence with which the test error bound holds.

Interest in transfer learning has been growing considerably in recent years. One broad line of work considers the setting where the practitioner only has access to labelled data from $\sourceDistribution$, possibly with some additional unlabelled data from $\targetDistribution$. A popular approach in that context is to formulate a measure of discrepancy between the distributions $\sourceDistribution$ and $\targetDistribution$ and to give test error bounds in terms of this discrepancy \citep{ben2010theory,ben2010impossibility,germain2015risk,mansour2009domain,mohri2012new,
cortes2019adaptation}.  This strategy has been shown to yield distribution-free bounds with wide applicability, but whenever the discrepancy is non-zero, the excess error is not guaranteed to converge to zero with the sample size.  In order to achieve consistent classification, we must impose additional structure \citep{ben2010impossibility}, e.g.~by focusing on label shift, covariate shift or label noise, each of which may be viewed as a special case of transfer learning.  In label shift \citep{zhang2013domain,lipton2018detecting}, the marginal distributions of $Y^\sourceDistribution$ and $Y^\targetDistribution$ differ, but the class-conditional covariate distributions are the same for $\sourceDistribution$ and $\targetDistribution$.  Covariate shift \citep{gretton2009covariate,candela2009dataset,
sugiyama2012density} concerns scenarios where the regression functions $\sourceRegressionFunction$ and $\targetRegressionFunction$ are assumed to be equal, but the marginal distributions of $\sourceDistribution$ and $\targetDistribution$ may differ.  In label noise \citep{blanchard2017domain,
reeve2019classification,scott2018generalized,scott2019learning}, $\sourceRegressionFunction$ and $\targetRegressionFunction$ differ.  This does not necessarily preclude consistent classification, even when $n_\targetDistribution = 0$, provided that additional restrictions are met.  For instance, the Bayes classifier may still be the same for $\sourceDistribution$ and $\targetDistribution$, in which case one can sometimes proceed as if there were no label noise \citep{menon2018learning,cannings2020}; alternatively, if the label noise only depends on the true class label, then the label noise parameters may be estimated under certain identifiability assumptions \citep{blanchard2016classification, reeve2019fastLabelNoise}. 

Other related work that considers the current setting where the statistician has access to labelled data from both source and target distributions includes \citet{kpotufe2018marginal} for the covariate shift problem, and \cite{hanneke2019value} and \cite{cai2019transfer} for general transfer learning.   The frameworks of these papers ensure that $\left(\sourceRegressionFunction(x)-1/2\right)\left(\targetRegressionFunction(x)-1/2\right)>0$ whenever $\targetRegressionFunction(x)\neq 1/2$, and hence the Bayes classifiers for $\sourceDistribution$ and $\targetDistribution$ are equal. In our terminology, this corresponds to the special case where $g_\ell(1/2)=1/2$ for all $\ell$.  In each of these works, the authors obtain minimax rates of convergence for the excess error in their respective problems, which in particular reveal that consistent classification is possible with $\mathcal{D}_{\sourceDistribution}$ alone, and the effect of $\mathcal{D}_{\targetDistribution}$ is to improve the rates.  The only work in this context of which we are aware that allows the Bayes classifier for the two distributions to differ is the very recent contribution of \citet*{maity2020minimax}.  These authors consider the label shift problem, so the differences between $\sourceDistribution$ and $\targetDistribution$ are captured through a single parameter governing the similarity of $\Prob(Y^{\sourceDistribution}=1)$ and $\Prob(Y^{\targetDistribution}=1)$.  \citet{maity2020minimax} show how this parameter can be efficiently estimated from the data (which can even be unlabelled), and are therefore also able to obtain minimax rates of convergence for their problem.

One of our main goals in this work is to allow more flexible forms of transfer, to make our framework applicable to the examples discussed above.  The price we pay for this generality is that our rates of convergence are necessarily slower than those of  \citet{kpotufe2018marginal}, \citet{cai2019transfer}, \citet{hanneke2019value} and \citet{maity2020minimax}. Nonetheless, our minimax rates conclusively demonstrate the benefits of transfer learning in a highly flexible setting.

The remainder of this paper is organised as follows: in Section~\ref{Sec:Setting}, we introduce our general transfer learning framework, and state our main minimax optimality result (Theorem~\ref{Thm:Main}).  Section~\ref{Sec:Methodology} gives a formal description of our algorithm, as well as a high-probability upper bound for its excess test error (Theorem~\ref{Thm:UpperBound}), while a conclusion is provided in Section~\ref{Sec:Conclusion}.  The proofs of Theorem~\ref{Thm:UpperBound} and the upper bound in Theorem~\ref{Thm:Main} are given in Section~\ref{Sec:UpperBound}, and the proof of the lower bound in Theorem~\ref{Thm:Main} is provided in Section~\ref{Sec:LowerBound}.  Auxiliary results and illustrative examples are deferred to the Appendix; there we also present the results of a brief simulation study.

We conclude this introduction with some notation used throughout the paper.  Given a set~$A$, we write $|A|$ for its cardinality, and $\mathrm{Par}(A)$ for the set of all finite partitions of $A$, i.e.~the set consisting of elements of the form $\{A_1,\ldots,A_m\}$, with $A_1,\ldots,A_m$ pairwise disjoint and $\cup_{\ell=1}^m A_\ell = A$.  We let $\N_0 := \N \cup \{0\}$, and for $n \in \N$, let $[n] := \{1,\ldots,n\}$.  For $x \in \R^d$, we write $\|x\|$ for the Euclidean norm of $x$, and, given $r > 0$, we write $B_r(x) := \{y \in \R^d:\|y-x\| < r\}$ for the open Euclidean ball of radius $r$ about $x$.  We let $\mathcal{L}_d$ denote Lebesgue measure on $\R^d$, and let $V_d := \mathcal{L}_d\bigl(B_1(0)\bigr) = \pi^{d/2}/\Gamma(1+d/2)$.  For $x \geq 0$, we let $\log_+(x) := \log x$ if $x \geq e$, and $\log_+(x) := 1$ otherwise.  If $\mu, \nu$ are probability measures on $(\mathcal{X},\mathcal{A})$, then we write $\mathrm{TV}(\mu,\nu) := \sup_{A \in \mathcal{A}} |\mu(A) - \nu(A)|$ for their total variation distance, and if $\mu$ is absolutely continuous with respect to $\nu$ with Radon--Nikodym derivative $d\mu/d\nu$, we write $\mathrm{KL}(\mu,\nu) := \int_{\mathcal{X}} \log \bigl(\frac{d\mu}{d\nu}\bigr) \, d\mu$ for the Kullback--Leibler divergence from $\nu$ to~$\mu$.    Finally, the support of a probability measure~$\mu$ on $\R^d$, denoted $\supp(\mu)$, is defined to be the intersection of all closed sets $C \subseteq \R^d$ with $\mu(C) = 1$.

\section{Statistical setting and main result}
\label{Sec:Setting}

\sloppy Let $\sourceDistribution$, $\targetDistribution$ be distributions on $\R^d \times \{0,1\}$ and let 
 $(X^{\sourceDistribution},Y^{\sourceDistribution}) \sim \sourceDistribution$ and $(X^{\targetDistribution},Y^{\targetDistribution}) \sim \targetDistribution$.  We recall the definitions of the regression functions~$\sourceRegressionFunction$ and $\targetRegressionFunction$ from~\eqref{Eq:RegressionFunctions}, and write $\sourceMarginalDistribution$ and $\targetMarginalDistribution$ for the marginal distributions of $X^{\sourceDistribution}$ and $X^{\targetDistribution}$ respectively. 
 
 A \emph{classifier} is a Borel measurable function $\classifier: \mathbb{R}^d \rightarrow \{0,1\}$.  In practice, classifiers are constructed on the basis of training data, and we will assume that for some $\numSource, \numTarget \in \N_0$, we have access to independent pairs $({X}^{\sourceDistribution}_1,{Y}^{\sourceDistribution}_1),\ldots,({X}^{\sourceDistribution}_{\numSource},{Y}^{\sourceDistribution}_{\numSource}) \sim \sourceDistribution$ and $({X}^{\targetDistribution}_1,{Y}^{\targetDistribution}_1),\ldots ,({X}^{\targetDistribution}_{\numTarget},{Y}^{\targetDistribution}_{\numTarget}) \sim \targetDistribution$.  Recall that as shorthand, we denote $\mathcal{D}_\sourceDistribution = \bigl(({X}^{\sourceDistribution}_1,{Y}^{\sourceDistribution}_1),\ldots ,({X}^{\sourceDistribution}_{\numSource},{Y}^{\sourceDistribution}_{\numSource})\bigr)$ and $\mathcal{D}_\targetDistribution = \bigl(({X}^{\targetDistribution}_1,{Y}^{\targetDistribution}_1),\ldots ,({X}^{\targetDistribution}_{\numTarget},{Y}^{\targetDistribution}_{\numTarget})\bigr)$.  A \emph{data-dependent classifier} $\classifierEst$ is a measurable function from $(\mathbb{R}^d\times \{0,1\})^{\numSource} \times \left(\mathbb{R}^d\times \{0,1\}\right)^{\numTarget}\times \mathbb{R}^d$ to $\{0,1\}$, and we let $\setofDDClassifiers$ denote the set of all such data-dependent classifiers.  In this work, the first arguments of $\classifierEst \in \setofDDClassifiers$ will always be $\mathcal{D}_\sourceDistribution$ and $\mathcal{D}_\targetDistribution$, so we will often suppress all but the final argument of $\classifierEst$, noting also that the mapping $x\mapsto \classifierEst(x)$ is a classifier.  Conversely, any classifier may be regarded as a data-dependent classifier that is constant in all but its final argument.  The \emph{test error} of $\classifierEst \in \setofDDClassifiers$ is given by
 \begin{equation}
\label{Eq:TestError}
\mathcal{R}(\classifierEst) := \mathbb{P}\bigl(\classifierEst(X^\targetDistribution) \neq Y^\targetDistribution\bigm| \mathcal{D}_\sourceDistribution,\mathcal{D}_\targetDistribution\bigr),
\end{equation}
where $(X^\targetDistribution,Y^\targetDistribution) \sim \targetDistribution$ is independent of our training data, and is minimised for every $\mathcal{D}_\sourceDistribution$ and $\mathcal{D}_\targetDistribution$ by the \emph{Bayes classifier} $\bayesClassifier$, where $\bayesClassifier(x) := \mathbbm{1}_{\{\targetRegressionFunction(x) \geq 1/2\}}$.  The \emph{excess test error} of $\classifierEst \in \setofDDClassifiers$ is given by 
\begin{align}\label{def:excessRisk}
\excessRisk(\classifierEst):=\risk(\classifierEst)-\risk(\bayesClassifier)= \int_{\{x:\classifierEst(x) \neq \bayesClassifier(x)\}} \left|2\targetRegressionFunction(x)-1\right| \, d\targetMarginalDistribution(x).
\end{align}

In order to provide a formal statement of our key transfer assumption, we first define the notion of a decision tree partition: 
\begin{definition}[Decision tree partitions]\label{thresholdPartitionsDef} Let $\mathbb{T}_1 := \bigl\{\{\mathbb{R}^\ambientDimension\}\bigr\} \subseteq \mathrm{Par}(\R^\ambientDimension)$, and for $L \geq 2$, define the subset of $\mathrm{Par}(\R^\ambientDimension)$ given by 
\[
\mathbb{T}_L := \bigl\{\{\X_1,\ldots, \X_{L-1} \cap H_{j,s},\X_{L-1} \setminus H_{j,s}\}: j \in [d], s \in \R, \{\X_1,\ldots,\X_{L-1}\} \in \mathbb{T}_{L-1}\bigr\},
\]
where $H_{j,s}:=\{(x_t)_{t \in [\ambientDimension]} \in \R^{\ambientDimension}: x_j \geq s\}$ for $j \in [\ambientDimension]$ and $s \in \R$.  The set of all decision tree partitions is $\cup_{L \in \N} \mathbb{T}_L$.
\end{definition}
We illustrate some elements of $\mathbb{T}_1, \mathbb{T}_2, \mathbb{T}_3$ and $\mathbb{T}_4$ when $d=2$ in Figure~\ref{Fig:T}.
\begin{figure}[htbp]
\begin{center}
\includegraphics[width=0.23\textwidth,trim=0 100 0 0]{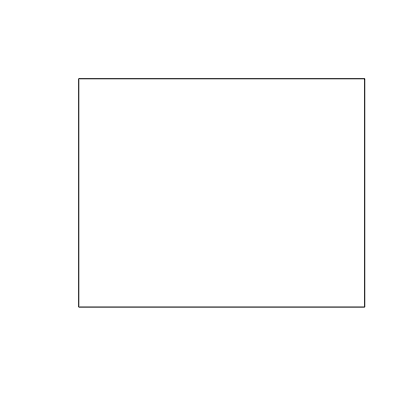}
\includegraphics[width=0.23\textwidth,trim=0 100 0 0]{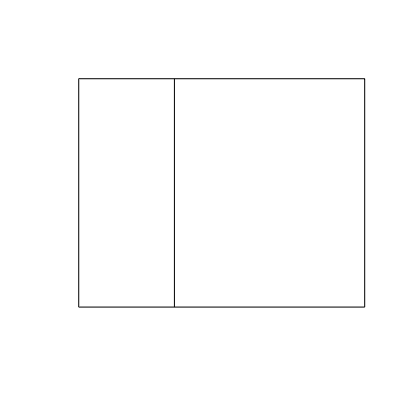}
\includegraphics[width=0.23\textwidth,trim=0 100 0 0]{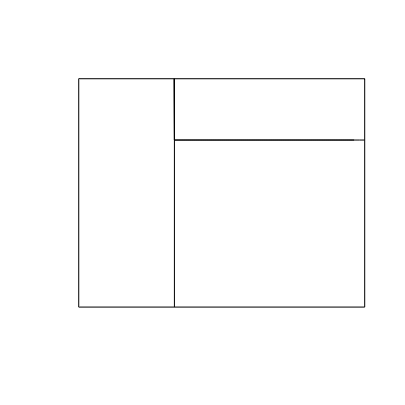}
\includegraphics[width=0.23\textwidth,trim=0 100 0 0]{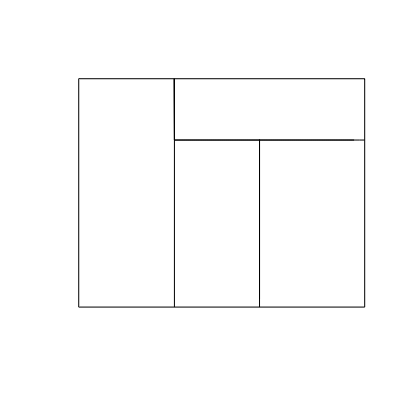}
\end{center}
\caption{\label{Fig:T}Illustration of elements of $\mathbb{T}_1$, $\mathbb{T}_2$, $\mathbb{T}_3$ and $\mathbb{T}_4$.}
\end{figure}
\begin{assumption}[Transfer]\label{transferAssumption} There exist $\{\X_1^*,\ldots,\X_{L^*}^*\} \in \mathbb{T}_{L^*}$, as well as $\Delta \in [0,1)$, $\expansivityConstant \in (0,1)$ and transfer functions $g_1,\ldots,g_{L^*} :[0,1]\rightarrow [0,1]$ such that $\bigl|\sourceRegressionFunction(x)-g_\ell(\targetRegressionFunction(x)\bigr)\bigr| \leq \Delta$ for every $\ell \in [L^*]$ and $x \in \X_{\ell}^*$; moreover,
\begin{equation}
\label{Eq:gell}
\frac{g_\ell(z)-g_\ell(1/2)}{z-1/2} \geq \expansivityConstant
\end{equation}
for every $\ell \in [L^*]$ and $z \in [0,1/2) \cup (1/2,1]$. 
\end{assumption}
To understand this assumption, first consider the case where $\Delta = 0$.  Then our condition states that for the $\X_\ell^*$ cell of our decision tree partition, we have the relationship $\sourceRegressionFunction =    g_\ell  \circ \targetRegressionFunction$, so that within this cell, $\sourceRegressionFunction(x)$ only depends on $x$ through $\targetRegressionFunction(x)$.  Moreover,~\eqref{Eq:gell} asks that each $g_\ell$ is strictly increasing at $1/2$, and is of course satisfied if each $g_\ell$ is differentiable with $g_\ell'(z) \geq \expansivityConstant$ for $z \in [0,1]$.  More generally, for $\Delta > 0$, Assumption~\ref{transferAssumption} only requires that the relationship $\sourceRegressionFunction =    g_\ell  \circ \targetRegressionFunction$ holds to within an error of $\Delta$ on each cell of our decision tree partition.

Our next assumption concerns the mass of the source and target distributions in the tails.  Given a probability distribution $\mu$ on $\R^d$ and $d_0 \in [0,d]$,  we define the \emph{lower density} $\omega_{\mu,d_0}:\mathbb{R}^d \rightarrow [0,1]$ of $\mu$ by
\begin{equation}
\label{Eq:omega}
\omega_{\mu,d_0}(x) := \inf_{r\in (0,1)} \frac{\mu\bigl(B_r(x)\bigr)}{r^{d_0}}. 
\end{equation}
For intuition, when~$\mu$ is absolutely continuous with respect to the volume form on a $d_0$-dimensional, orientable manifold in $\mathbb{R}^d$ with density (Radon--Nikodym derivative) $f_\mu$, and if the infimum in~\eqref{Eq:omega} is replaced with a $\liminf$ as $r \searrow 0$, then $\omega_{\mu,d_0}$ is almost everywhere equal to a constant multiple of $f_\mu$ \citep[][Lemma~4.1.2]{ledrappier1985metric}.  This explains our lower density terminology.  Further properties of this lower density, which can in fact be defined on general separable metric spaces, can be inferred from common assumptions in the classification literature, including an assumption of \emph{regular support} \citep{audibert2007fast} and a \emph{strong minimal mass assumption} \citep{gadat2016classification}; see Lemmas~\ref{regularSuperLevelSetsLemma} and~\ref{Lemma:SMM} for details.  We also note that the definition in~\eqref{Eq:omega} has some similarities with that of a Hardy--Littlewood operator \citep{hardy1930maximal}, though one important difference with the standard definition is that here an infimum replaces a supremum.
\begin{assumption}[Marginals]
\label{sourceDensityAssumption} 
There exist $\targetDimension \in [1,\ambientDimension]$, $ \targetTailExponent  >0$ and $\TailConstant > 1$ such that
\begin{align}
\label{Eq:targetMarginal}
\targetMarginalDistribution\bigl(\bigl\{ x \in \R^d: \omega_{\mu_\targetDistribution,d_\targetDistribution}(x)<\xi\bigr\} \bigr) \leq \TailConstant \cdot \xi^{\targetTailExponent}
\end{align}
for all $\xi>0$.  Moreover, there exist $\sourceDimension \in [\targetDimension, \ambientDimension]$ and $\sourceTailExponent  >0$ such that 
\begin{align}
\label{Eq:sourceMarginal}
\targetMarginalDistribution\bigl(\bigl\{ x \in \R^d: \omega_{\mu_\sourceDistribution,d_\sourceDistribution}(x)<\xi\bigr\} \bigr) \leq \TailConstant \cdot \xi^{\sourceTailExponent}
\end{align}
for all $\xi>0$.  
\end{assumption}
To understand the first part of Assumption~\ref{sourceDensityAssumption}, first consider the case where $\mu_Q$ is absolutely continuous with respect to $\mathcal{L}_d$.  In that case, condition~\eqref{Eq:targetMarginal} can be viewed as similar to other tail conditions in the classification literature that control the $\mu_Q$ measure of the set on which this density is small \citep[e.g.][Assumption~A4]{gadat2016classification}.  Thus,~\eqref{Eq:targetMarginal} is a generalisation of such a tail condition, because we do not require $\mu_Q$ to be absolutely continuous with respect to $\mathcal{L}_d$, and instead work with its lower density $\omega_{\mu_Q,d_Q}$.  The great advantage of this formulation in~\eqref{Eq:targetMarginal} is that it allows us to avoid assuming that this lower density is bounded away from zero on the support of $\mu_Q$; Example~\ref{Ex:gammafamily} provides a simple, univariate parametric family of densities $\{f_\gamma:\gamma > 0\}$ for which $\gamma_Q = \gamma$ is the optimal choice.  

Further intuition about the first part of Assumption~\ref{sourceDensityAssumption} can be gained from several results in the appendix that we now summarise.  In Lemma~\ref{momentTailAssumptionSepMetricSpace}, we show that if $\mu_{\targetDistribution}$ has a finite $\rho$th moment for some $\rho > 0$, then~\eqref{Eq:targetMarginal} holds with $\targetDimension = d$ and $\targetTailExponent = \rho/(\rho+d)$.  The proof relies on Vitali's covering lemma \citep[e.g.][Theorem~1]{evans2015measure}, and we believe the result may find application elsewhere; see Remark~\ref{Rem:Moment} after Lemma~\ref{momentTailAssumptionSepMetricSpace}.  As a consequence of a general result about Weibull-type tails~(Lemma~\ref{exponentialTailsLemma}), Proposition~\ref{Prop:LCGenerald} shows that when $\mu_{\targetDistribution}$ has a log-concave density on $\R^d$ with $d_0$-dimensional support,~\eqref{Eq:targetMarginal} holds with $\targetDimension = d_0$ and any $\targetTailExponent < 1$; in fact, when $d_Q = 1$, we may even take $\targetTailExponent = 1$~(Proposition~\ref{Prop:LCd1}).  Moreover, Proposition~\ref{Lemma:MixtureLC} extends these results to finite mixtures of log-concave distributions, with~$\TailConstant$ depending linearly on the number of mixture components (and not depending on the mixing proportions).  In fact, more generally, Propositions~\ref{lemma:mixtureTails} and~\ref{Prop:Products} provide simple stability results for the property~\eqref{Eq:targetMarginal} under finite mixtures and products respectively.  As additional important examples, whenever~$\mu_Q$ has bounded, $d_0$-dimensional support, we may take $d_Q = d_0$ and $\gamma_Q = 1$ (by Lemma~\ref{boundedSetTailAssumptionSepMetricSpace}); moreover, if~$\mu_Q$ has a density that is bounded away from zero on a $d_0$-dimensional, regular support, then we may take $d_Q = d_0$ and $\gamma_Q$ to be arbitrarily large (by Lemma~\ref{regularSuperLevelSetsLemma}).

The second part of Assumption~\ref{sourceDensityAssumption} relates $\mu_P$ and $\mu_Q$ together: it controls the $\mu_Q$ measure of the set on which the lower density of $\mu_P$ is small, thereby capturing the extent to which the source measure covers the target measure.  For instance, Example~\ref{Ex:Gaussianscale} reveals that when $\mu_Q = N(0,1)$ and $\mu_P = N(0,\sigma^2)$, we may take $\gamma_P = \sigma^2$, while, from Example~\ref{Ex:GaussianLocation}, we see that when $\mu_Q = N(0,1)$ and $\mu_P = N(a,1)$ for some $a \neq 0$, we may take any $\gamma_P < 1$.  We remark that if~\eqref{Eq:targetMarginal} holds and if, in the terminology of \citet{kpotufe2018marginal}, $(P,Q)$ have \emph{transfer-exponent} $\kappa \in [0,\infty]$, then~\eqref{Eq:sourceMarginal} holds for any $\sourceDimension \geq \targetDimension + \kappa$ and with $\sourceTailExponent = \targetTailExponent$; see Lemma~\ref{Lemma:TransferTransferExponents}.  Moreover, Example~\ref{Ex:Prototype} provides a prototypical setting where working with the condition~\eqref{Eq:sourceMarginal} allows us to obtain faster rates of convergence than would be the case if we instead deduced this rate from the corresponding transfer-exponent.

Our next two assumptions are standard margin \citep[e.g.][]{polonik1995measuring,mammen1999} and smoothness assumptions.  We emphasise that these are only imposed on the distribution $\targetDistribution$, and we require no corresponding properties for $\sourceDistribution$.

\begin{assumption}[Margin]\label{marginAssumption} There exist $\marginExponent >0$ and $\marginConstant \geq 1$ such that for all $\zeta>0$ we have $\targetMarginalDistribution\bigl(\bigl\{x \in \R^d:\left|\targetRegressionFunction(x)-1/2\right|<\zeta\bigr\}\bigr) \leq \marginConstant \cdot \zeta^{\marginExponent}$. 
\end{assumption}

\begin{assumption}[Smoothness]\label{holderContinuityAssumption} There exist $\holderExponent \in (0,1]$ and $\holderConstant \geq 1$ such that $\bigl| \targetRegressionFunction(x_0)-\targetRegressionFunction(x_1)\bigr| \leq \holderConstant \cdot \|x_0-x_1\|^{\holderExponent}$ for all $x_0$, $x_1 \in \R^d$.
\end{assumption}

It will be convenient to write $\theta$ for the vector of parameters that appear in Assumptions~2--4, namely $(\targetDimension,\targetTailExponent,\sourceDimension,  \sourceTailExponent,\TailConstant,\marginExponent,\marginConstant,\holderExponent,\holderConstant)$, and to write $\Theta$ for the corresponding parameter space.  We will also make use of an augmented parameter vector that incorporates the additional parameters that appear in Assumption~1, by letting $\theta^\sharp := (\Delta,\expansivityConstant,L^*,\theta)$, with corresponding parameter space $\Theta^\sharp$.  For $\theta^\sharp \in \Theta^\sharp$, we write $\mathcal{P}_{\theta^\sharp}$ for the set of pairs $(P,Q)$ of distributions satisfying Assumptions~1--4 with parameter $\theta^\sharp$.

We are now in a position to state our main result.
\begin{theorem}
\label{Thm:Main}
Fix $\theta^\sharp = (\Delta,\expansivityConstant,L^*,\theta) \in \Theta^\sharp$ with $\holderExponent/(2\holderExponent+\targetDimension)<\targetTailExponent$, $\holderExponent/(2\holderExponent+\sourceDimension)<\sourceTailExponent$, $\marginExponent\holderExponent \leq \targetDimension$, $\sourceTailExponent(1 - \targetTailExponent) \leq \targetTailExponent$ and $\marginConstant\geq 1 + 2^{2\targetDimension/\holderExponent}\targetDimension^{\targetDimension/2}V_{\targetDimension}$.  For $j \in \{\mathrm{L},\mathrm{U}\}$, let
\begin{align*}
\firstTerm_{\numSource,\numTarget}^j &:= \biggl(\frac{a_0^j}{\expansivityConstant^2 \cdot \numSource}\biggr)^{\frac{\holderExponent \sourceTailExponent (1+\marginExponent) }{\sourceTailExponent(2\holderExponent+\sourceDimension)+\marginExponent \holderExponent }} + \min\biggl\{\biggl(  \frac{{L^*} a_1^j}{ \numTarget}\biggr)^{\frac{1+\marginExponent}{2+\marginExponent}},(1-\expansivityConstant)^{1+\marginExponent}\biggr\} + \biggl(\frac{\Delta}{\expansivityConstant}\biggr)^{1+\marginExponent}, \\
\secondTerm_{\numTarget}^j &:= \biggl(\frac{b^j}{\numTarget}\biggr)^{\frac{\holderExponent \targetTailExponent (1+\marginExponent) }{\targetTailExponent(2\holderExponent+\targetDimension)+\marginExponent \holderExponent}},
\end{align*}
where $a_0^{\mathrm{L}} = a_1^{\mathrm{L}} = b^{\mathrm{L}} := 1$, $a_0^\mathrm{U} := \log_+ (\numSource)$, $a_1^{\mathrm{U}} := \log_+\bigl( L^* \ambientDimension (\numSource + \numTarget)\bigr)$ and $b^{\mathrm{U}} := \log_+ (\numTarget)$.  Then there exist $c_\theta,C_\theta > 0$, depending only on $\theta$, such that 
\begin{align}
\label{Eq:Main}
c_\theta \bigl(\firstTerm_{\numSource,\numTarget}^{\mathrm{L}} \wedge \secondTerm_{\numTarget}^{\mathrm{L}} \wedge 1\bigr) \leq \inf_{\classifierEst \in \setofDDClassifiers}\sup_{(P,Q) \in \mathcal{P}_{\theta^\sharp}} \mathbb{E}\bigl\{\excessRisk(\classifierEst)\bigr\} \leq C_\theta\bigl(\firstTerm_{\numSource,\numTarget}^{\mathrm{U}} \wedge\secondTerm_{\numTarget}^{\mathrm{U}} \wedge 1\bigr).
\end{align}
\end{theorem}
Theorem~\ref{Thm:Main} establishes the optimal rates of convergence for the excess risk over our classes, up to logarithmic factors.  It is important to note that $c_\theta$ and $C_\theta$ do not depend on $(\Delta,\expansivityConstant,L^*)$ (and nor on $\numSource$ or $\numTarget$); thus the theorem reveals the optimal dependence of the worst-case excess risk on these parameters too.  Moreover, as we will show in Theorem~\ref{Thm:UpperBound}, the minimax rate can be achieved up to a poly-logarithmic factor when $\phi \leq 1- \numTarget^{-1/(2+\marginExponent)}$ by a procedure that is completely adaptive, in the sense that it only takes $\sourceSample$ and $\targetSample$ as inputs (and not any component of $\theta^\sharp$).

The restrictions on the parameters in Theorem~\ref{Thm:Main} are mild.  For instance, by Lemma~\ref{momentTailAssumptionSepMetricSpace}, the conditions $\holderExponent/(2\holderExponent+\targetDimension)<\targetTailExponent$ and $\sourceTailExponent(1 - \targetTailExponent) \leq \targetTailExponent$ hold whenever $\sup_{(P,Q) \in \mathcal{P}_{\theta^\sharp}} \mathbb{E}\bigl(\|X^\targetDistribution\|^{1 \vee \targetDimension \sourceTailExponent}\bigr) < \infty$.  The condition $\marginExponent\holderExponent \leq \targetDimension$ rules out `super-fast rates' (in the terminology of \citet{audibert2007fast}) and is guaranteed to hold whenever there exist $(P,Q) \in \mathcal{P}_{\theta^\sharp}$ and $x_0 \in \R^\ambientDimension$ such that $\targetRegressionFunction(x_0) = 1/2$ and $\omega_{\targetMarginalDistribution,\targetDimension}(x_0) > 0$ (see Lemma~\ref{Lemma:TBD}).  The first two parameter restrictions in Theorem~\ref{Thm:Main} are only required for the upper bound, while the other three are only needed for the lower bound.  We also remark that Theorem~\ref{Thm:Main} holds even when $\numSource$ or $\numTarget$ are zero.  In the former case, the problem reduces to a standard classification problem, while in the latter case, Theorem~\ref{Thm:Main} provides results for relaxations of the covariate shift model in which $\sourceRegressionFunction$ and $\targetRegressionFunction$ are close.

By careful inspection of the proof of Theorem~\ref{Thm:Main}, we see that the first terms $\firstTerm_{\numSource,\numTarget}^{\mathrm{L}}$ and $\firstTerm_{\numSource,\numTarget}^{\mathrm{U}}$ in the bounds are due to the transfer learning error, and comprise three separate contributions.  The first term arises from the error incurred in estimating $\sourceRegressionFunction$.  The second represents the difficulty of identifying the correct decision tree partition $\{\X_1^*,\ldots,\X_{L^*}^*\}$, as well as learning $g_1(1/2),\ldots,g_{L^*}(1/2)$; this term is negligible if $\expansivityConstant$ is sufficiently close to 1.  In fact, as we will see from the proof of Theorem~\ref{Thm:UpperBound} below, it is not necessary to carry out this step when $\expansivityConstant$ is close to 1.  Finally, the third term reflects the extent to which~$\sourceRegressionFunction$ can be approximated by $g_\ell \circ \targetRegressionFunction$ on $\X_\ell^*$.  The $\secondTerm_{\numTarget}^{\mathrm{L}}$ and $\secondTerm_{\numTarget}^{\mathrm{U}}$ terms represent the rate of convergence achievable by ignoring $\sourceSample$ and performing standard classification using $\targetSample$; in the context of transfer learning, our primary interest is in the setting where $\numSource \gg \numTarget$, and where the minima in~\eqref{Eq:Main} are attained by $\firstTerm_{\numSource,\numTarget}^{\mathrm{L}}$ and $\firstTerm_{\numSource,\numTarget}^{\mathrm{U}}$ respectively.

To set the rates $\secondTerm_{\numTarget}^{\mathrm{L}}$ and $\secondTerm_{\numTarget}^{\mathrm{U}}$ in context, it may be helpful to consider the case where $\targetMarginalDistribution$ is absolutely continuous with respect to $\Lebesgue$ with a density that is bounded away from zero on its regular support; see Definition~\ref{Def:Regular}.  In that case, we may take $\targetTailExponent$ to be arbitrarily large and $\targetDimension = \ambientDimension$; notice that setting $\targetTailExponent = \infty$ and $\targetDimension = \ambientDimension$ in $\secondTerm_{\numTarget}^{\mathrm{L}}$ recovers the rate $\numTarget^{-\frac{\holderExponent(1+\marginExponent)}{2\holderExponent+\ambientDimension}}$ for the standard classification problem (with no source data) in \citet{audibert2007fast} under this regular support hypothesis.  Returning to the more general transfer learning setting, if we take $\sourceTailExponent = \targetTailExponent = \infty$ and $\sourceDimension = \targetDimension = \ambientDimension$, then the minimum of the first term in $\firstTerm_{\numSource,\numTarget}^{\mathrm{L}}$ and $\secondTerm_{\numTarget}^{\mathrm{L}}$ matches the rate obtained by \citet{cai2019transfer}.  The second and third terms in $\firstTerm_{\numSource,\numTarget}^{\mathrm{L}}$ represent the necessary additional price for the generality of our framework.

To illustrate Theorem~\ref{Thm:Main}, and ignoring logarithmic factors for simplicity, consider the special case where $\sourceTailExponent = \targetTailExponent$ and $\sourceDimension = \targetDimension$ (which would in particular be the case if the marginal distributions $\sourceMarginalDistribution$ and $\targetMarginalDistribution$ coincide).  Then Theorem~\ref{Thm:Main} reveals that in order for transfer learning to be effective (as opposed to simply constructing a classifier based on $\targetSample$), we require $\expansivityConstant^2 \cdot \numSource \gg \numTarget$.  If we further assume that $\sourceTailExponent = \targetTailExponent = 1$, that $\sourceDimension = \targetDimension = \ambientDimension$, that $\Delta = 0$ and that $\marginExponent = \holderExponent = 1$, then we benefit from transfer learning provided that $\expansivityConstant^2 \cdot \numSource \gg \numTarget$ and $L^* \ll \numTarget^{\ambientDimension/(\ambientDimension+3)}$.  In general, the scope for transfer learning to have an impact increases as $\sourceTailExponent$ and $\expansivityConstant$ increase, and as $\sourceDimension$, $L^*$ and $\Delta$ decrease.

\section{Methodology and upper bound}
\label{Sec:Methodology}

In this section, we introduce our adaptive algorithm for transfer learning and provide a high-probability bound for its excess risk.  To understand the main idea, consider the case where $\Delta = 0$ in Assumption~\ref{transferAssumption}, and where we are told the correct decision tree partition and transfer functions.  In this setting, when $x \in \X_\ell^*$, the sign of $\sourceRegressionFunction(x) - g_{\ell}(1/2)$ agrees with the sign of $\targetRegressionFunction(x) - 1/2$, so we aim to construct a nearest-neighbour based estimate of the former quantity using $\sourceSample$.  In practice, this estimate will depend on a choice of decision tree, but this can be calibrated using a subsample from~$\targetSample$.  Separately, we also construct a standard $k$-nearest neighbour estimate of $\targetRegressionFunction(x) - 1/2$ via the same subsample from~$\targetSample$, and make our final choice between the two data-dependent classifiers using empirical risk minimisation over the held-out data from $\targetSample$.  The independence of the two subsamples from $\targetSample$ allows us to work conditionally on the first subsample at this final step to obtain our final performance guarantees.

In giving a formal description of our algorithm, we will assume that $\numTarget \geq 2$ (when $\numTarget \leq 1$, the upper bound in Theorem~\ref{Thm:Main} is attained by applying a nearest-neighbour method to $\sourceSample$), and it will also be convenient initially to assume that $\numSource > 0$.  For $x \in \R^d$ and $k \in [\numSource]$ we let $X_{(k)}^{\sourceDistribution} \equiv X_{(k)}^{\sourceDistribution}(x)$ denote the $k$th nearest neighbour of $x$ in~$\sourceSample$ in Euclidean norm (where for definiteness, in the case of ties, we preserve the original ordering of the indices), and let 
$Y_{(k)}^{\sourceDistribution} \equiv Y_{(k)}^{\sourceDistribution}(x)$ denote the concomitant label.  We then split $\targetSample$ into two subsamples $\targetSample^0:=\bigl((X_1^{\targetDistribution},Y_1^{\targetDistribution}),\ldots,(X_{\numTargetHalf}^{\targetDistribution},Y_{\numTargetHalf}^{\targetDistribution})\bigr)$ and $\targetSample^1:=\bigl((X_{\numTargetHalf+1}^{\targetDistribution},Y_{\numTargetHalf+1}^{\targetDistribution}),\ldots,(X_{\numTarget}^{\targetDistribution},Y_{\numTarget}^{\targetDistribution})\bigr)$.  For $k \in [\numTargetHalf]$ we let $X_{(k)}^{\targetDistribution} \equiv X_{(k)}^{\targetDistribution}(x)$ denote the $k$th nearest neighbour of $x$ in~$\targetSample^0$ and similarly let
$Y_{(k)}^{\targetDistribution} \equiv Y_{(k)}^{\targetDistribution}(x)$ denote the concomitant label.

Given $L \in \mathbb{N}$ and a decision tree partition $\{\X_1,\ldots,\X_L\} \in \mathbb{T}_L$, we define the \emph{leaf function} $\ell:\R^d \rightarrow [L]$ by $\ell(x) := j$ whenever $x \in \X_j$.  Let $\setOfDecisionTreeFunctionsL$ denote the set of decision tree functions $\decisionTreeFunction:\R^{\ambientDimension} \rightarrow (0,1)$ of the form $x \mapsto \tau_{\ell(x)}$ for some $\{\X_1,\ldots,\X_L\} \in \mathbb{T}_L$ with leaf function $\ell$, and some $(\tau_1,\ldots,\tau_L) \in \{0,1/\numSource,2/\numSource, \ldots, 1\}^L$.  It is also convenient to define $\mathcal{H}_0$ to consist of the single (constant) function that maps $\R^{\ambientDimension}$ to $1/2$ (this will handle the case when $\expansivityConstant$ is very close to 1).  Given $k \in [\numSource]$, $L \in \mathbb{N}_0$ and $\decisionTreeFunction \in \setOfDecisionTreeFunctionsL$, we let 
\begin{align}
  \label{Eq:marginFuncSource}
  \empiricalMarginSourceDecisionTreeK(x) :=\frac{1}{k}\sum_{i=1}^k \bigl\{ Y_{(i)}^{\sourceDistribution}(x)-\decisionTreeFunction\bigl(  X_{(i)}^{\sourceDistribution}(x) \bigr) \bigr\}
\end{align}
denote an empirical estimate of $\sourceRegressionFunction(x) - g_{\ell(x)}(1/2)$.  To choose $k$, we fix a robustness parameter $\robustness \in [\numSource^2]/\numSource = \{1/\numSource,2/\numSource,\ldots,\numSource\}$, and use a Lepski-type procedure to define
\begin{align}
\label{Eq:lepskiKSource}
\hat{k} \equiv \lepskiChoiceOfKSource(x):=   \max\biggl\{ {k \in [\numSource-1]}: \bigl| \hat{m}_{r,h}^P(x)\bigr| \leq  \frac{\robustness}{\sqrt{r}} \text{ for all } r \in [k]\biggr\}+1.
\end{align}
Fixing a confidence level $\delta \in (0,1)$, we will see in Proposition~\ref{prop:lowRegretWithTheRightDecisionTree} that the choice $\robustness^*= \min\bigl\{ \lceil 3 \log_+^{1/2}(\numSource/\delta)\rceil, \numSource\bigr\}$ yields classifiers that perform well with probability at least $1-\delta$.  However, we seek a procedure with simultaneous guarantees across all levels $\delta$, so we will provide a data-dependent choice below.  We now choose $\decisionTreeFunction$ by applying empirical risk minimisation over $\targetSample^0$, so that
\begin{align}\label{ermOverDecisionTreesChoice}
\hat{\decisionTreeFunction} \in \argmin_{\decisionTreeFunction \in \setOfDecisionTreeFunctionsL} \sum_{i=1}^{\numTargetHalf} \Bigl\{Y^{\targetDistribution}_i\mathbbm{1}_{\bigl\{\empiricalMargin^{\sourceDistribution}_{\hat{k},\decisionTreeFunction}(X^{\targetDistribution}_i) < 0\bigr\}} + (1 - Y^{\targetDistribution}_i)\mathbbm{1}_{\bigl\{\empiricalMargin^{\sourceDistribution}_{\hat{k},\decisionTreeFunction}(X^{\targetDistribution}_i) \geq 0\bigr\}}\Bigr\}.
\end{align}
As defined, $\hat{\decisionTreeFunction}$ involves a minimisation over an infinite set of decision tree functions; however, by Lemma~\ref{countingDecisionTreesOnAFiniteSetLemma}, a minimiser can be found by restricting the class $\setOfDecisionTreeFunctionsL$ to a finite set that may in principle be computed from the data.  See Section~\ref{Sec:Empirical} for a discussion of implementational aspects.  Having determined $\hat{\decisionTreeFunction}$, we can now define a family $\potentialFunctionSet^{\sourceDistribution}:= \bigl\{ \classifierEst^{\sourceDistribution}_{\robustness,L}: \robustness \in [\numSource^2]/\numSource, L \in \{0\} \cup [\numTarget]\bigr\} \subseteq \setofDDClassifiers$, where $\classifierEst^{\sourceDistribution}_{\robustness,L}(x) :=\mathbbm{1}_{\{\empiricalMargin^{\sourceDistribution}_{\hat{k},\hat{\decisionTreeFunction}}(x) \geq 0\}}$.  If $\numSource = 0$, then we set $\potentialFunctionSet^{\sourceDistribution}:= \emptyset$.

The second part of our procedure involves applying a $k$-nearest neighbour classifier to~$\targetSample^0$.  More precisely, for $k \in [\numTargetHalf]$, we first define
\begin{align}
  \label{Eq:marginFuncTarget}
  \empiricalMarginTargetK(x) :=\frac{1}{k}\sum_{i=1}^k \biggl\{Y_{(i)}^{\targetDistribution}(x)-\frac{1}{2} \biggr\}.
\end{align}
Given $\robustness \in [\numTarget^2]/\numTarget$, we select a number of neighbours 
\begin{align}
\label{Eq:lepskiKTarget}
\tilde{k} \equiv \lepskiChoiceOfKTarget(x):=  \max\biggl\{ {k \in \bigl[\numTargetHalf-1\bigr]}: \bigl| \hat{m}_r^Q(x)\bigr| \leq  \frac{\robustness}{\sqrt{r}} \text{ for all } r \in [k]\biggr\}+1,
\end{align}
and define another family $\potentialFunctionSet^{\targetDistribution}:= \bigl\{ \classifierEst^{\targetDistribution}_{\robustness}: \robustness \in [\numTarget^2]/\numTarget\bigr\} \subseteq \setofDDClassifiers$ by $\classifierEst^{\targetDistribution}_{\robustness}(x):= \mathbbm{1}_{\{ \empiricalMargin^{\targetDistribution}_{\tilde{k}}(x) \geq 0\}}$.

Our final data-dependent classifier, then, is obtained by empirical risk minimisation over~$\targetSample^1$: we pick
\begin{align*}
\classifierEst_{\mathrm{ATL}}  \in \argmin_{\classifier \in \potentialFunctionSet^{\sourceDistribution}\cup \potentialFunctionSet^{\targetDistribution}} \sum_{i = \numTargetHalf+1}^{\numTarget} \one_{\{\classifier(X^{\targetDistribution}_i)\neq Y^{\targetDistribution}_i\}}. 
\end{align*}
The following theorem provides a high-probability bound on the performance of $\classifierEst_{\mathrm{ATL}}$ over~$\mathcal{P}_{\theta^\sharp}$: 
\begin{theorem}
\label{Thm:UpperBound}
Fix $\theta^\sharp = (\Delta,\expansivityConstant,L^*,\theta) \in \Theta^\sharp$ with $\holderExponent/(2\holderExponent+\sourceDimension)<\sourceTailExponent$ and $\holderExponent/(2\holderExponent+\targetDimension)<\targetTailExponent$. Given $\numSource \in \mathbb{N}_0$, $\numTarget \geq 2$ and $\delta \in (0,1)$, we let
\begin{align*}
\firstTerm_{\numSource,\numTarget,\delta} &:= \biggl(\frac{a_{0,\delta}}{\expansivityConstant^2 \cdot \numSource}\biggr)^{\frac{\holderExponent \sourceTailExponent (1+\marginExponent) }{\sourceTailExponent(2\holderExponent+\sourceDimension)+\marginExponent \holderExponent }} + \min\biggl\{\biggl(  \frac{{L^*} a_{1,\delta}}{ \numTarget}\biggr)^{\frac{1+\marginExponent}{2+\marginExponent}},(1-\expansivityConstant)^{1 + \marginExponent}\biggr\} + \biggl(\frac{\Delta}{\expansivityConstant}\biggr)^{1+\marginExponent}, \\
\secondTerm_{\numTarget,\delta} &:= \biggl(\frac{b_{\delta}}{\numTarget}\biggr)^{\frac{\holderExponent \targetTailExponent (1+\marginExponent) }{\targetTailExponent(2\holderExponent+\targetDimension)+\marginExponent \holderExponent}}, \qquad D_{\numSource,\numTarget,\delta} := \biggl(\frac{d_{\delta}}{\numTarget}\biggr)^{\frac{1+\marginExponent }{2+\marginExponent}},
\end{align*}
where $a_{0,\delta} := \log_+ (\numSource/\delta)$, $a_{1,\delta} := \log_+\bigl( L^* \ambientDimension\numSource/\delta\bigr)$, $b_{\delta} := \log_+ (\numTarget/\delta)$ and $d_{\delta} := \log_+ \bigl((\numSource + \numTarget)/\delta\bigr)$.  Then there exists $C_\theta > 0$, depending only on $\theta$, such that 
\begin{align}
\label{Eq:UpperBound}
\sup_{(P,Q) \in \mathcal{P}_{\theta^\sharp}} \mathbb{P}\biggl\{\excessRisk(\classifierEst_{\mathrm{ATL}}) > C_\theta \cdot \Bigl(\min \bigl( \firstTerm_{\numSource,\numTarget,\delta},\secondTerm_{\numTarget,\delta}\bigr) + D_{\numSource,\numTarget,\delta}\Bigr)\biggr\} \leq \delta.
\end{align}
\end{theorem}
An important point to note is that the definition of $\classifierEst_{\mathrm{ATL}}$ does not depend on the confidence level $\delta$, yet the probabilistic guarantee in~\eqref{Eq:UpperBound} holds simultaneously over all such levels.  The terms $\firstTerm_{\numSource,\numTarget,\delta}$ and $\secondTerm_{\numTarget,\delta}$ are very closely related to $\firstTerm_{\numSource,\numTarget}^{\mathrm{U}}$ and $\secondTerm_{\numTarget}^{\mathrm{U}}$ in Theorem~\ref{Thm:Main}; indeed, the only changes are in the logarithmic factor.  Integrating the tail probability bound~\eqref{Eq:UpperBound} over $\delta \in (0,1)$ therefore reveals that in the primary regimes of interest, the upper bound in Theorem~\ref{Thm:Main} can be attained using an algorithm that is agnostic to $\Delta$, $\expansivityConstant$ and $L^*$.  Comparing Theorem~\ref{Thm:UpperBound} with the upper bound in Theorem~\ref{Thm:Main}, we see that there is an additional term $D_{\numSource,\numTarget,\delta}$.  This term only contributes when $\expansivityConstant$ is extremely close to 1 (i.e.~when $1-\expansivityConstant \ll \numTarget^{-1/(2+\marginExponent)}$ up to a logarithmic factor) and $\firstTerm_{\numSource,\numTarget,\delta} \ll \secondTerm_{\numTarget,\delta}$.  In this case, $\sourceRegressionFunction$ is very close to $\targetRegressionFunction$, and the upper bound in Theorem~\ref{Thm:Main} can be attained by applying a standard nearest-neighbour method to $\sourceSample$. 

We recall that the second term $\secondTerm_{\numTarget,\delta}$ in~\eqref{Eq:UpperBound} arises from ignoring $\sourceSample$ and performing  classification using $\targetSample$.  Our analysis here builds on prior work on error rates in $k$-nearest neighbour classification \citep[e.g.][]{kulkarni1995rates,hall2008choice,samworth2012optimal,chaudhuri2014rates,biau2015lectures,gadat2016classification,reeve2017minimax,cannings2020local}; see also the seminal early work by \citet{fix1951discriminatory}, \citet{cover1967nearest} and \citet{stone1977consistent}.  The main novelty in our arguments, however, is in obtaining the $\firstTerm_{\numSource,\numTarget,\delta}$ term, which quantifies the extent to which our algorithm can exploit $\sourceSample$ to classify data from $\targetDistribution$ more accurately than can be done with $\targetSample$ alone.  Here, we combine analyses of nearest neighbour classification (but using $\sourceSample$ instead of $\targetSample$) with a covering number argument for the number of possible decision trees on $\sourceSample$ \citep{scott06ddt,biau2013cellular,wager2015adaptive}, allowing for an approximation error.

\section{Conclusion}
\label{Sec:Conclusion}

In this paper, we have argued that transfer learning has great potential for practitioners in the modern data-rich era.  Frequently, there is an abundance of data that, while not arising from the target population, are still able to provide useful information about inferential questions of interest.  We have introduced a general framework to study this phenomenon in the context of binary classification, and have derived the optimal rates of convergence in this setting.  Moreover, we have shown that these optimal rates are attainable by a fully adaptive algorithm that takes only our source and target data as inputs.

The scope of transfer learning is very wide indeed, encompassing not only other forms of transfer relationship and data acquisition mechanisms, but also alternative learning tasks such as regression, density estimation and clustering.  We therefore look forward to future developments in this field. 

\section{Proofs of Theorem~\ref{Thm:UpperBound} and upper bound in Theorem~\ref{Thm:Main}}
\label{Sec:UpperBound}

The proof of Theorem~\ref{Thm:UpperBound} is split into two subsections: the first controls the contribution to the excess test error of a data-dependent classifier calibrated via a given decision tree, while the second handles the additional error incurred in choosing the decision tree and other tuning parameters via empirical risk minimisation.  Both subsections require several intermediate results.

\subsection{Excess test error of decision tree-calibrated nearest neighbour classifiers}

\newcommand{\optimalDecisionTreeFunction}{h^*}
\newcommand{\classifierWithOptimalDecisionTreeFunction}{\classifierEst_{\robustness^*,\optimalDecisionTreeFunction}^{\sourceDistribution}}
\newcommand{\classifierWithGeneralDecisionTreeFunction}{\classifierEst_{\robustness^*,h}^{\sourceDistribution}}
\newcommand{\classifierWithhzeroDecisionTreeFunction}{\classifierEst_{\robustness^*,h_0}^{\sourceDistribution}}

We introduce some additional terminology.  Given $\sigma > 0$, $L \in \N_0$ and  $\decisionTreeFunction \in \setOfDecisionTreeFunctionsL$, define $\classifierEst_{\robustness,\decisionTreeFunction}^{\sourceDistribution} \in \setofDDClassifiers$ by 
\begin{align}\label{defOfKnnClassifierForSingleDTreeAndRobustness}
\classifierEst_{\robustness,\decisionTreeFunction}^{\sourceDistribution}(x):=\mathbbm{1}_{\bigl\{\empiricalMargin^{\sourceDistribution}_{\hat{k},\decisionTreeFunction}(x) \geq 0\bigr\}},
\end{align}
where $\empiricalMarginSourceDecisionTreeK(\cdot)$ and $\hat{k} \equiv \lepskiChoiceOfKSource(\cdot)$ are defined in~\eqref{Eq:marginFuncSource} and~\eqref{Eq:lepskiKSource} respectively.  Note that $\classifierEst_{\robustness,\decisionTreeFunction}^{\sourceDistribution}(x)$ is measurable with respect to the sigma algebra generated by $\sourceSample$, for every $x \in \R^{\ambientDimension}$. Proposition~\ref{prop:lowRegretWithTheRightDecisionTree} below is the main result of this subsection, and provides a high-probability bound for the excess test error of $\classifierEst_{\robustness,\decisionTreeFunction}^{\sourceDistribution}$ for a particular choice of $\robustness$ and a general decision tree $\decisionTreeFunction$.  It will be applied three times in the proof of Theorem~\ref{Thm:UpperBound}.
\begin{proposition}\label{prop:lowRegretWithTheRightDecisionTree} Let $\numSource \in \N$.  Fix $\theta^\sharp = (\Delta,\expansivityConstant,L^*,\theta) \in \Theta^\sharp$, where $\theta = (\targetDimension,\targetTailExponent,\sourceDimension,  \sourceTailExponent,\TailConstant,\marginExponent,\marginConstant,\holderExponent,\holderConstant)$, with $\holderExponent/(2\holderExponent+\sourceDimension)<\sourceTailExponent$, and $(P,Q) \in \mathcal{P}_{\theta^\sharp}$.  For $\decisionTreeFunction \in \setOfDecisionTreeFunctionsLStar \cup \mathcal{H}_0$, let 
\[
\Delta_h := \Delta + \max_{\ell \in [L^*]} \sup_{x \in \X_\ell^*} |h(x) - g_{\ell}(1/2)|.
\]
Then there exists $\tilde{C}_{\theta} > 0$, depending only on $\theta$, such that for every $\delta \in (0,1)$, if we set $\robustness^*=  \min\bigl\{ \lceil 3 \log_+^{1/2}(\numSource/\delta)\rceil, \numSource\bigr\}$, then 
\begin{align}
\label{Eq:ProbBound}
\mathbb{P}\biggl[\excessRisk\bigl(\classifierWithGeneralDecisionTreeFunction\bigr) > \tilde{C}_{\theta}\biggl\{\biggl(\frac{\log_+ (\numSource/\delta)}{\expansivityConstant^2 \cdot \numSource}\biggr)^{\frac{\holderExponent \sourceTailExponent (1+\marginExponent) }{\marginExponent \holderExponent + \sourceTailExponent(2\holderExponent+\sourceDimension)}}+ \biggl(\frac{\Delta_h}{\expansivityConstant}\biggr)^{1+\marginExponent}\biggr\}\biggr] \leq \delta.
\end{align}
\end{proposition}
The first term in the probability bound in~\eqref{Eq:ProbBound} corresponds to the difficulty of estimating~$\sourceRegressionFunction$, while, in the second term, $\Delta_h$ quantifies the approximation error of the decision tree function $h$.  The proof of Proposition~\ref{prop:lowRegretWithTheRightDecisionTree} is given after several preliminary lemmas. 

For $\delta \in (0,1)$ and $x \in \R^d$ with $\sourceDensityFunction(x) > 0$, we define the event 
\begin{align*}
\neighboursAreCloseEventDelta:=\underset{4\log_+(\numSource/\delta)\leq k < \numSource \cdot \sourceDensityFunction(x)/2}{\underset{k \in [\numSource]}{\bigcap}}\biggl\{ \left\|X_{(k)}^{\sourceDistribution}(x)-x\right\|  \leq \left(\frac{2k}{\numSource \cdot \sourceDensityFunction(x)}\right)^{{1}/{\sourceDimension}}\biggr\}.
\end{align*}

\begin{lemma}\label{lemma:neighboursClose} Let $\numSource \in \N$ and $(P,Q) \in \mathcal{P}_{\theta^\sharp}$.  For $x \in \R^{\ambientDimension}$ with $\sourceDensityFunction(x)>0$, we have $\Prob\bigl(\neighboursAreCloseEventDelta^c\bigr) \leq \delta$.
\end{lemma}
\begin{proof} Suppose that $k \in [\numSource]$ satisfies $4\log_+(\numSource/\delta)\leq k < \numSource \cdot \sourceDensityFunction(x)/2$, and let $r \equiv r_k := \bigl\{2k/\bigl(\numSource \cdot \sourceDensityFunction(x)\bigr)\bigr\}^{1/\sourceDimension}$. Since $r<1$, we have
\begin{align*}
\sourceMarginalDistribution\bigl(B_r(x)\bigr) \geq \sourceDensityFunction(x) \cdot r^{\sourceDimension} = \frac{2k}{\numSource}.    
\end{align*}
Hence, by the multiplicative Chernoff bound \citep[Theorem 2.3(c)]{mcdiarmid1998concentration}, we have
\begin{align*}
\Prob\biggl\{\left\|X_{(k)}^{\sourceDistribution}(x)-x\right\|  > \left(\frac{2k}{\numSource \cdot \sourceDensityFunction(x)}\right)^{{1}/{\sourceDimension}}\biggr\} & \leq \Prob\biggl\{\sum_{i=1}^{\numSource} \mathbbm{1}_{\{X_i^\sourceDistribution \in B_r(x)\}} <k\biggr\}\\
&\leq \Prob\biggl\{\sum_{i=1}^{\numSource} \mathbbm{1}_{\{X_i^\sourceDistribution \in B_r(x)\}} < \frac{\numSource}{2} \cdot \sourceMarginalDistribution\bigl(B_r(x)\bigr)  \biggr\}\\
&\leq e^{-\numSource \cdot \sourceMarginalDistribution(B_r(x))/8} \leq e^{-k/4} \leq \frac{\delta}{\numSource}. 
\end{align*}
The conclusion of the lemma now follows by a union bound.
\end{proof}

\begin{lemma}\label{lemma:targetRegFunctionAtNeighboursClose} Let $\numSource \in \N$, let $(\sourceDistribution,\targetDistribution) \in \mathcal{P}_{\theta^\sharp}$ and let $x \in \R^{\ambientDimension}$ be such that $\sourceDensityFunction(x)>0$.  On the event $\neighboursAreCloseEventDelta$, we have that 
\begin{align}\label{conclusionOfLemmaTargetRegFunctionAtNeighboursClose}
\max_{i \in [k]}\left|\targetRegressionFunction\left(X_{(i)}^{\sourceDistribution}(x)\right)-\targetRegressionFunction(x)\right|  <\holderConstant \cdot \left(\frac{2\cdot \max\{k,\lceil 4 \log_+(\numSource/\delta)\rceil\}}{\numSource \cdot \sourceDensityFunction(x)}\right)^{{\holderExponent}/{\sourceDimension}}
\end{align}
for all $k \in [\numSource]$. 
\end{lemma}
\begin{proof}
First, if $\numSource \cdot \sourceDensityFunction(x)/2 \leq k \leq \numSource$ or $\lceil 4\log_+(\numSource/\delta)\rceil > \numSource$, then the result follows from the fact that the right-hand side of (\ref{conclusionOfLemmaTargetRegFunctionAtNeighboursClose}) is at least 1.   
Second, if $4\log_+(\numSource/\delta)\leq k < \numSource \cdot \sourceDensityFunction(x)/2$, then~(\ref{conclusionOfLemmaTargetRegFunctionAtNeighboursClose}) follows from the definition of $\neighboursAreCloseEventDelta$ combined with Assumption~\ref{holderContinuityAssumption}. Finally, if $k < \lceil 4\log_+(\numSource/\delta)\rceil \leq \numSource$, then on $\neighboursAreCloseEventDelta$,
\begin{align*}
\max_{i \in [k]} \left|\targetRegressionFunction\left(X_{(i)}^{\sourceDistribution}(x)\right)-\targetRegressionFunction(x)\right|
&\leq \max_{i \in [ \min\{\lceil 4\log_+(\numSource/\delta)\rceil, \numSource\}]}\left|\targetRegressionFunction\left(X_{(i)}^{\sourceDistribution}(x)\right)-\targetRegressionFunction(x)\right| \\
&\leq \holderConstant \cdot \left(\frac{2\cdot \lceil 4 \log_+(\numSource/\delta)\rceil}{\numSource \cdot \sourceDensityFunction(x)}\right)^{{\holderExponent}/{\sourceDimension}}\\
&= \holderConstant \cdot \left(\frac{2\cdot \max\{k,\lceil 4 \log_+(\numSource/\delta)\rceil\}}{\numSource \cdot \sourceDensityFunction(x)}\right)^{{\holderExponent}/{\sourceDimension}},
\end{align*}
where the second inequality follows from the first two cases applied to $\lceil 4\log_+(\numSource/\delta)\rceil$.
\end{proof}

For $\delta \in (0,1)$ and $x \in \R^d$, we now define another event
\begin{align*}
\hoeffdingTypeEventDelta:= \bigcap_{k=1}^{\numSource} \biggl\{ \frac{1}{k}\sum_{i=1}^k \left[Y_{(i)}^{\sourceDistribution}(x)- \sourceRegressionFunction\left(X_{(i)}^{\sourceDistribution}(x)\right)\right] \leq \sqrt{\frac{\log_+(\numSource/\delta)}{2k}}
\biggr\}.
\end{align*}

\begin{lemma}\label{hoeffdingTypeLemmaSourceSample} Let $\numSource \in \N$ and $(\sourceDistribution,\targetDistribution) \in \mathcal{P}_{\theta^\sharp}$.  For every $\delta \in (0,1)$ and $x \in \R^d$, we have $\Prob\bigl(\hoeffdingTypeEventDelta^c\bigr) \leq \delta$.
\end{lemma}
\begin{proof} First note that conditional on $(X^{\sourceDistribution}_i)_{i \in [\numSource]}$, the labels $Y_{(1)}^{\sourceDistribution}(x),\ldots,Y_{(\numSource)}^{\sourceDistribution}(x)$ are independent Bernoulli random variables with respective means $\sourceRegressionFunction\bigl(X_{(1)}^{\sourceDistribution}(x)\bigr),\ldots,\sourceRegressionFunction\bigl(X_{(\numSource)}^{\sourceDistribution}(x)\bigr)$. Hence, by Hoeffding's inequality, for each $k \in [\numSource]$,
\begin{align*}
\Prob\biggl\{\frac{1}{k}\sum_{i=1}^k \left[Y_{(i)}^{\sourceDistribution}(x)- \sourceRegressionFunction\left(X_{(i)}^{\sourceDistribution}(x)\right)\right] > \sqrt{\frac{\log_+(\numSource/\delta)}{2k}}\hspace{2mm}\bigg|\hspace{2mm} (X^{\sourceDistribution}_i)_{i \in [\numSource]}\biggr\} \leq \frac{\delta}{\numSource}.
\end{align*}
The conclusion of the lemma follows by taking expectations over $(X^{\sourceDistribution}_i)_{i \in [\numSource]}$, and then a union bound over $k \in [\numSource]$.
\end{proof}

\begin{lemma}\label{lemma:gapMustBeLargeToMakeAnError} Let $\numSource \in \N$ and $(P,Q) \in \mathcal{P}_{\theta^\sharp}$.  Suppose that $x\in \R^{\ambientDimension}$ and $\decisionTreeFunction \in \setOfDecisionTreeFunctionsLStar \cup \mathcal{H}_0$ satisfy
\begin{align*}
\biggl| \targetRegressionFunction(x)-\frac{1}{2}\biggr| \geq 50\cdot \holderConstant  \left(\frac{\log_+(\numSource/\delta)}{\expansivityConstant^2 \cdot \numSource \cdot \sourceDensityFunction(x)}\right)^{\frac{\holderExponent}{2\holderExponent+\sourceDimension}}+\frac{2\Delta_h}{\expansivityConstant},
\end{align*}
and let $\robustness^*=  \min\bigl\{ \lceil 3 \log_+^{1/2}(\numSource/\delta)\rceil, \numSource\bigr\}$.  Then, on the event $\neighboursAreCloseEventDelta \cap \hoeffdingTypeEventDelta$, we have $\classifierWithGeneralDecisionTreeFunction(x)=  \bayesClassifier(x)$.
\end{lemma}

\newcommand{\smoothnessExponentSource}{\lambda_{\sourceDistribution}}

\begin{proof} For the purpose of the proof, we let
\begin{align*}
\epsilon :=50\cdot \holderConstant  \left(\frac{\log_+(\numSource/\delta)}{\expansivityConstant^2 \cdot \numSource \cdot \sourceDensityFunction(x)}\right)^{\frac{\holderExponent}{2\holderExponent+\sourceDimension}}+\frac{2\Delta_h}{\expansivityConstant},
\end{align*}
so that $|\targetRegressionFunction(x)-1/2|\geq \epsilon$ (in particular, this means that $\epsilon \leq 1/2$).  We consider only the case where $ \targetRegressionFunction(x)-1/2 \geq \epsilon$, since the case where $\targetRegressionFunction(x)-1/2\leq - \epsilon$ follows by symmetry.  We further suppose throughout the proof that the event $\neighboursAreCloseEventDelta \cap \hoeffdingTypeEventDelta$ holds.  Let $\ell^*$ denote the leaf function corresponding to $\{\X_1^*,\ldots,\X_{L^*}^*\} \in \thresholdPartitions_{L^*}$ in Assumption~\ref{transferAssumption}.  Fixing $k \in [\numSource]$, 
it then follows from Assumption~\ref{transferAssumption} and Lemma~\ref{lemma:targetRegFunctionAtNeighboursClose} that for all $i \in [k]$, and $h \in \setOfDecisionTreeFunctionsLStar \cup \mathcal{H}_0$, 
\begin{align*}
\sourceRegressionFunction(X_{(i)}^{\sourceDistribution}) - h(X_{(i)}^{\sourceDistribution}) &= \sourceRegressionFunction(X_{(i)}^{\sourceDistribution}) - g_{\ell^*(X_{(i)}^{\sourceDistribution})}\bigl(\targetRegressionFunction\bigl(X_{(i)}^{\sourceDistribution})\bigr) \\
&\hspace{0.9cm}+ g_{\ell^*(X_{(i)}^{\sourceDistribution})}\bigl(\targetRegressionFunction\bigl(X_{(i)}^{\sourceDistribution})\bigr) - g_{\ell^*(X_{(i)}^{\sourceDistribution})}(1/2) + g_{\ell^*(X_{(i)}^{\sourceDistribution})}(1/2) - h(X_{(i)}^{\sourceDistribution}) \\
&\geq \expansivityConstant \cdot \bigl\{\targetRegressionFunction(X_{(i)}^{\sourceDistribution})-1/2\bigr\} - \Delta - \max_{\ell \in [L^*]} \sup_{x \in \X_\ell^*} |h(x) - g_{\ell}(1/2)| \\
&\geq \expansivityConstant \cdot \bigl\{\targetRegressionFunction(X_{(i)}^{\sourceDistribution})-1/2\bigr\} - \frac{\expansivityConstant \cdot \epsilon}{2} \\
&\geq \expansivityConstant \cdot \biggl\{\frac{\epsilon}{2}-\holderConstant \cdot \left(\frac{2\cdot \max\{k,\lceil 4 \log_+(\numSource/\delta)\rceil\}}{\numSource \cdot \sourceDensityFunction(x)}\right)^{{\holderExponent}/{\sourceDimension}}\biggr\}.
\end{align*}
We deduce that
\begin{align}\label{eq:lowerBoundOnEmpiricalMarginForOptDecisionTreeVariableK}
\empiricalMargin_{k,h}^{\sourceDistribution}(x)&\equiv \frac{1}{k}\sum_{i=1}^k \bigl\{ Y_{(i)}^{\sourceDistribution}(x)- h(X_{(i)}^{\sourceDistribution}) \bigr\} \nonumber\\
& \geq \expansivityConstant \cdot \biggl\{\frac{\epsilon}{2}-\holderConstant \cdot \left(\frac{2\cdot \max\{k,\lceil 4 \log_+(\numSource/\delta)\rceil\}}{\numSource \cdot \sourceDensityFunction(x)}\right)^{{\holderExponent}/{\sourceDimension}}\biggr\} - \sqrt{\frac{\log_+(\numSource/\delta)}{2k}}
\end{align}
for all $k \in [\numSource]$.  Now define
\begin{align*}
k^* &:= \min \biggl\{k \in \N: \expansivityConstant\cdot \holderConstant \cdot \left(\frac{2k}{\numSource \cdot \sourceDensityFunction(x)}\right)^{{\holderExponent}/{\sourceDimension}} \geq  \sqrt{\frac{\log_+(\numSource/\delta)}{k}}\biggr\}.
\end{align*}
Since $\epsilon \leq 1/2$, we have $\numSource \geq \numSource \cdot \sourceDensityFunction(x) \geq (100C_S)^{(2\holderExponent+\sourceDimension)/\holderExponent} \expansivityConstant^{-2} \log_+(\numSource/\delta) \geq (100C_S)^{(2\holderExponent+\sourceDimension)/\holderExponent} \log_+(\numSource/\delta)$, so
\[
\lceil 4 \log_+(\numSource/\delta)\rceil < k^* \leq \bigg\lceil \biggl( \frac{ \log_+(\numSource/\delta)}{\expansivityConstant^2}\biggr)^{\frac{\sourceDimension}{2\holderExponent+\sourceDimension}}\cdot \bigl(\numSource \cdot \sourceDensityFunction(x)\bigr)^{\frac{2\holderExponent}{2\holderExponent+\sourceDimension}}\bigg\rceil \leq \numSource.
\]
Moreover, we also see that $\robustness^*= \min\bigl\{ \lceil 3 \log_+^{1/2}(\numSource/\delta)\rceil, \numSource\bigr\}= \lceil 3 \log_+^{1/2}(\numSource/\delta)\rceil$.  Hence by (\ref{eq:lowerBoundOnEmpiricalMarginForOptDecisionTreeVariableK}), we have
\begin{align}\label{Eq:kHatLeqKStar}
\empiricalMargin_{k^*,h}^{\sourceDistribution}(x) -\frac{\robustness^*}{\sqrt{k^*}} &>  \expansivityConstant \cdot \biggl\{\frac{\epsilon}{2}-\holderConstant \cdot \left(\frac{2k^*}{\numSource \cdot \sourceDensityFunction(x)}\right)^{{\holderExponent}/{\sourceDimension}}\biggr\}- 5\sqrt{\frac{\log_+(\numSource/\delta)}{k^*}} \nonumber\\
&\geq \expansivityConstant \cdot \biggl\{\frac{\epsilon}{2}-6\holderConstant \cdot \left(\frac{2k^*}{\numSource \cdot \sourceDensityFunction(x)}\right)^{{\holderExponent}/{\sourceDimension}}\biggr\} \nonumber\\
& = \expansivityConstant \cdot \biggl\{\frac{\epsilon}{2}-24 \holderConstant \cdot  \left(\frac{\log_+(\numSource/\delta)}{\expansivityConstant^2 \cdot \numSource \cdot \sourceDensityFunction(x)}\right)^{\frac{\holderExponent}{2\holderExponent+\sourceDimension}} \biggr\} > 0.
\end{align}
We conclude that $\hat{k} \equiv \hat{k}_{\robustness^*,h}^{\sourceDistribution}(x) \leq k^*$.  Moreover, by applying (\ref{eq:lowerBoundOnEmpiricalMarginForOptDecisionTreeVariableK}) once again, we have for $k < k^*$ that
\begin{align}\label{eq:analyseKLeqKStar}
\empiricalMargin_{k,h}^{\sourceDistribution}(x)&+\frac{\robustness^*}{\sqrt{k}} \nonumber \\
& \geq  \expansivityConstant \cdot \biggl\{\frac{\epsilon}{2}-\holderConstant \cdot \left(\frac{2\cdot \max\{k,\lceil 4 \log_+(\numSource/\delta)\rceil\}}{\numSource \cdot \sourceDensityFunction(x)}\right)^{{\holderExponent}/{\sourceDimension}}\biggr\}- \sqrt{\frac{\log_+(\numSource/\delta)}{k}}+\frac{\robustness^*}{\sqrt{k}} \nonumber\\
& >  
- \sqrt{\frac{\log_+(\numSource/\delta)}{\max\{k,\lceil 4 \log_+(\numSource/\delta)\rceil\}}} - \sqrt{\frac{\log_+(\numSource/\delta)}{k}} + \frac{\robustness^*}{\sqrt{k}}\geq 0.
\end{align}
But by definition of $\hat{k}$, we have $\bigl|\empiricalMargin_{\hat{k},h}^{\sourceDistribution}(x)\bigr| > \robustness^*/\hat{k}^{1/2}$, so from~\eqref{Eq:kHatLeqKStar} and \eqref{eq:analyseKLeqKStar}, we deduce that $\empiricalMargin_{\hat{k},h}^{\sourceDistribution}(x) > 0$, and hence that $\classifierWithGeneralDecisionTreeFunction(x)= \one_{\{ \empiricalMargin_{\hat{k},h}^{\sourceDistribution}(x) \geq 0\}} =1= \bayesClassifier(x)$, as required. 
\end{proof}

\newcommand{\sourceKnnPeformingWellSubset}{A_{\delta}(\sourceSample)}

We now define a random subset $\sourceKnnPeformingWellSubset$ of $\R^{\ambientDimension}$ by 
\begin{align*}
\sourceKnnPeformingWellSubset:= \Bigl\{ x \in \R^{\ambientDimension}: \neighboursAreCloseEventNoDelta^{\delta/(2\numSource^{1+\marginExponent})}(x) \cap \hoeffdingTypeEventNoDelta^{\delta/(2\numSource^{1+\marginExponent})}(x) \ \text{holds}\Bigr\}.
\end{align*}

\begin{lemma}\label{lemma:sourceKnnPeformingWellSubsetIsLarge} Let $\numSource \in \N$ and $(P,Q) \in \mathcal{P}_{\theta^\sharp}$.  We have $\Prob\bigl\{ \targetMarginalDistribution\bigl( { \sourceKnnPeformingWellSubset}^c\bigr) \geq 1/\numSource^{1+\marginExponent} \bigr\} \leq  \delta$.
\end{lemma}
\begin{proof} By Markov's inequality and Fubini's theorem, as well as Lemmas~\ref{lemma:neighboursClose} and~\ref{hoeffdingTypeLemmaSourceSample}, we have
\begin{align*}
\Prob\bigl\{ \targetMarginalDistribution\bigl( { \sourceKnnPeformingWellSubset}^c\bigr) &\geq 1/\numSource^{1+\marginExponent} \bigr\} \leq \numSource^{1+\marginExponent} \cdot \E\bigl\{ \targetMarginalDistribution\bigl( { \sourceKnnPeformingWellSubset}^c\bigr)  \bigr\}\\ &= \numSource^{1+\marginExponent} \cdot \E\int_{\R^{\ambientDimension}} \mathbbm{1}_{\{x \in \sourceKnnPeformingWellSubset^c \}} \, d\targetMarginalDistribution(x) \\
& \leq \numSource^{1+\marginExponent} \int_{\R^{\ambientDimension}} \Bigl[ \Prob\bigl\{\neighboursAreCloseEventNoDelta^{\delta/(2\numSource^{1+\marginExponent})}(x)^c\bigr\} +\Prob\bigl\{\hoeffdingTypeEventNoDelta^{\delta/(2\numSource^{1+\marginExponent})}(x)^c\bigr\}\Bigr] \, d\targetMarginalDistribution(x) \leq \delta,
\end{align*}
as required.
\end{proof}

\newcommand{\tailThresholdForProof}{t_0}

We are now ready to provide the proof of Proposition~ \ref{prop:lowRegretWithTheRightDecisionTree}.
\begin{proof}[Proof of Proposition \ref{prop:lowRegretWithTheRightDecisionTree}] We begin by introducing some further notation for the proof. Let
\begin{align*}
\Lambda_0&:= 100 \cdot \holderConstant \cdot \left(\frac{\log_+(2\numSource^{2+{\marginExponent}}/\delta)}{\expansivityConstant^2 \cdot \numSource}\right)^{\frac{\holderExponent}{2\holderExponent+\sourceDimension}},\\
\tailThresholdForProof&:= \min\biggl\{ \left(\frac{\Lambda_0\expansivityConstant}{4\Delta_h}\right)^{\frac{2\holderExponent+\sourceDimension}{\holderExponent}},\Lambda_0^{\frac{\marginExponent(2\holderExponent+\sourceDimension)}{\marginExponent\holderExponent+\sourceTailExponent(2\holderExponent+\sourceDimension)}} \biggr\}. 
\end{align*}
We then generate a countable partition $\left(\mathcal{T}_j\right)_{j \in \N_0}$ of $\sourceKnnPeformingWellSubset$ by
\begin{align*}
\mathcal{T}_0&:=\left\lbrace x \in \sourceKnnPeformingWellSubset: \sourceDensityFunction(x) \geq t_0\right\rbrace,\\
\mathcal{T}_j&:=\left\lbrace x \in \sourceKnnPeformingWellSubset: 2^{-j} \cdot t_0 \leq \sourceDensityFunction(x) < 2^{-(j-1)}\cdot t_0\right\rbrace,
\end{align*}
for $j \in \N$. By Lemma \ref{lemma:gapMustBeLargeToMakeAnError} for each $j \in \N_0$ we have
\begin{align*}
\left| 2\targetRegressionFunction(x)-1\right| \cdot \one_{\left\lbrace x \in \mathcal{T}_j: \classifierWithGeneralDecisionTreeFunction(x) \neq   \bayesClassifier(x)\right\rbrace} &< \Lambda_0 \cdot (2^{-j} \cdot t_0)^{-\frac{\holderExponent}{2\holderExponent+\sourceDimension}}+\frac{4\Delta_h}{\expansivityConstant} \\
&\leq 2\Lambda_0 \cdot (2^{-j} \cdot t_0)^{-\frac{\holderExponent}{2\holderExponent+\sourceDimension}},
\end{align*}
where the second inequality uses $t_0 \leq \bigl(\frac{\Lambda_0\expansivityConstant}{4\Delta_h}\bigr)^{\frac{2\holderExponent+\sourceDimension}{\holderExponent}}$. By Assumption \ref{marginAssumption} we have
\begin{align*}
\int_{\left\lbrace x \in \mathcal{T}_0: \classifierWithGeneralDecisionTreeFunction(x) \neq   \bayesClassifier(x)\right\rbrace}\left| 2\targetRegressionFunction(x)-1\right| \, &d\targetMarginalDistribution(x) \leq 2^{1+\marginExponent}\marginConstant \cdot {\Lambda_0}^{1+\marginExponent}  \cdot  t_0^{-\frac{\holderExponent({1+\marginExponent})}{2\holderExponent+\sourceDimension}} \\
&\leq 2^{1+\marginExponent}\marginConstant \cdot \max\biggl\{  \left(\frac{4\Delta_h}{\expansivityConstant}\right)^{1+\marginExponent} ,{\Lambda_0}^{\frac{\sourceTailExponent(1+\marginExponent)(2\holderExponent+\sourceDimension)}{\marginExponent \holderExponent+\sourceTailExponent(2\holderExponent+\sourceDimension)}} \biggr\}.
\end{align*}
On the other hand, by Assumption \ref{sourceDensityAssumption} and the assumption that $\sourceTailExponent-\frac{\holderExponent}{2\holderExponent+\sourceDimension} > 0$, for $j \in \N$ we have 
\begin{align*}
\int_{\left\lbrace x \in \mathcal{T}_j: \classifierWithGeneralDecisionTreeFunction(x) \neq   \bayesClassifier(x)\right\rbrace}\left| 2\targetRegressionFunction(x)-1\right| \, & d\targetMarginalDistribution(x) \leq  (2^{1+\sourceTailExponent}\TailConstant) \cdot \Lambda_0 \cdot (2^{-j} \cdot t_0)^{\sourceTailExponent-\frac{\holderExponent}{2\holderExponent+\sourceDimension}}\\
& \leq  (2^{1+\sourceTailExponent}\TailConstant) \cdot {\Lambda_0}^{\frac{\sourceTailExponent(1+\marginExponent)(2\holderExponent+\sourceDimension)}{\marginExponent \holderExponent+\sourceTailExponent(2\holderExponent+\sourceDimension)}} \cdot 2^{-j\bigl(\sourceTailExponent - \frac{\holderExponent}{2\holderExponent+\sourceDimension}\bigr)}.
\end{align*}
Putting the above together, we have
\begin{align*}
\excessRisk\bigl(\classifierWithGeneralDecisionTreeFunction\bigr)&= \int_{\left\lbrace x \in \R^{\ambientDimension}: \classifierWithGeneralDecisionTreeFunction(x) \neq   \bayesClassifier(x)\right\rbrace}\left| 2\targetRegressionFunction(x)-1\right| \, d\targetMarginalDistribution(x) \\& \leq \targetMarginalDistribution\bigl(\sourceKnnPeformingWellSubset^c\bigr) +\sum_{j = 0}^{\infty}\int_{\left\lbrace x \in \mathcal{T}_j: \classifierWithGeneralDecisionTreeFunction(x) \neq   \bayesClassifier(x)\right\rbrace}\left| 2\targetRegressionFunction(x)-1\right| \, d\targetMarginalDistribution(x)\\& \leq \targetMarginalDistribution\bigl(\sourceKnnPeformingWellSubset^c\bigr)+2^{1+\marginExponent}\marginConstant \cdot \left(\frac{4\Delta_h}{\expansivityConstant}\right)^{1+\marginExponent}\\& \hspace{2cm}+\biggl\{2^{1+\marginExponent}\marginConstant +(2^{1+\sourceTailExponent}\TailConstant) \cdot\sum_{j=1}^{\infty}2^{-j\bigl(\sourceTailExponent - \frac{\holderExponent}{2\holderExponent+\sourceDimension}\bigr)}\biggr\}\cdot{\Lambda_0}^{\frac{\sourceTailExponent(1+\marginExponent)(2\holderExponent+\sourceDimension)}{\marginExponent \holderExponent+\sourceTailExponent(2\holderExponent+\sourceDimension)}}\\
& \leq \targetMarginalDistribution\bigl(\sourceKnnPeformingWellSubset^c\bigr) +\frac{\tilde{C}_{\theta}}{2}\biggl\{\biggl(\frac{\log_+ (\numSource/\delta)}{\expansivityConstant^2 \cdot \numSource}\biggr)^{\frac{\holderExponent \sourceTailExponent (1+\marginExponent) }{\marginExponent \holderExponent + \sourceTailExponent(2\holderExponent+\sourceDimension)}}+ \biggl(\frac{\Delta_h}{\expansivityConstant}\biggr)^{1+\marginExponent}\biggr\},
\end{align*}
where $\tilde{C}_{\theta} \geq 2$ depends only on $\theta$. Note that the final inequality again uses the hypothesis that $\sourceTailExponent > \holderExponent/(2\holderExponent+\sourceDimension)$. By Lemma  \ref{lemma:sourceKnnPeformingWellSubsetIsLarge}, it now follows that 
\begin{align*}
\Prob\Biggl[ \excessRisk\bigl(\classifierWithGeneralDecisionTreeFunction\bigr)> \frac{1}{\numSource^{1+\marginExponent}}+\frac{\tilde{C}_{\theta}}{2}\biggl\{\biggl(\frac{\log_+ (\numSource/\delta)}{\expansivityConstant^2 \cdot \numSource}\biggr)^{\frac{\holderExponent \sourceTailExponent (1+\marginExponent) }{\marginExponent \holderExponent + \sourceTailExponent(2\holderExponent+\sourceDimension)}}+ \biggl(\frac{\Delta_h}{\expansivityConstant}\biggr)^{1+\marginExponent}\biggr\}\Biggr] \leq \delta.
\end{align*}
Since ${\frac{\holderExponent \sourceTailExponent}{\marginExponent \holderExponent + \sourceTailExponent(2\holderExponent+\sourceDimension)}} \leq 1$,  the conclusion follows.
\end{proof}
We now give the three different instantiations of Proposition~\ref{prop:lowRegretWithTheRightDecisionTree} mentioned above. 
\begin{corollary}
Let $\numSource \in \N$.  Fix $\theta^\sharp = (\Delta,\expansivityConstant,L^*,\theta) \in \Theta^\sharp$, where $\theta = (\targetDimension,\targetTailExponent,\sourceDimension,  \sourceTailExponent,\TailConstant,\marginExponent,\marginConstant,\holderExponent,\holderConstant)$, with $\holderExponent/(2\holderExponent+\sourceDimension)<\sourceTailExponent$, and $(P,Q) \in \mathcal{P}_{\theta^\sharp}$.  Taking $\tilde{C}_{\theta} \geq 2$ from Proposition~\ref{prop:lowRegretWithTheRightDecisionTree}, there exists $\optimalDecisionTreeFunction \in \setOfDecisionTreeFunctionsLStar$ such that for every $\delta \in (0,1)$, if we set $\robustness^*= \min\bigl\{ \lceil 3 \log_+^{1/2}(\numSource/\delta)\rceil, \numSource\bigr\}$, then
\begin{align*}
\mathbb{P}\biggl[\excessRisk\bigl(\classifierWithOptimalDecisionTreeFunction\bigr) > (2^{\marginExponent} + 1)\tilde{C}_{\theta}\biggl\{\biggl(\frac{\log_+ (\numSource/\delta)}{\expansivityConstant^2 \cdot \numSource}\biggr)^{\frac{\holderExponent \sourceTailExponent (1+\marginExponent) }{\marginExponent \holderExponent + \sourceTailExponent(2\holderExponent+\sourceDimension)}}+ \biggl(\frac{\Delta}{\expansivityConstant}\biggr)^{1+\marginExponent}\biggr\}\biggr] \leq \delta.
\end{align*}
\end{corollary}
\begin{proof}
The decision tree function $\optimalDecisionTreeFunction \in \setOfDecisionTreeFunctionsLStar$ is constructed as follows. Recall that $\{\X_1^*,\ldots,\X_{L^*}^*\} \in \thresholdPartitions_{L^*}$ denotes the decision tree partition given by Assumption~\ref{transferAssumption}, with corresponding transfer functions $g_1,\ldots,g_{L^*}:[0,1]\rightarrow [0,1]$ and leaf function $\ell^*:\R^d \rightarrow [L^*]$.  Fix $(\tau_1^*,\ldots,\tau_{L^*}^*) \in \{0,1/\numSource,2/\numSource, \ldots, 1\}^{L^*}$ such that $|\tau_{\ell}^*-g_{\ell}(1/2)|\leq 1/\numSource$ for each $\ell \in [L^*]$, and define $\optimalDecisionTreeFunction:\R^{\ambientDimension} \rightarrow (0,1)$ by  $\optimalDecisionTreeFunction(x) := \tau^*_{\ell^*(x)}$.  When $\expansivityConstant^2 \cdot \numSource \geq 1$, the claim now follows from Proposition~\ref{prop:lowRegretWithTheRightDecisionTree}, noting that
\begin{align*}
\Bigl(\frac{\Delta + 1/\numSource}{\expansivityConstant}\Bigr)^{1+\marginExponent} &\leq 2^{\marginExponent}\biggl\{\Bigl(\frac{\Delta}{\expansivityConstant}\Bigr)^{1+\marginExponent} + \Bigl(\frac{1}{\expansivityConstant^2 \cdot \numSource}\Bigr)^{(1+\marginExponent)/2}\biggr\} \\
&\leq 2^{\marginExponent}\Bigl(\frac{\Delta}{\expansivityConstant}\Bigr)^{1+\marginExponent} + 2^{\marginExponent} \biggl(\frac{\log_+ (\numSource/\delta)}{\expansivityConstant^2 \cdot \numSource}\biggr)^{\frac{\holderExponent \sourceTailExponent (1+\marginExponent) }{\marginExponent \holderExponent + \sourceTailExponent(2\holderExponent+\sourceDimension)}}.
\end{align*}
But if $\expansivityConstant^2 \cdot \numSource < 1$, then the claim follows from the fact that $\excessRisk\bigl(\classifierWithOptimalDecisionTreeFunction\bigr) \leq 1$.
\end{proof}
\begin{corollary}\label{justUseHalfCor}
Let $\numSource \in \N$.  Fix $\theta^\sharp = (\Delta,\expansivityConstant,L^*,\theta) \in \Theta^\sharp$, where $\theta = (\targetDimension,\targetTailExponent,\sourceDimension,  \sourceTailExponent,\TailConstant,\marginExponent,\marginConstant,\holderExponent,\holderConstant)$, with $\holderExponent/(2\holderExponent+\sourceDimension)<\sourceTailExponent$, and $(P,Q) \in \mathcal{P}_{\theta^\sharp}$.  Taking $\tilde{C}_{\theta} \geq 2$ from Proposition~\ref{prop:lowRegretWithTheRightDecisionTree}, and writing $h_0$ for the unique element of $\mathcal{H}_0$, for every $\delta \in (0,1)$, if we set $\robustness^*=  \min\bigl\{ \lceil 3 \log_+^{1/2}(\numSource/\delta)\rceil, \numSource\bigr\}$, then
\begin{align*}
\mathbb{P}\biggl[\excessRisk\bigl(\classifierWithhzeroDecisionTreeFunction\bigr) > 2^{\marginExponent}\tilde{C}_{\theta}\biggl\{\biggl(\frac{\log_+ (\numSource/\delta)}{\expansivityConstant^2 \cdot \numSource}\biggr)^{\frac{\holderExponent \sourceTailExponent (1+\marginExponent) }{\marginExponent \holderExponent + \sourceTailExponent(2\holderExponent+\sourceDimension)}}+ \Bigl(\frac{\Delta}{\expansivityConstant}\Bigr)^{1+\marginExponent} + (1 - \expansivityConstant)^{1+\marginExponent}\biggr\}\biggr] \leq \delta.
\end{align*}
\end{corollary}
\begin{proof} 
Observe that, by Assumption~\ref{transferAssumption}, for any $\ell \in [L^*]$,
\begin{align*}
g_\ell(1/2) - 1/2 &\leq - \{g_\ell(1) - g_\ell(1/2)\} + 1/2 \leq (1-\expansivityConstant)/2, \\
1/2 - g_\ell(1/2) &\leq -\{g_\ell(1/2) - g_\ell(0)\} + 1/2 \leq (1-\expansivityConstant)/2.
\end{align*}
Hence
\[
\Delta_{h_0} = \Delta + \max_{\ell \in [L^*]} |g_\ell(1/2) - 1/2| \leq \Delta+\frac{1-\expansivityConstant}{2}.
\]
We assume without loss of generality that $\expansivityConstant\geq 1/2$, since otherwise the conclusion follows from the facts that $\excessRisk\bigl(\classifierWithhzeroDecisionTreeFunction\bigr) \leq 1$ and $\tilde{C}_{\theta} \geq 2$.
The result then follows by Proposition \ref{prop:lowRegretWithTheRightDecisionTree}.
\end{proof}

Now, for $\theta^\flat = (\targetDimension,\targetTailExponent,\TailConstant,\marginExponent,\marginConstant,\holderExponent,\holderConstant)$, let $\mathcal{Q}_{\theta^\flat}$ denote the set of distributions $Q$ on $\mathbb{R}^{\ambientDimension} \times \{0,1\}$ that satisfy Assumptions~\ref{marginAssumption},~\ref{holderContinuityAssumption} and the first part of Assumption~\ref{sourceDensityAssumption} with parameter~$\theta^\flat$.
\begin{corollary}\label{knnOnFirstHalfTargetDataCor}
Let $\numTarget \in \N$.  Fix $\theta^\flat = (\targetDimension,\targetTailExponent,\TailConstant,\marginExponent,\marginConstant,\holderExponent,\holderConstant)$ with $\holderExponent/(2\holderExponent+\targetDimension)<\targetTailExponent$, and $Q \in \mathcal{Q}_{\theta^\flat}$.  There exists $\tilde{C}_{\theta^\flat} \geq 2$, depending only on $\theta^\flat$, such that for every $\delta \in (0,1)$, if we set $\tilde{\robustness}=  \min\bigl\{ \lceil 3 \log_+^{1/2}(\numTarget/\delta)\rceil, \numTarget\bigr\}$, then
\begin{align*}
\mathbb{P}\biggl\{\excessRisk\bigl(\classifierEst^{\targetDistribution}_{\tilde{\robustness}}\bigr) > \tilde{C}_{\theta^\flat}\biggl(\frac{\log_+ (\numTarget/\delta)}{\numTarget}\biggr)^{\frac{\holderExponent \targetTailExponent (1+\marginExponent) }{\marginExponent \holderExponent + \targetTailExponent(2\holderExponent+\targetDimension)}}\biggr\} \leq \delta.
\end{align*}
\end{corollary}
\begin{proof}
The result follows from Corollary~\ref{justUseHalfCor}, with $\targetDistribution$ in place of $\sourceDistribution$, with $\targetSample^0$ in place of $\sourceSample$, and with $\Delta = 0$ and $\phi = 1$.
\end{proof}

\subsection{Empirical risk minimisation}

\newcommand{\ermFiniteHypothesisSet}{\setOfClassifiers}
\newcommand{\optimalHypothesis}{f^{*}}
\newcommand{\ermSample}{\sample}
\newcommand{\ermSampleSize}{n}
\newcommand{\ermHypothesis}{\hat{f}_{\ermSampleSize}}

In this section we control the additional error incurred by selecting our decision tree $\decisionTreeFunction$ and robustness parameter $\robustness$ by empirical risk minimisation.  Our analysis will make use of the following result, similar versions of which are well known \citep[e.g.,][]{tsybakov2004optimal}, but we include a proof for completeness.
\begin{proposition}\label{ermWithMarginCondition} Suppose that Assumption \ref{marginAssumption} holds and let $\ermFiniteHypothesisSet$ be a non-empty, finite set of classifiers and let $\optimalHypothesis \in \argmin_{f \in \ermFiniteHypothesisSet} \risk(f)$. Let $(X_1,Y_1),\ldots,(X_{\ermSampleSize},Y_{\ermSampleSize})$ be independent pairs with distribution~$\targetDistribution$, and let $\ermHypothesis \in \argmin_{f \in \ermFiniteHypothesisSet} \sum_{i=1}^{\ermSampleSize} \mathbbm{1}_{\{ f(X_i) \neq Y_i\}}$.  Then, for any $\delta \in (0,1]$,
\begin{align*}
\Prob\biggl\{ \excessRisk(\ermHypothesis)>2\excessRisk( \optimalHypothesis)+ 64\marginConstant^{\frac{1}{2+\marginExponent}} \biggl(\frac{\log(2\left|\ermFiniteHypothesisSet\right|/\delta)}{\ermSampleSize} \biggr)^{\frac{1+\marginExponent}{2+\marginExponent}}\biggr\} \leq \delta.
\end{align*}
\end{proposition}
\begin{proof}
By Assumption~\ref{marginAssumption}, for every $\epsilon>0$ and $f \in \ermFiniteHypothesisSet$,
\begin{align*}
\excessRisk(\classifier) &= \int_{\{x \in \R^{\ambientDimension}: \classifier(x) \neq \bayesClassifier(x)\}} |2 \targetRegressionFunction(x) - 1| \, d\targetMarginalDistribution(x) \\
&\geq 2\epsilon \cdot \targetMarginalDistribution\left( \left\lbrace x \in \R^{\ambientDimension}: \classifier(x) \neq \bayesClassifier(x) \text{ and }\left|\targetRegressionFunction(x) - {1}/{2}\right| \geq \epsilon \right\rbrace \right) \\ 
&\geq 2\epsilon \cdot \Bigl\{\targetMarginalDistribution\left( \left\lbrace x \in \R^{\ambientDimension}: \classifier(x) \neq \bayesClassifier(x) \right\rbrace\right)-\targetMarginalDistribution\left( \left\lbrace x \in \R^{\ambientDimension}:\left|\targetRegressionFunction(x)- {1}/{2}\right| <\epsilon \right\rbrace \right)\Bigr\} \\ &\geq 2 \epsilon \cdot \Bigl\{\targetMarginalDistribution\bigl( \bigl\{ x \in \R^{\ambientDimension}: \classifier(x) \neq \bayesClassifier(x) \bigr\}\bigr)-\marginConstant \cdot \epsilon^{\marginExponent}\Bigr\}.
\end{align*}
In particular, taking $\epsilon= \bigl\{\targetMarginalDistribution\bigl( \bigl\{ x \in \R^{\ambientDimension}: \classifier(x) \neq \bayesClassifier(x) \bigr\}\bigr)/(2\marginConstant)\bigr\}^{1/\marginExponent}$, we deduce that 
\begin{equation}
\label{Eq:muQbound}
\targetMarginalDistribution\left( \left\lbrace x \in \R^{\ambientDimension}: \classifier(x) \neq \bayesClassifier(x) \right\rbrace \right) \leq  (2 \marginConstant)^{\frac{1}{1+\marginExponent}}\cdot \excessRisk(\classifier)^{\frac{\marginExponent}{1+\marginExponent}}.   
\end{equation}
Now, for each $f \in \ermFiniteHypothesisSet$, let $Z_i^{f}:= \mathbbm{1}_{\{f(X_i)\neq Y_i\}} - \mathbbm{1}_{\{\bayesClassifier(X_i)\neq Y_i\}}$ for $i \in [n]$, noting that $\E(Z_i^{f})=\excessRisk(f)$ and $\E\bigl\{(Z_i^{f})^2\bigr\}\leq (2 \marginConstant)^{\frac{1}{1+\marginExponent}}\cdot \excessRisk(\classifier)^{\frac{\marginExponent}{1+\marginExponent}}$ by~\eqref{Eq:muQbound}.  Define the event 
\begin{align*}
E^{\delta}_3:=\bigcap_{f \in \ermFiniteHypothesisSet}\Biggl\{  \biggl|\frac{1}{\ermSampleSize}\sum_{i=1}^{\ermSampleSize}Z_i^f-\excessRisk(f)\biggr| \leq \sqrt{\frac{2(2 \marginConstant)^{\frac{1}{1+\marginExponent}}\excessRisk(\classifier)^{\frac{\marginExponent}{1+\marginExponent}}\log(2\left|\ermFiniteHypothesisSet\right|/\delta)}{\ermSampleSize}}+\frac{2\log(2\left|\ermFiniteHypothesisSet\right|/\delta)}{3\ermSampleSize} \Biggr\}.
\end{align*}
Note that $|Z_i^{f}|\leq 1$, so by Bernstein's inequality \citep{bernstein1924modification} combined with a union bound, we have $\Prob\bigl\{\left(E^{\delta}_3\right)^c\bigr\} \leq \delta$. Hence, on the event $E^{\delta}_3$, we have
\begin{align*}
\excessRisk&(\ermHypothesis)-\excessRisk( \optimalHypothesis) \leq \frac{1}{\ermSampleSize}\sum_{i=1}^{\ermSampleSize}Z_i^{\ermHypothesis} - \frac{1}{\ermSampleSize}\sum_{i=1}^{\ermSampleSize}Z_i^{\optimalHypothesis} +\biggl|\frac{1}{\ermSampleSize}\sum_{i=1}^{\ermSampleSize}Z_i^{\ermHypothesis}-\excessRisk(\ermHypothesis)\biggr|+ \biggl|\frac{1}{\ermSampleSize}\sum_{i=1}^{\ermSampleSize}Z_i^{\optimalHypothesis}-\excessRisk(\optimalHypothesis)\biggr| \\
& \leq 2\sqrt{\frac{\marginConstant^{\frac{1}{1+\marginExponent}} \excessRisk(\ermHypothesis)^{\frac{\marginExponent}{1+\marginExponent}} \log(2\left|\ermFiniteHypothesisSet\right|/\delta)}{\ermSampleSize}}+2\sqrt{\frac{\marginConstant^{\frac{1}{1+\marginExponent}} \excessRisk(\optimalHypothesis)^{\frac{\marginExponent}{1+\marginExponent}} \log(2\left|\ermFiniteHypothesisSet\right|/\delta)}{\ermSampleSize}}+\frac{4\log(2\left|\ermFiniteHypothesisSet\right|/\delta)}{3\ermSampleSize}\\
& \leq 4\sqrt{\frac{\marginConstant^{\frac{1}{1+\marginExponent}} \excessRisk(\ermHypothesis)^{\frac{\marginExponent}{1+\marginExponent}} \log(2\left|\ermFiniteHypothesisSet\right|/\delta)}{\ermSampleSize}}+\frac{4\log(2\left|\ermFiniteHypothesisSet\right|/\delta)}{3\ermSampleSize}.
\end{align*}
Thus, by considering separately the cases $\excessRisk(\ermHypothesis)\leq 2\excessRisk( \optimalHypothesis)$ and $\excessRisk(\ermHypothesis) > 2\excessRisk( \optimalHypothesis)$, we see that on the event $E^{\delta}_3$,
\begin{align*}
\excessRisk(\ermHypothesis)\leq 2\excessRisk( \optimalHypothesis)+ 64\marginConstant^{\frac{1}{2+\marginExponent}} \left(\frac{\log(2\left|\ermFiniteHypothesisSet\right|/\delta)}{\ermSampleSize} \right)^{\frac{1+\marginExponent}{2+\marginExponent}},
\end{align*}
as required.
\end{proof}
In order to apply Proposition~\ref{ermWithMarginCondition}, we will first derive a bound on the number of possible decision tree functions over $\sourceSample$.   Recall for $L \in \N$ that $\setOfDecisionTreeFunctionsL$ denotes the set of decision tree functions $\decisionTreeFunction:\R^{\ambientDimension} \rightarrow (0,1)$ to be those of the form $x \mapsto \tau_{\ell(x)}$ for some $\{\X_1,\ldots,\X_L\} \in \mathbb{T}_L$ with leaf function $\ell$, and some $(\tau_1,\ldots,\tau_L) \in \{0,1/\numSource,2/\numSource, \ldots, 1\}^L$. Given a set $\mathcal{S}\subseteq \R^{\ambientDimension}$, we let $h|_{\mathcal{S}}:\mathcal{S} \rightarrow (0,1)$ denote the restriction of $h$ to $\mathcal{S}$.

\begin{lemma}\label{countingDecisionTreesOnAFiniteSetLemma} Let $\mathcal{S} \subseteq \R^{\ambientDimension}$ be a set of cardinality at most $\numSource$, and let $L \in \N$.  Then the set $\{h|_{\mathcal{S}}: h \in \setOfDecisionTreeFunctionsL\}$ has cardinality at most $\bigl\{L\ambientDimension (\numSource+1)\bigr\}^{2L}$.
\end{lemma}
\begin{proof} For the proof, let $\mathbb{L}_L$ denote the set of leaf functions $\ell: \R^{\ambientDimension} \rightarrow \{1,\ldots,L\}$ corresponding to decision tree partitions $\{\X_1,\ldots,\X_L\} \in \mathbb{T}_L$. We begin by bounding the cardinality of the set of restricted leaf functions $\{\ell|_{\mathcal{S}}:\ell \in \mathbb{L}_L\}$. Observe that each restricted leaf function $\ell_{\mathcal{S}}$ may be constructed recursively by a sequence of $L-1$ splits. Each split point may be specified by choosing 
\begin{enumerate}
    \item[(a)] one of at most $L-1$ existing leaf nodes;
    \item[(b)] one of $\ambientDimension$ dimensions to split along;
    \item[(c)] one of at most $\numSource+1$ possible split points.
\end{enumerate}
Hence, $|\{\ell|_{\mathcal{S}}:\ell \in \mathbb{L}_L\}| \leq \bigl\{(L-1)\ambientDimension (\numSource+1)\bigr\}^{L-1}$. Moreover, there are at most $(\numSource+1)^L$ possible choices for  $(\tau_1,\ldots,\tau_L) \in \{0,1/\numSource,2/\numSource, \ldots, 1\}^L$. Since each $h|_{\mathcal{S}}$ is of the form $x\mapsto \tau_{\ell|_{\mathcal{S}}(x)}$ for some $\ell|_{\mathcal{S}}$ with $\ell \in \mathbb{L}_L$ and $(\tau_1,\ldots,\tau_L) \in \{0,1/\numSource,2/\numSource, \ldots, 1\}^L$, the result follows.
\end{proof}

\begin{corollary}\label{countingPossibleKNNsCalibratedByADecisionTree} Fix $\mathcal{D}_\sourceDistribution = \bigl((X_1^{\sourceDistribution},Y_1^{\sourceDistribution}),\ldots,(X_{\numSource}^{\sourceDistribution},Y_{\numSource}^{\sourceDistribution})\bigr) \in (\mathbb{R}^d \times \{0,1\})^{n_P}$.  For every $L \in \N$ and $\sigma>0$, we have $|\{\classifierEst_{\robustness,\decisionTreeFunction}^{\sourceDistribution}:\decisionTreeFunction \in \setOfDecisionTreeFunctionsL\}| \leq \bigl\{L\ambientDimension (\numSource+1)\bigr\}^{2L}$, where $\classifierEst_{\robustness,\decisionTreeFunction}^{\sourceDistribution}$ is defined in (\ref{defOfKnnClassifierForSingleDTreeAndRobustness}).
\end{corollary}
\begin{proof}
Let $\mathcal{S} = \{X_i^{\sourceDistribution}\}_{i=1}^{\numSource}$.  Observe from the definition~\eqref{Eq:marginFuncSource} that $\empiricalMargin_{k,h_0}^{\sourceDistribution} = \empiricalMargin_{k,h_1}^{\sourceDistribution}$ whenever $h_0|_\mathcal{S} = h_1|_\mathcal{S}$. Hence, by \eqref{Eq:lepskiKSource} and \eqref{defOfKnnClassifierForSingleDTreeAndRobustness}  the same is true of $\hat{k} \equiv \lepskiChoiceOfKSource(\cdot)$ and $\classifierEst_{\robustness,\decisionTreeFunction}^{\sourceDistribution}$. Thus, by Lemma \ref{countingDecisionTreesOnAFiniteSetLemma}, we have
\begin{align*}
|\{\classifierEst_{\robustness,\decisionTreeFunction}^{\sourceDistribution}:\decisionTreeFunction \in \setOfDecisionTreeFunctionsL\}| = \bigl|\{\classifierEst^{\sourceDistribution}_{\robustness,\decisionTreeFunction|_{\mathcal{S}}}:\decisionTreeFunction \in \setOfDecisionTreeFunctionsL\}\bigr| \leq \bigl|\{h|_{\mathcal{S}}: h \in \setOfDecisionTreeFunctionsL\}\bigr| \leq \bigl\{L\ambientDimension (\numSource+1)\bigr\}^{2L},
\end{align*}
as required.
\end{proof}
Recall from~\eqref{defOfKnnClassifierForSingleDTreeAndRobustness} that $\classifierEst^{\sourceDistribution}_{\robustness,L}(\cdot) =\mathbbm{1}_{\{\empiricalMargin^{\sourceDistribution}_{\hat{k},\hat{\decisionTreeFunction}}(\cdot) \geq 0\}}$, where $\hat{k} \equiv \hat{k}_{\robustness,\hat{\decisionTreeFunction}}^{\sourceDistribution}(\cdot)$ is defined in~\eqref{Eq:lepskiKSource}, and where $\hat{\decisionTreeFunction} \in \mathcal{H}_L$ is selected by empirical risk minimisation over $\targetSample^0$ as in~\eqref{ermOverDecisionTreesChoice}.  We are now in position to apply Proposition~\ref{ermWithMarginCondition} to obtain the main conclusion of this subsection.
\begin{proposition}\label{lemma:lowRegretWithTheRightNumLeavesAndRobustness} 
Fix $\theta^\sharp = (\Delta,\expansivityConstant,L^*,\theta) \in \Theta^\sharp$, where $\theta = (\targetDimension,\targetTailExponent,\sourceDimension,  \sourceTailExponent,\TailConstant,\marginExponent,\marginConstant,\holderExponent,\holderConstant)$, with $\holderExponent/(2\holderExponent+\sourceDimension)<\sourceTailExponent$, and $(P,Q) \in \mathcal{P}_{\theta^\sharp}$.  There exists ${C}_{\theta}' > 0$, depending only on $\theta$, such that for every $\delta \in (0,1)$, if we set $\robustness^* =  \min\bigl\{ \lceil 3 \log_+^{1/2}(\numSource/\delta)\rceil, \numSource\bigr\}$, then with probability at least $1 - 2\delta$, we have 
\begin{align*}
\excessRisk\bigl(\classifierEst^{\sourceDistribution}_{\robustness^*,L^*}\bigr) \leq {C}_{\theta}'\biggl\{\biggl(\frac{\log_+ (\numSource/\delta)}{\expansivityConstant^2 \cdot \numSource}\biggr)^{\frac{\holderExponent \sourceTailExponent (1+\marginExponent) }{\marginExponent \holderExponent + \sourceTailExponent(2\holderExponent+\sourceDimension)}} \! \! + \biggl(\frac{\Delta}{\expansivityConstant}\biggr)^{1+\marginExponent} \! \! \! +\biggl(\frac{L^*\log_+ (L^* \ambientDimension \numSource/\delta)}{ \numTarget}\biggr)^{\frac{1+\marginExponent}{2+\marginExponent}}\biggr\}.
\end{align*}
\end{proposition}

\begin{proof} Recalling $\tilde{C}_{\theta}$ from Proposition~\ref{prop:lowRegretWithTheRightDecisionTree}, by Proposition~\ref{ermWithMarginCondition} combined with Corollary~\ref{countingPossibleKNNsCalibratedByADecisionTree}, we can find ${C}_{\theta}' \geq 2\tilde{C}_{\theta}$, depending only on $\theta$, such that 
\begin{align*}
\Prob\biggl\{ \excessRisk\bigl(\classifierEst^{\sourceDistribution}_{\robustness^*,L^*}\bigr) > 2 \excessRisk\bigl(\classifierEst^{\sourceDistribution}_{\robustness^*,h^*}\bigr) +{C}_{\theta}' \biggl(\frac{L^*\log_+ (L^* \ambientDimension \numSource/\delta)}{ \numTarget}\biggr)^{\frac{1+\marginExponent}{2+\marginExponent}} \biggm| \sourceSample\biggr\} \leq \delta.
\end{align*}
Moreover, by Proposition~\ref{prop:lowRegretWithTheRightDecisionTree},
\begin{align*}
\Prob\biggl[ \excessRisk\bigl(\classifierEst^{\sourceDistribution}_{\robustness^*,h^*}\bigr)  >\tilde{C}_{\theta}\biggl\{\biggl(\frac{\log_+ (\numSource/\delta)}{\expansivityConstant^2 \cdot \numSource}\biggr)^{\frac{\holderExponent \sourceTailExponent (1+\marginExponent) }{\marginExponent \holderExponent + \sourceTailExponent(2\holderExponent+\sourceDimension)}}+ \biggl(\frac{\Delta}{\expansivityConstant}\biggr)^{1+\marginExponent}\biggr\}\biggr] \leq \delta.
\end{align*}
The result follows.
\end{proof}

\subsection{Completion of the proofs of Theorem \ref{Thm:UpperBound} and upper bound in Theorem~\ref{Thm:Main}}

\begin{proof}[Proof of Theorem \ref{Thm:UpperBound}]  
Since $\hat{\mathcal{F}}^{\sourceDistribution}$ and $\hat{\mathcal{F}}^{\targetDistribution}$ were constructed using only $\sourceSample \cup \targetSample^0$ (and not $\targetSample^1$), we may apply Proposition \ref{ermWithMarginCondition} conditionally on $\sourceSample \cup \targetSample^0$ and take expectations to obtain that with probability at least $1 - \delta/4$, we have
\[
\excessRisk(\hat{f}_{\mathrm{ATL}}) \leq 2 \min\bigl\{\excessRisk(\hat{f}_{\robustness^*,0}^{\sourceDistribution}),\excessRisk(\hat{f}_{\robustness^*,L^* \wedge \numTarget}^{\sourceDistribution}),\excessRisk(\hat{f}_{\tilde{\robustness}}^{\targetDistribution})\bigr\}+ 64\marginConstant^{\frac{1}{2+\marginExponent}} \biggl(\frac{\log(8|\hat{\mathcal{F}}^{\sourceDistribution} \cup \hat{\mathcal{F}}^{\targetDistribution}|/\delta)}{\lceil\numTarget/2\rceil} \biggr)^{\frac{1+\marginExponent}{2+\marginExponent}},
\]
where $\robustness^* = \min\bigl\{ \lceil 3 \log_+^{1/2}(\numSource/\delta)\rceil, \numSource\bigr\}$ and $\tilde{\robustness} = \min\bigl\{ \lceil 3 \log_+^{1/2}(\numTarget/\delta)\rceil, \numTarget\bigr\}$.  Now $|\hat{\mathcal{F}}^{\sourceDistribution} \cup \hat{\mathcal{F}}^{\targetDistribution}| \leq \numSource^2(\numTarget + 1) + \numTarget^2$, so the result follows from Corollary~\ref{justUseHalfCor}, Proposition~\ref{lemma:lowRegretWithTheRightNumLeavesAndRobustness} and Corollary~\ref{knnOnFirstHalfTargetDataCor}, which give the required high-probability bounds for $\excessRisk(\hat{f}_{\robustness^*,0}^{\sourceDistribution})$, $\excessRisk(\hat{f}_{\robustness^*,L^* \wedge \numTarget}^{\sourceDistribution})$ and $\excessRisk(\hat{f}_{\tilde{\robustness}}^{\targetDistribution})$ respectively.
\end{proof}
\begin{proof}[Proof of upper bound in Theorem~\ref{Thm:Main}]
We consider four cases.  First, if $\firstTerm_{\numSource,\numTarget}^{\mathrm{U}} = \min\bigl\{\firstTerm_{\numSource,\numTarget}^{\mathrm{U}},\secondTerm_{\numTarget}^{\mathrm{U}},1\bigr\}$ and $\bigl({L^*} a_1^{\mathrm{U}}/\numTarget\bigr)^{\frac{1+\marginExponent}{2+\marginExponent}} \leq (1-\expansivityConstant)^{1+\marginExponent}$, then the result follows by taking $\delta = \numTarget^{-\frac{1+\marginExponent}{2+\marginExponent}}$ in  Proposition~\ref{lemma:lowRegretWithTheRightNumLeavesAndRobustness}.  Second, if $\firstTerm_{\numSource,\numTarget}^{\mathrm{U}} = \min\bigl\{\firstTerm_{\numSource,\numTarget}^{\mathrm{U}},\secondTerm_{\numTarget}^{\mathrm{U}},1\bigr\}$ and $\bigl({L^*} a_1^{\mathrm{U}}/\numTarget\bigr)^{\frac{1+\marginExponent}{2+\marginExponent}} > (1-\expansivityConstant)^{1+\marginExponent}$, then the result follows by taking $\delta = \numSource^{-\frac{1+\marginExponent}{2+\marginExponent}}$ in Corollary~\ref{justUseHalfCor}.  Third, if $\secondTerm_{\numTarget}^{\mathrm{U}} = \min\bigl\{\firstTerm_{\numSource,\numTarget}^{\mathrm{U}},\secondTerm_{\numTarget}^{\mathrm{U}},1\bigr\}$, then the result follows by taking $\delta = \numTarget^{-\frac{1+\marginExponent}{2+\marginExponent}}$ in Corollary~\ref{knnOnFirstHalfTargetDataCor}.  Finally, $\min\bigl\{\firstTerm_{\numSource,\numTarget}^{\mathrm{U}},\secondTerm_{\numTarget}^{\mathrm{U}}\bigr\} > 1$, then the result follows from the fact that the excess risk of any data-dependent classifier is at most 1.
\end{proof}

\section{Proof of the lower bound in Theorem~\ref{Thm:Main}}
\label{Sec:LowerBound}

The proof of the lower bound in Theorem~\ref{Thm:Main} begins with a version of Assouad's lemma for transfer learning (Section~\ref{Sec:Assouad}) that translates the problem into one of constructing an appropriate family of distributions indexed by a hypercube.  To apply this lemma, we first construct the respective marginal distributions (Section~\ref{Sec:Marginal}) and then the corresponding regression functions (Section~\ref{Sec:RegressionFunctions}).  The lower bound is finally obtained via two applications of these results, reflecting the different challenges of estimating the decision tree function (Section~\ref{Sec:Thresholds}) and the source regression function (Section~\ref{Sec:SRF}).

\subsection{Assouad's lemma for transfer learning}
\label{Sec:Assouad}

\newcommand{\epsilonSource}{\epsilon_{\sourceDistribution}}
\newcommand{\epsilonTarget}{\epsilon_{\targetDistribution}}
\newcommand{\wSource}{w_{\sourceDistribution}}
\newcommand{\wTarget}{w_{\targetDistribution}}
\newcommand{\uSource}{u_{\sourceDistribution}}
\newcommand{\uTarget}{u_{\targetDistribution}}
\newcommand{\vSource}{v_{\sourceDistribution}}
\newcommand{\vTarget}{v_{\targetDistribution}}

The following result is a variant of Assouad's lemma (e.g.~\citet[Lemma 2]{yu1997assouad}, \citet[Section 3.12]{kimobtaining}), adapted to our setting.
\begin{lemma}\label{lemma:assouadTypeLemmaForTransferClassification} Let $\mathcal{P}$ be a set of pairs of distributions $(\sourceDistribution,\targetDistribution)$, each on $\R^d \times \{0,1\}$. Let $\numSource$, $\numTarget \in \N_0$, $m \in \N$, $\Sigma=\{-1,1\}^m$, $(x_t)_{t \in [m]} \in (\R^{\ambientDimension})^m$, $\epsilonSource$, $\epsilonTarget$ $ \in [0,1/4]$, $\uSource$, $\uTarget \in [0,1/m]$, $\vSource$, $\vTarget \in [0,1]$ and $\bigl\{(\sourceDistribution^{\sigma},\targetDistribution^{\sigma}):\sigma \in \Sigma\bigr\} \subseteq \mathcal{P}$ with respective regression functions $\sourceRegressionFunction^{\sigma}:\R^d \rightarrow [0,1]$, $\targetRegressionFunction^{\sigma}:\R^d \rightarrow [0,1]$ and marginals $\sourceMarginalDistribution$, $\targetMarginalDistribution$ on $\R^d$ satisfy:
\begin{enumerate}[(i)]
\item  $2^5(\numSource \uSource \epsilonSource^2+\numTarget \uTarget \epsilonTarget^2) \leq 1$;
\item  $\epsilonSource(2\vSource-1)=\epsilonTarget(2\vTarget-1)=0$;
\item for $t \in [m]$, we have $\sourceMarginalDistribution(\{x_t\}) = \uSource$ and $\targetMarginalDistribution(\{x_t\})=\uTarget$; 
    \item for $\sigma = (\sigma_1,\ldots,\sigma_m) \in \Sigma$ and $t \in [m]$, we have  $\sourceRegressionFunction^{\sigma}(x_t)=\vSource+\sigma_t \cdot \epsilonSource$ and $\targetRegressionFunction^{\sigma}(x_t)=\vTarget+\sigma_t \cdot \epsilonTarget$;
    \item for $\sigma$, $\sigma' \in \Sigma$, $x \in \supp(\sourceMarginalDistribution) \setminus \{x_t\}_{t \in [m]}$, we have $\sourceRegressionFunction^{\sigma}(x)= \sourceRegressionFunction^{\sigma'}(x)$; moreover, for $x \in \supp(\targetMarginalDistribution) \setminus \{x_t\}_{t \in [m]}$, we have $\targetRegressionFunction^{\sigma}(x)= \targetRegressionFunction^{\sigma'}(x)$.
\end{enumerate}
Then 
\[
\inf_{\classifierEst \in \setofDDClassifiers}\sup_{(P,Q) \in \mathcal{P}} \mathbb{E}\bigl\{\excessRisk(\classifierEst)\bigr\} \geq  \frac{m\uTarget \epsilonTarget}{2}.
\]
\end{lemma}

To prove Lemma~\ref{lemma:assouadTypeLemmaForTransferClassification}, we introduce some additional notation and provide a preliminary lemma.  For $\sigma \in \Sigma$, let $\nu^{\sigma}$ denote the product measure $(\sourceDistribution^{\sigma})^{\numSource}\times (\targetDistribution^{\sigma})^{\numTarget}$. In addition, given $\sigma=(\sigma_1,\ldots,\sigma_m) \in \Sigma$ and $t \in [m]$, we define ${\sigma^{t}}=(\sigma^{t}_1,\ldots,\sigma^t_m) \in \Sigma$ by $\sigma_t^{t}:=-\sigma_t$ and $\sigma^{t}_{t'}:=\sigma_{t'}$ for $t' \in [m]\setminus \{t\}$.

\begin{lemma}\label{lemma:boundTV} In the setting of Lemma~\ref{lemma:assouadTypeLemmaForTransferClassification}, we have $\totalVariation\bigl(\nu^{\sigma},\nu^{{\sigma^t}}\bigr) \leq 1/2$ for every $\sigma \in \Sigma $ and $t \in [m]$.
\end{lemma}
\begin{proof} We first show that $\kullbackLeiblerDivergence(\sourceDistribution^{\sigma},\sourceDistribution^{{\sigma^t}}) \leq 16\uSource\epsilonSource^2$. Without loss of generality, we assume that $\epsilonSource > 0$, since otherwise $\sourceDistribution^{{\sigma^t}} = \sourceDistribution^{\sigma}$. Thus, $\sourceRegressionFunction^{{\sigma^t}}(x_t)=1- \sourceRegressionFunction^{\sigma}(x_t)=1/2-\epsilonSource \cdot \sigma_t$ and $\sourceRegressionFunction^{{\sigma^t}}(x)= \sourceRegressionFunction^{\sigma}(x)$ for all $x \in \supp(\sourceMarginalDistribution)\setminus \{x_t\}$. Hence,
\begin{align*}
\kullbackLeiblerDivergence\bigl(\sourceDistribution^{\sigma},\sourceDistribution^{{\sigma^t}}\bigr)&=\int_{\R^{\ambientDimension} \times \{0,1\}}\log\Bigl(\frac{d\sourceDistribution^{\sigma}}{d \sourceDistribution^{{\sigma^t}}}\Bigr) \, d\sourceDistribution^{\sigma}\\
&=\int_{\R^{\ambientDimension}} \biggl\{ \sourceRegressionFunction^{\sigma}(x)\log\biggl(\frac{\sourceRegressionFunction^{\sigma}(x)}{\sourceRegressionFunction^{{\sigma^t}}(x)}\biggr)+\bigl(1-\sourceRegressionFunction^{\sigma}(x)\bigr)\log\biggl(\frac{1-\sourceRegressionFunction^{\sigma}(x)}{1-\sourceRegressionFunction^{{\sigma^t}}(x)}\biggr)\biggr\} \, d\sourceMarginalDistribution(x)\\
&= \bigl(2\sourceRegressionFunction^{\sigma}(x_t)-1\bigr)\log\biggl(\frac{\sourceRegressionFunction^{\sigma}(x_t)}{1-\sourceRegressionFunction^{\sigma}(x_t)}\biggr) \cdot \sourceMarginalDistribution(\{x_t\})\\
&= 2\epsilonSource \sigma_t\cdot\log\biggl(\frac{1+2\epsilonSource \sigma_t}{1-2\epsilonSource \sigma_t}\biggr) \cdot \uSource = 2\uSource\epsilonSource\cdot\log\biggl(\frac{1+2\epsilonSource}{1-2\epsilonSource}\biggr)\\&\leq \frac{8\uSource\epsilonSource^2}{1-2\epsilonSource}\leq 16\uSource\epsilonSource^2,
\end{align*}
where the penultimate inequality uses the inequality $\log a \leq a-1$ for $a\geq 1$. By the same argument, we also have $\kullbackLeiblerDivergence(\targetDistribution^{\sigma},\targetDistribution^{{\sigma^t}}) \leq 16\uTarget\epsilonTarget^2$.  By the additive property of Kullback--Leibler divergence for product measures, we conclude that
\[
\kullbackLeiblerDivergence\bigl(\nu^{\sigma},\nu^{{\sigma^t}}\bigr) = \numSource \kullbackLeiblerDivergence\bigl(\sourceDistribution^{\sigma},\sourceDistribution^{{\sigma^t}}\bigr) + \numTarget \kullbackLeiblerDivergence\bigl(\targetDistribution^{\sigma},\targetDistribution^{{\sigma^t}}\bigr) \leq  16(\numSource\uSource\epsilonSource^2 +\numTarget \uTarget \epsilonTarget^2) \leq \frac{1}{2}.
\]
Thus, by Pinsker's inequality \citep[e.g.,][Lemma 2.5]{tsybakov2009introduction},
\[
\totalVariation\bigl(\nu^{\sigma},\nu^{\sigma^t}\bigr)  \leq \sqrt{\kullbackLeiblerDivergence\bigl(\nu^{\sigma},\nu^{\sigma^t}\bigr)/2} \leq 1/2,
\]
as required.
\end{proof}

We now return to the proof of Lemma \ref{lemma:assouadTypeLemmaForTransferClassification}.
\begin{proof}[Proof of Lemma \ref{lemma:assouadTypeLemmaForTransferClassification}] 
Without loss of generality, we assume that $\epsilonTarget >0$, so $\vTarget=1/2$. Fix $\classifierEst \in \setofDDClassifiers$. Given $z\in \mathcal{Z} :=  \left(\R^{\ambientDimension} \times \{0,1\}\right)^{\numSource} \times \left(\R^{\ambientDimension} \times \{0,1\}\right)^{\numTarget}$, let $\classifierEst_z:\R^{\ambientDimension}\rightarrow \{0,1\}$ denote the  mapping obtained by taking $z$ as the first argument in $\classifierEst$. Then
\begin{align*}
\sup_{(P,Q) \in \mathcal{P}} \mathbb{E}&\bigl\{\excessRisk(\classifierEst)\bigr\} \geq \max_{\sigma \in \Sigma} \int_{\mathcal{Z}} \int_{\{ x \in \R^{\ambientDimension}:\classifierEst_z(x) \neq f^*_{\targetDistribution^{\sigma}}(x)\}}\bigl|2\targetRegressionFunction^{\sigma}(x)-1\bigr| \, d\targetMarginalDistribution(x) \, d\nu^{\sigma}(z)\\
&\geq \frac{1}{2^m}\sum_{\sigma \in \Sigma} \int_{\mathcal{Z}} \sum_{t=1}^m\bigl|2\targetRegressionFunction^{\sigma}(x_t)-1\bigr| \cdot \mathbbm{1}_{\{\classifierEst_z(x_t) \neq f^*_{\targetDistribution^{\sigma}}(x_t)\}} \, \targetMarginalDistribution(\{x_t\}) \, d\nu^{\sigma}(z)\\
&= \frac{\uTarget \epsilonTarget}{2^{m-1}}\sum_{t=1}^m \sum_{\sigma \in \Sigma} \nu^{\sigma}\bigl(\bigl\{ \classifierEst_z(x_t) \neq f^*_{\targetDistribution^{\sigma}}(x_t)\bigr\} \bigr)\\
&= \frac{\uTarget \epsilonTarget}{2^{m}}\sum_{t=1}^m \sum_{\sigma \in \Sigma} \Bigl\{ \nu^{\sigma}\bigl(\bigl\{ \classifierEst_z(x_t) \neq f^*_{\targetDistribution^{\sigma}}(x_t)\bigr\} \bigr)+\nu^{\sigma^{t}}\bigl(\bigl\{ \classifierEst_z(x_t) \neq f^*_{\targetDistribution^{\sigma^{t}}}(x_t)\bigr\} \bigr)\Bigr\} \\
&\geq \frac{\uTarget \epsilonTarget}{2^{m}}\sum_{t=1}^m\sum_{\sigma \in \Sigma} \bigl\{ 1-\totalVariation( \nu^{\sigma},\nu^{\sigma^{t}})\bigr\} \geq \frac{m\uTarget \epsilonTarget}{2},
\end{align*}
where the penultimate inequality uses the fact that $f^*_{\targetDistribution^{{\sigma}^{t}}}(x_t)=1-f^*_{\targetDistribution^{\sigma}}(x_t)$ and the final inequality follows from Lemma~\ref{lemma:boundTV}.
\end{proof}

\subsection{Marginal construction}
\label{Sec:Marginal}

\newcommand{\contractionKappaTarget}{\kappa_{\targetDistribution}}
\newcommand{\contractionKappaSource}{\kappa_{\sourceDistribution}}

The marginal distributions $\sourceMarginalDistribution$ and $\targetMarginalDistribution$ in our lower bound construction will not vary with the vertices of our hypercube $\Sigma$ in Lemma~\ref{lemma:assouadTypeLemmaForTransferClassification}.  An interesting consequence of this fact is that our lower bound in Theorem~\ref{Thm:Main} will continue to hold, even if these marginal distributions were known (or equivalently, if we were also provided with an infinite sample of unlabelled training data from either distribution).  These measures will consist of a mixture of a discrete uniform distribution on a lattice of points in the non-negative orthant in $\R^d$ and a component consisting of a uniform distribution on a $d_Q$-dimensional hyper-rectangle in the opposite orthant, as illustrated in Figure~\ref{Fig:LowerBoundPlot}.  Moreover, the lattice component of the support of $\mu_Q$ will be a $d_Q$-dimensional slice within the $d_P$-dimensional lattice component of the support of $\mu_P$.  The structure of these marginals is designed to put as much probability mass as possible on the lattice points, as these will be the points that are difficult to classify, and will maximise the lower bound in Lemma~\ref{lemma:assouadTypeLemmaForTransferClassification}.  On the other hand, Condition~(i) of Lemma~\ref{lemma:assouadTypeLemmaForTransferClassification} constrains us to have a sufficiently large lattice that no individual point provides too much information to the learner.  The component supported on the hyper-rectangle is used to ensure that the margin condition (Assumption~\ref{marginAssumption}) is satisfied. 
\begin{figure}[htbp!]
  \begin{center}
    \includegraphics[width=0.6\textwidth]{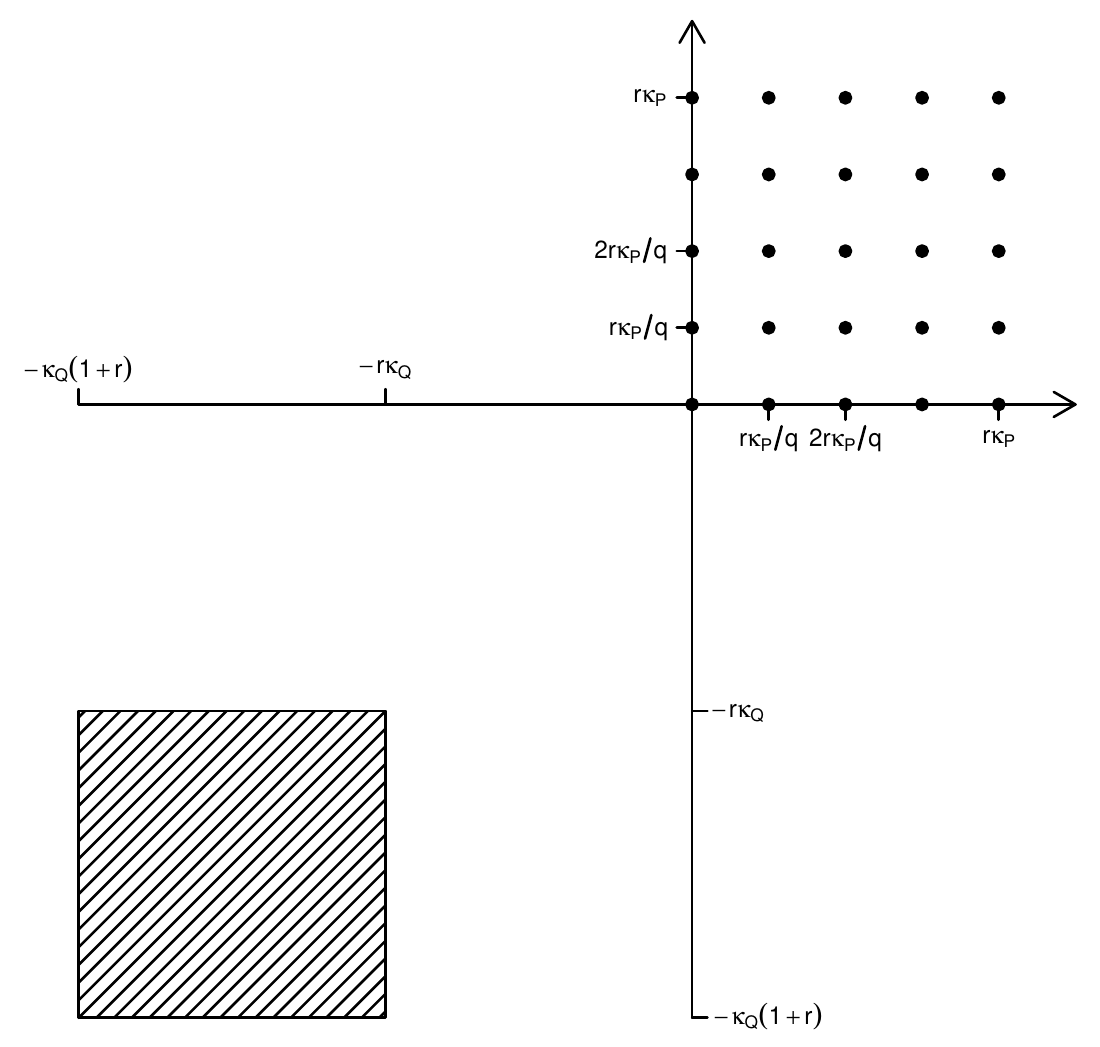}
  \end{center}
\caption{\label{Fig:LowerBoundPlot}Illustration of the support of the measure $\mu_{q,r,w,d_0}$ in~\eqref{Eq:GeneralMarginal}.}
\end{figure}

To describe the construction more formally, define $\contractionKappaSource:=1/(2\sourceDimension^{1/2})$ and $\contractionKappaTarget:=1/(2\targetDimension^{1/2})$.  For $q \in \mathbb{N}$ and $d_0 \in \{\sourceDimension,\targetDimension\}$, let $\widetilde{\mathcal{T}}_{q,d_0} := \{0,1, \ldots, q-1\}^{d_0}\times \{0\}^{d-d_0} \subseteq \R^{\ambientDimension}$.  Now let $({{\tilde{x}^{q}}_t})_{t=1}^{q^{\targetDimension}}$ be an enumeration of the set $\widetilde{\mathcal{T}}_{q,\targetDimension}$ and let $({{\tilde{x}^{q}}_t})_{t=q^{\targetDimension}+1}^{q^{\sourceDimension}}$ be an enumeration of $\widetilde{\mathcal{T}}_{q,\sourceDimension} \setminus \widetilde{\mathcal{T}}_{q,\targetDimension}$.  For each $q \in \N$, $r >0$, $t \in [q^{\sourceDimension}]$ and $d_0 \in \{\sourceDimension,\targetDimension\}$, we let $x^{q,r}_t:=(r/q) \cdot \contractionKappaSource \cdot \tilde{x}^{q}_t$ and $\mathcal{T}_{q,r,d_0} := \bigl\{x^{q,r}_t: t \in [q^{d_0}]\bigr\}$.  For a Borel subset $A$ of $\mathbb{R}^d$, let $A_{\targetDimension} := \bigl\{(x_1,\ldots,x_{\targetDimension}):(x_1,\ldots,x_{\targetDimension},0,\ldots,0) \in A\bigr\}$ and let $A_{\targetDimension,\sourceDimension} := A \cap \bigl(\R^{\targetDimension} \times [0,1)^{\sourceDimension - \targetDimension} \times \{0\}^{\ambientDimension - \sourceDimension}\bigr)$.  Given $q \in \N$,  $w\in [0,1/2]$, $r>0$, $d_0 \in \{\sourceDimension,\targetDimension\}$, we define a probability measure $\mu_{q,r,w,d_0}$ on $\R^{\ambientDimension}$ by
\begin{align}
\label{Eq:GeneralMarginal}
\mu_{q,r,w,d_0}(A) &:= \frac{(1-w)}{\contractionKappaTarget^{\targetDimension}} \mathcal{L}_{\targetDimension}\Bigl(A_{\targetDimension} \cap \bigl[-\contractionKappaTarget(1+r),-r\contractionKappaTarget\bigr]^{\targetDimension}\Bigr) +\frac{w}{N_{q,r,d_0}}\sum_{t=1}^{q^{d_0}}\mathbbm{1}_{\{x^{q,r}_t \in A_{\targetDimension,\sourceDimension}\}},
\end{align}
for Borel subsets $A$ of $\R^{\ambientDimension}$, where $N_{q,r,d_0} := q^{\targetDimension}\min\bigl\{\lceil q/(r \contractionKappaSource)\rceil,q\bigr\}^{d_0 - \targetDimension} = \bigl|(\mathcal{T}_{q,r,d_0})_{\targetDimension,\sourceDimension}\bigr|$.  Our marginal measures $\mu_P$ and $\mu_Q$ will be chosen as instances of $\mu_{q,r,w,d_0}$ for particular choices of $q$, $r$, $w$ and $d_0$; see Corollary~\ref{lemma:constructedMuMeasureSatisfiesTailAssumption}.

We begin with a couple of preliminary lemmas, before presenting the main properties of these marginal distributions in Lemma~\ref{lemma:constructedMuMeasureOmega} and Corollary~\ref{lemma:constructedMuMeasureSatisfiesTailAssumption} below.  This latter result provides  sufficient conditions for the marginals to satisfy Assumption~\ref{sourceDensityAssumption}.

\begin{lemma}\label{lemma:cubeBallIntersectLemma} For $\kappa \in (0,\targetDimension^{-1/2})$ and $x \in [0,\kappa]^{\targetDimension}$, we have $ \mathcal{L}_{\targetDimension}\bigl(\{y \in \R^{\targetDimension}:\|y-x\| < s\} \cap [0,\kappa]^{\targetDimension}\bigr) \geq  (\kappa s)^{\targetDimension}$ for all $s \in [0,1]$, and $ \mathcal{L}_{\targetDimension}\bigl(\{y \in \R^{\targetDimension}:\|y-x\| < s\}\cap [0,\kappa]^{\targetDimension}\bigr) \leq V_{\targetDimension}\cdot s^{\targetDimension}$ for all $s>0$.
\end{lemma}

\begin{proof} To prove the lower bound, we take $x \in [0,\kappa]^{\targetDimension}$, $s \in [0,1]$, and consider the map $\psi_{x,s}:z\mapsto s\cdot (z-x)+x$ on $\R^{\targetDimension}$. Observe that  $\psi_{x,s}([0,\kappa]^{\targetDimension})\subseteq  [0,\kappa]^{\targetDimension}$. On the other hand, since $x \in \psi_{x,s}([0,\kappa]^{\targetDimension})$ and $\diam(\psi_{x,s}([0,\kappa]^{\targetDimension})) \leq s\kappa \targetDimension^{1/2} <s$, we also have $\psi_{x,s}([0,\kappa]^{\targetDimension})\subseteq \{y \in \R^{\targetDimension}:\|y-x\| < s\}$. Hence,
\begin{align*}
\mathcal{L}_{\targetDimension}\bigl(\{y \in \R^{\targetDimension}:\|y-x\| < s\}\cap [0,\kappa]^{\targetDimension}\bigr) \geq \mathcal{L}_{\targetDimension}\left(\psi_{x,s}([0,\kappa]^{\targetDimension})\right) \geq (\kappa s)^{\targetDimension}.
\end{align*}
The upper bound follows from the fact that $\{y \in \R^{\targetDimension}:\|y-x\| < s\} \cap [0,1]^{\targetDimension} \subseteq \{y \in \R^{\targetDimension}:\|y-x\| < s\}$.
\end{proof}

\begin{lemma}\label{lemma:countPointsInBallAndGrid} For $q \in\N$, $d_0 \in \{\sourceDimension,\targetDimension\}$, $x \in \widetilde{\mathcal{T}}_{q,d_0}$ and $s \leq 4q\sourceDimension^{1/2}$, we have $\bigl|\widetilde{\mathcal{T}}_{q,d_0} \cap B_s(x)\bigr|\geq \{s/(2^4\sourceDimension^{1/2})\}^{d_0}$.
\end{lemma}
\begin{proof} First observe that if $q=1$, then $\bigl|\widetilde{\mathcal{T}}_{q,d_0} \cap B_s(x)\bigr| = |\{0\}| =1\geq \{s/(2^4 \sourceDimension^{1/2})\}^{d_0}$.  For $q \geq 2$, we have $s/(2^4\sourceDimension^{1/2}) \leq (q-1)/2$. Hence, for each $x \in \widetilde{\mathcal{T}}_{q,d_0}$, we can find a $\sourceDimension$-dimensional, axis-aligned cube $A$ with vertex $x$ and side length $s/(2^4\sourceDimension^{1/2})$ containing at least $\lceil s/(2^4\sourceDimension^{1/2}) \rceil^{d_0}$ elements of  $\widetilde{\mathcal{T}}_{q,d_0}$. Thus, $\bigl|\widetilde{\mathcal{T}}_{q,d_0}\cap B_s(x)\bigr|\geq \bigl|\widetilde{\mathcal{T}}_{q,d_0}\cap A\bigr|\geq \{s/(2^4\sourceDimension^{1/2})\}^{d_0}$.
\end{proof}

\begin{lemma}\label{lemma:constructedMuMeasureOmega} Let $q \in \N$, $r >0$, $w \in [0,1/2]$ and $d_0 \in \{\sourceDimension,\targetDimension\}$. We have
\begin{enumerate}[(i)]
    \item $\omega_{\mu_{q,r,w,d_0},d_0}(x) \geq 1-w$ for all $x \in \left[-\contractionKappaTarget(1+r),-r\contractionKappaTarget\right]^{\targetDimension}\times\{0\}^{\ambientDimension-\targetDimension}$;
    \item $\omega_{\mu_{q,r,w,d_0},d_0}(x) \geq 2^{-3d_0} \cdot \min\bigl\{1, w \cdot q^{d_0} \cdot N_{q,r,d_0}^{-1} \cdot r^{-d_0}\bigr\}$ for all $x \in \mathcal{T}_{q,r,\targetDimension}$.
\end{enumerate}
\end{lemma}
\begin{proof} To prove (i), we take $x \in \left[-\contractionKappaTarget(1+r),-r\contractionKappaTarget\right]^{\targetDimension}\times\{0\}^{\ambientDimension-\targetDimension}$ and $s \in (0,1)$.  As shorthand, we write $B := B_s(x)$, so that $B_{\targetDimension} = \bigl\{(x_1,\ldots,x_{\targetDimension}):(x_1,\ldots,x_{\targetDimension},0,\ldots,0) \in B\bigr\}$.  By Lemma~\ref{lemma:cubeBallIntersectLemma} combined with the translation invariance of Lebesgue measure, we have
\begin{align*}
\mu_{q,r,w,d_0}(B) \geq (1-w)\cdot \contractionKappaTarget^{-\targetDimension} \cdot \mathcal{L}_{\targetDimension}\bigl(B_{\targetDimension} \cap \left[-\contractionKappaTarget(1+r),-r\contractionKappaTarget\right]^{\targetDimension}\bigr) \geq (1-w) \cdot s^{d_0}.
\end{align*}
The claim~(i) follows.  

To prove (ii), we take $x =x_{t}^{q,r} \in \mathcal{T}_{q,r,\targetDimension}$. If $s \in \bigl(0,(2r) \wedge 1\bigr]$, then $\tilde{s}:=\{q/(r\contractionKappaSource)\}\cdot s \leq \min\bigl\{4q\sourceDimension^{1/2},q/(r\contractionKappaSource)\bigr\}$, so by Lemma \ref{lemma:countPointsInBallAndGrid} we have
\begin{align*}
\mu_{q,r,w,d_0}\bigl(B_s(x)\bigr) &\geq \frac{w}{N_{q,r,d_0}} \cdot \bigl|\mathcal{T}_{q,r,d_0}\cap B_s(x^{q,r}_t)\bigr|= \frac{w}{N_{q,r,d_0}} \cdot \bigl|\widetilde{\mathcal{T}}_{q,d_0}\cap B_{\tilde{s}}(\tilde{x}^{q}_t)\bigr|\\
&\geq \frac{w}{N_{q,r,d_0}} \cdot \bigl\{\tilde{s}/(2^4\sourceDimension^{1/2})\bigr\}^{d_0} = \frac{w}{N_{q,r,d_0}} \cdot \bigl\{qs/(2^4r\contractionKappaSource \sourceDimension^{1/2})\bigr\}^{d_0} \\
&= 2^{-3d_0} \cdot w \cdot r^{-d_0} \cdot \frac{q^{d_0}}{N_{q,r,d_0}} \cdot s^{d_0}.
\end{align*}
On the other hand, if $s \in ( 2r,1]$ then with $z_r:=(\overbrace{-r\contractionKappaTarget,\ldots,-r\contractionKappaTarget}^{\targetDimension},\overbrace{0,\ldots,0}^{\ambientDimension-\targetDimension}) \in \R^d$, we have $\|x-z_r\| \leq \|x\| + \|z_r\| \leq r/2 + r/2 < s/2$.  Hence, by~(i), we have $\mu_{q,r,w,d_0}\bigl(B_s(x)\bigr)\geq \mu_{q,r,w,d_0}\bigl(B_{s/2}(z_r)\bigr)\geq (1-w)\cdot (s/2)^{d_0} \geq 2^{-(d_0+1)}s^{d_0} \geq 2^{-3d_0} s^{d_0}$, and the conclusion follows.
\end{proof}

\begin{corollary}\label{lemma:constructedMuMeasureSatisfiesTailAssumption}  Take $\TailConstant > 1$, $\targetDimension \in [1,\ambientDimension]$, $\sourceDimension \in [\targetDimension, \ambientDimension]$ and $ \sourceTailExponent$, $\targetTailExponent  >0$. Suppose that $q \in \N$, $r>0$, and $\wSource$, $\wTarget \in [0,2^{-3\sourceDimension(\sourceTailExponent \vee \targetTailExponent)}\wedge (1-\TailConstant^{-1/(\sourceTailExponent \wedge\targetTailExponent)})]$ satisfy $\wTarget (\wSource \cdot q^{\sourceDimension} \cdot N_{q,r,\sourceDimension}^{-1})^{-\sourceTailExponent}r^{\sourceDimension\sourceTailExponent} \leq 2^{-3\sourceDimension\sourceTailExponent}$ and $\wTarget^{1-\targetTailExponent} r^{\targetDimension \targetTailExponent} \leq 2^{-3\targetDimension\targetTailExponent}$. Then Assumption \ref{sourceDensityAssumption} is satisfied for $\sourceMarginalDistribution=\mu_{q,r,\wSource,\sourceDimension}$ and $\targetMarginalDistribution=\mu_{q,r,\wTarget,\targetDimension}$.  
\end{corollary}
\begin{proof}
For the first condition of Assumption~\ref{sourceDensityAssumption} consider initially $\xi \in \bigl(0,2^{-3\targetDimension} \cdot \min\{1, \wTarget \cdot r^{-\targetDimension}\}\bigr]$.  Then by Lemma~\ref{lemma:constructedMuMeasureOmega},
\[
\targetMarginalDistribution\bigl(\bigl\{ x \in \R^d: \omega_{\mu_\targetDistribution,d_\targetDistribution}(x)<\xi\bigr\} \bigr) = 0 \leq \TailConstant \cdot \xi^{\targetTailExponent}.
\]
If $\xi \in \bigl(2^{-3\targetDimension} \cdot \min\{1, \wTarget \cdot r^{-\targetDimension}\},1 - \wTarget\bigr]$, then by Lemma~\ref{lemma:constructedMuMeasureOmega} again,
\[
\targetMarginalDistribution\bigl(\bigl\{ x \in \R^d: \omega_{\mu_\targetDistribution,d_\targetDistribution}(x)<\xi\bigr\} \bigr) = \wTarget \leq \TailConstant \cdot 2^{-3\targetDimension \targetTailExponent}(1 \wedge \wTarget^{\targetTailExponent} r^{-\targetDimension \targetTailExponent}) \leq \TailConstant\cdot \xi^{\targetTailExponent}.
\]
Finally, if $\xi \in (1-\wTarget,\infty)$, then
\[
\targetMarginalDistribution\bigl(\bigl\{ x \in \R^d: \omega_{\mu_\targetDistribution,d_\targetDistribution}(x)<\xi\bigr\} \bigr) = 1 \leq \TailConstant (1 - \wTarget)^{\targetTailExponent} \leq \TailConstant\cdot \xi^{\targetTailExponent},
\]
as required.

For the second condition of Assumption~\ref{sourceDensityAssumption} let  $\xi \in \bigl(0,2^{-3\sourceDimension} \cdot \min\{1, \wSource \cdot q^{\sourceDimension} \cdot N_{q,r,\sourceDimension}^{-1} \cdot r^{-\sourceDimension}\}\bigr]$.  Then by Lemma~\ref{lemma:constructedMuMeasureOmega},
\[
\targetMarginalDistribution\bigl(\bigl\{ x \in \R^d: \omega_{\mu_\sourceDistribution,\sourceDimension}(x)<\xi\bigr\} \bigr) = 0 \leq \TailConstant \cdot \xi^{\sourceTailExponent}.
\]
If $\xi \in \bigl(2^{-3\sourceDimension} \cdot \min\bigl\{1, \wSource \cdot q^{\sourceDimension} \cdot N_{q,r,\sourceDimension}^{-1} \cdot r^{-\sourceDimension}\bigr\},1 - \wSource\bigr]$, then by Lemma~\ref{lemma:constructedMuMeasureOmega} again,
\begin{align*}
\targetMarginalDistribution\bigl(\bigl\{ x \in \R^d: \omega_{\mu_\sourceDistribution,\sourceDimension}(x)<\xi\bigr\} \bigr) = \wTarget &\leq \TailConstant \cdot 2^{-3\sourceDimension \sourceTailExponent}\bigl\{1 \wedge \bigl(\wSource \cdot q^{\sourceDimension} \cdot N_{q,r,\sourceDimension}^{-1} \cdot r^{-\sourceDimension}\bigr)^{\sourceTailExponent}\bigr\} \\ 
&\leq \TailConstant\cdot \xi^{\sourceTailExponent}.
\end{align*}
Finally, if $\xi \in (1-\wSource,\infty)$, then
\[
\targetMarginalDistribution\bigl(\bigl\{ x \in \R^d: \omega_{\mu_\sourceDistribution,\sourceDimension}(x)<\xi\bigr\} \bigr) = 1 \leq \TailConstant (1 - \wSource)^{\sourceTailExponent} \leq \TailConstant\cdot \xi^{\sourceTailExponent},
\]
as required.
\end{proof}

\subsection{Target regression function construction}
\label{Sec:RegressionFunctions}

\newcommand{\constructedRegressionFunction}{\eta_{\epsilon,q,r,\sigma}}

We now describe a construction of a family of target regression functions that are indexed by the vertices of a hypercube as in Lemma~\ref{lemma:assouadTypeLemmaForTransferClassification}.  We begin by defining the restrictions of the elements of this family to the support of $\mu_P$; on this set, these restrictions will be perturbations of the uninformative regression function that takes the constant value $1/2$.  The perturbations should be as large as possible, to maximise the quantity $\epsilon_Q$ in Lemma~\ref{lemma:assouadTypeLemmaForTransferClassification} and to ensure that the margin condition (Assumption~\ref{marginAssumption}) holds, but need to be small enough that the restrictions can be extended to functions on $\R^d$ that satisfy the H\"older continuity condition (Assumption~\ref{holderContinuityAssumption}).

Given $\epsilon \in (0,1/8]$, $q \in \N$, $r >0$, $\sigma =(\sigma_t)_{t=1}^{q^{\targetDimension}} \in \{-1,1\}^{q^{\targetDimension}}$, we first define $\constructedRegressionFunction^{\circ}: \left[-\contractionKappaTarget(1+r),-r\contractionKappaTarget\right]^{\targetDimension}\times\{0\}^{\ambientDimension-\targetDimension}\cup \mathcal{T}_{q,r,\sourceDimension} \rightarrow \R$ by
\begin{align*}
\constructedRegressionFunction^{\circ}(x):=\begin{cases}\frac{1}{2}-2\epsilon -\frac{1}{4}\|x-z_r\|^{\holderExponent}&\text{ if } x\in \left[-\contractionKappaTarget(1+r),-r\contractionKappaTarget\right]^{\targetDimension}\times\{0\}^{\ambientDimension-\targetDimension}\\
\frac{1}{2}+\sigma_t\cdot \epsilon &\text{ if }x=x^{q,r}_t\text{ with }t\leq q^{\targetDimension}\\
\frac{1}{2}-2\epsilon &\text{ if }x=x^{q,r}_t\text{ with } q^{\targetDimension}<t\leq q^{\sourceDimension},
\end{cases}
\end{align*}
where $z_r:=(\overbrace{-r\contractionKappaTarget,\ldots,-r\contractionKappaTarget}^{\targetDimension},\overbrace{0,\ldots,0}^{\ambientDimension-\targetDimension})$.  The main results of this subsection (Corollary~\ref{lemma:existenceConstructedRegressionFunctionIsHolder} and Lemma~\ref{checkingTheMarginConditionForConstructionLemma}) provide sufficient conditions for an extension $\constructedRegressionFunction$ of $\constructedRegressionFunction^{\circ}$ to the whole of~$\R^d$ to satisfy Assumptions~\ref{holderContinuityAssumption} and~\ref{marginAssumption} respectively.  Recalling that $\contractionKappaSource=1/(2\sourceDimension^{1/2})$ and $\contractionKappaTarget=1/(2\targetDimension^{1/2})$, we first present a basic property of $\constructedRegressionFunction^{\circ}$.  
\begin{lemma}\label{lemma:constructedRegressionFunctionIsHolder} Let $q \in \N$, $r>0$, $\holderExponent \in (0,1]$, $\sigma =(\sigma_t)_{t=1}^{q^{\targetDimension}} \in \{-1,1\}^{q^{\targetDimension}}$ and $\epsilon \in \bigl(0,1/8 \wedge  (1/6)\cdot(r\cdot \contractionKappaSource / q)^{\holderExponent}\bigr]$. Then $|\constructedRegressionFunction^{\circ}(x)-\constructedRegressionFunction^{\circ}(x')| \leq \|x-x'\|^{\holderExponent}$ for all $x$, $x' \in \left[-\contractionKappaTarget(1+r),-r\contractionKappaTarget\right]^{\targetDimension}\times\{0\}^{\ambientDimension-\targetDimension}\cup \mathcal{T}_{q,r,\sourceDimension}$. Moreover, $\constructedRegressionFunction^{\circ}(x) \in [0,1]$ for all $x \in \left[-\contractionKappaTarget(1+r),-r\contractionKappaTarget\right]^{\targetDimension}\times\{0\}^{\ambientDimension-\targetDimension}\cup \mathcal{T}_{q,r,\sourceDimension}$.
\end{lemma}
\begin{proof} To prove the first part of the lemma, we consider three cases. First, if $x$, $x' \in \left[-\contractionKappaTarget(1+r),-r\contractionKappaTarget\right]^{\targetDimension}\times\{0\}^{\ambientDimension-\targetDimension}$, then by Minkowski's inequality
\begin{align*}
|\constructedRegressionFunction^{\circ}(x)-\constructedRegressionFunction^{\circ}(x')| = \frac{1}{4}\left|\|x-z_r\|^{\holderExponent} -\|x'-z_r\|^{\holderExponent}\right| \leq \frac{1}{4}\|x-x'\|^{\holderExponent}.    
\end{align*}
Second, if $x \in \left[-\contractionKappaTarget(1+r),-r\contractionKappaTarget\right]^{\targetDimension}\times\{0\}^{\ambientDimension-\targetDimension}$ and $x' \in \mathcal{T}_{q,r,\sourceDimension}$, then 
\begin{align}
\label{Eq:SecondCase}
|\constructedRegressionFunction^{\circ}(x)-\constructedRegressionFunction^{\circ}(x')| &\leq \frac{1}{4}\|x-z_r\|^{\holderExponent} +3\epsilon \nonumber \\
&\leq \frac{1}{2} \bigl\{\|x-z_r\|^{\holderExponent}+(r\cdot \contractionKappaSource)^{\holderExponent} \bigr\} \leq \bigl\{\|x-z_r\|+(r\cdot \contractionKappaTarget) \bigr\}^{\holderExponent}.
\end{align}
Now let $x_r \in \R^\ambientDimension$ denote the point where the line segment joining $x$ and $0$ meets the boundary of the convex set $\mathcal{C}_r := [-r\contractionKappaTarget,\infty)^{\targetDimension} \times \{0\}^{\ambientDimension-\targetDimension} \subseteq \R^\ambientDimension$, and note that $\|x_r\| \geq r \cdot \contractionKappaTarget$.  Observe that $z_r$ is the Euclidean projection of $x$ onto $\mathcal{C}_r$.  Hence
\begin{equation}
\label{Eq:Projection}
\|x-z_r\|+ r\cdot \contractionKappaTarget \leq \|x-x_r\| + \|x_r\| = \|x\| \leq \|x-x'\|.
\end{equation}
The combination of~\eqref{Eq:SecondCase} and~\eqref{Eq:Projection} establishes the desired property in the second case. 

Finally, if $x$, $x' \in \mathcal{T}_{q,r,\sourceDimension}$ with $x \neq x'$, then 
\begin{align*}
|\constructedRegressionFunction^{\circ}(x)-\constructedRegressionFunction^{\circ}(x')| \leq 3\epsilon \leq \biggl( \frac{r \cdot \contractionKappaSource}{q}\biggr)^{\holderExponent}\leq \|x-x'\|^{\holderExponent}.
\end{align*}
To prove the second part of the lemma, suppose first that $x \in \left[-\contractionKappaTarget(1+r),-r\contractionKappaTarget\right]^{\targetDimension}\times\{0\}^{\ambientDimension-\targetDimension}$.  Then, since $\|x-z_r\|\leq 1$ and $\epsilon \in (0, 1/8]$, we must have $\constructedRegressionFunction^{\circ}(x) \in[ 0,1]$. On the other hand, if $x \in \mathcal{T}_{q,r,\sourceDimension}$, then $\constructedRegressionFunction^{\circ}(x) \in \{1/2-2\epsilon,1/2-\epsilon,1/2+\epsilon\} \subseteq [0,1]$.
\end{proof}

\begin{corollary}\label{lemma:existenceConstructedRegressionFunctionIsHolder} Let $q \in \N$, $r>0$, $\holderExponent \in (0,1]$, $\sigma =(\sigma_t)_{t=1}^{q^{\targetDimension}} \in \{-1,1\}^{q^{\targetDimension}}$ and $\epsilon \in \bigl(0,1/8 \wedge  (1/6)\cdot(r\cdot \contractionKappaSource / q)^{\holderExponent}\bigr]$. Then there exists a function $\constructedRegressionFunction:\R^{\ambientDimension} \rightarrow [0,1]$ such that 
\begin{align}\label{eq:defOfConstructedRegressionFunction}
\constructedRegressionFunction(x):=\begin{cases}\frac{1}{2}-2\epsilon -\frac{1}{4}\|x-z_r\|^{\holderExponent}&\text{ if } x\in \left[-\contractionKappaTarget(1+r),-r\contractionKappaTarget\right]^{\targetDimension}\times\{0\}^{\ambientDimension-\targetDimension}\\
\frac{1}{2}+\sigma_t\cdot \epsilon &\text{ if }x=x^{q,r}_t\text{ with }t\leq q^{\targetDimension}\\
\frac{1}{2}-2\epsilon &\text{ if }x=x^{q,r}_t\text{ with } q^{\targetDimension}<t\leq q^{\sourceDimension},
\end{cases}
\end{align}
and $|\constructedRegressionFunction(x)-\constructedRegressionFunction(x')| \leq \|x-x'\|^{\holderExponent}$ for all $x$, $x' \in \R^{\ambientDimension}$. In particular, Assumption \ref{holderContinuityAssumption} holds for the regression function $\targetRegressionFunction=\constructedRegressionFunction$ with $\holderConstant=1$. 
\end{corollary}
\begin{proof} By Lemma~\ref{lemma:constructedRegressionFunctionIsHolder}, the function $\constructedRegressionFunction^{\circ}: \left[-\contractionKappaTarget(1+r),-r\contractionKappaTarget\right]^{\targetDimension}\times\{0\}^{\ambientDimension-\targetDimension}\cup \mathcal{T}_{q,r,\sourceDimension} \rightarrow [0,1]$ is H\"{o}lder continuous with exponent $\holderExponent$ and constant $1$ on its domain. By McShane's extension theorem \cite[Corollary 1]{mcshane1934extension}, there exists an extension ${\constructedRegressionFunction'}: \R^{\ambientDimension} \rightarrow \R$ which is H\"{o}lder continuous with exponent $\holderExponent$ and constant $1$, and satisfies  $\constructedRegressionFunction'(x)= \constructedRegressionFunction^{\circ}$ for $x \in \left[-\contractionKappaTarget(1+r),-r\contractionKappaTarget\right]^{\targetDimension}\times\{0\}^{\ambientDimension-\targetDimension}\cup \mathcal{T}_{q,r,\sourceDimension}$.  The function $\constructedRegressionFunction:\R^{\ambientDimension} \rightarrow [0,1]$ given by $\constructedRegressionFunction(x):= \bigl\{\constructedRegressionFunction'(x) \vee 0\bigr\} \wedge 1$ has the desired properties.
\end{proof}

\begin{lemma}\label{checkingTheMarginConditionForConstructionLemma} Let $q \in \N$, $r>0$, $\holderExponent \in (0,1]$, $\sigma =(\sigma_t)_{t=1}^{q^{\targetDimension}} \in \{-1,1\}^{q^{\targetDimension}}$, $\epsilon \in \bigl(0,1/8 \wedge  (1/6)\cdot(r\cdot \contractionKappaSource / q)^{\holderExponent}\bigr]$, $\marginConstant\geq 1 + 2^{2\targetDimension/\holderExponent}\targetDimension^{\targetDimension/2}V_{\targetDimension}$, $\marginExponent \in [0,\targetDimension/\holderExponent]$ and $\wTarget \in \bigl[0,(1/2) \wedge\epsilon^{\marginExponent}\bigr]$. Then Assumption~\ref{marginAssumption} holds whenever $\targetDistribution$ has marginal  $\targetMarginalDistribution=\mu_{q,r,\wTarget,\targetDimension}$ and regression function $\targetRegressionFunction=\constructedRegressionFunction$.
\end{lemma}
\begin{proof} Without loss of generality, take $\zeta < 1$.  First suppose $\zeta \geq \epsilon$.  By~\eqref{eq:defOfConstructedRegressionFunction}, if $x\in \supp(\targetMarginalDistribution)\setminus \mathcal{T}_{q,r,\targetDimension}$ and $|\targetRegressionFunction(x) - 1/2| < \zeta$, then $\|x-z_r\|\leq (4\zeta)^{1/\holderExponent}$. As shorthand, we write $B := B_{(4\zeta)^{1/\holderExponent}}(z_r)$, so that $B_{\targetDimension} = \bigl\{(x_1,\ldots,x_{\targetDimension}):(x_1,\ldots,x_{\targetDimension},0,\ldots,0) \in B\bigr\}$. Hence,
\begin{align*}
\targetMarginalDistribution\bigl(\bigl\{x \in \R^d:\left|\targetRegressionFunction(x)-1/2\right|<\zeta\bigr\}\bigr) &\leq \targetMarginalDistribution\bigl(\mathcal{T}_{q,r,\targetDimension}\bigr)+ \targetMarginalDistribution(B)\\
&\leq \wTarget+\contractionKappaTarget^{-\targetDimension} \cdot \mathcal{L}_{\targetDimension}\bigl(B_{\targetDimension} \cap \left[-\contractionKappaTarget(1+r),-r\contractionKappaTarget\right]^{\targetDimension}\bigr)\\
& \leq \epsilon^{\marginExponent}+(2\contractionKappaTarget)^{-\targetDimension} V_{\targetDimension} (4\zeta)^{\targetDimension/\holderExponent} \leq \marginConstant \cdot \zeta^{\marginExponent}.
\end{align*}
On the other hand, if $\zeta<\epsilon$, then $\targetMarginalDistribution\bigl(\bigl\{x \in \R^d:\left|\targetRegressionFunction(x)-1/2\right|<\zeta\bigr\}\bigr)=0 \leq \marginConstant \cdot \zeta^{\marginExponent}$, as required.
\end{proof}

\subsection{Difficulty of estimating the decision tree function}
\label{Sec:Thresholds}

Lemma~\ref{lemma:thresholdEstimationLowerBound} below provides an initial minimax lower bound that arises from the difficulty of estimating the decision tree function.  The proof will involve the marginal distributions $\mu_P$ and $\mu_Q$ constructed in Section~\ref{Sec:Marginal}, the family of target regression functions constructed in Section~\ref{Sec:RegressionFunctions}, and will contain a description of the construction of the corresponding family of source regression functions that is appropriate for this lower bound.  Recall the definition of $B_{\numTarget}^{\mathrm{L}}$ from Theorem~\ref{Thm:Main}.
\begin{lemma}
\label{lemma:thresholdEstimationLowerBound}
Fix $\theta^\sharp = (\Delta,\expansivityConstant,L^*,\theta) \in \Theta^\sharp$ with $\marginExponent\holderExponent \leq \targetDimension$, $\sourceTailExponent(1 - \targetTailExponent) \leq \targetTailExponent$ and $\marginConstant\geq 1 + 2^{2\targetDimension/\holderExponent}\targetDimension^{\targetDimension/2}V_{\targetDimension}$.  Then there exists $c_{\theta,0} > 0$, depending only on $\theta$, such that 
\begin{align}
\label{Eq:ThresholdEstLB}
 \inf_{\classifierEst \in \setofDDClassifiers}\sup_{(P,Q) \in \mathcal{P}_{\theta^\sharp}} \mathbb{E}\bigl\{\excessRisk(\classifierEst)\bigr\} \geq  c_{\theta,0}\biggl\{ \biggl(  \frac{{L^*}}{ \numTarget}\biggr)^{\frac{1+\marginExponent}{2+\marginExponent}} \wedge  B_{\numTarget}^{\mathrm{L}} \wedge (1-\expansivityConstant)^{1+\marginExponent}\biggr\}.
\end{align}
\end{lemma}
\begin{proof}
Our goal is to define a particular instantiation of the construction in Lemma~\ref{lemma:assouadTypeLemmaForTransferClassification}, which requires us to specify $m \in \N$, $(x_t)_{t \in [m]} \in (\R^{\ambientDimension})^m$, $\epsilonSource$, $\epsilonTarget$ $ \in [0,1/4]$, $\uSource$, $\uTarget$ $\in [0,1/m]$, $\vSource$, $\vTarget \in [0,1]$, regression functions $\sourceRegressionFunction^{\sigma}:\R^d \rightarrow [0,1]$, $\targetRegressionFunction^{\sigma}:\R^d \rightarrow [0,1]$ for $\sigma \in \Sigma = \{-1,1\}^m$, and marginals $\sourceMarginalDistribution$, $\targetMarginalDistribution$ on $\R^d$.

To this end, we first define some intermediate quantities that depend only on $\theta$.   Let 
\begin{align*}
\rho \equiv \rho_\theta &:= \frac{\targetTailExponent(\targetDimension-\marginExponent \holderExponent) + \marginExponent \holderExponent}{\targetTailExponent(2\holderExponent + \targetDimension) + \marginExponent \holderExponent}; \quad a_1 \equiv a_{1,\theta} := 2^{-3\sourceDimension(\sourceTailExponent \vee \targetTailExponent)}\wedge (1-\TailConstant^{-1/(\sourceTailExponent \wedge\targetTailExponent)}); \nonumber \\
\rho_1 \equiv \rho_{1,\theta} &:= \frac{\targetDimension}{\holderExponent(2+\marginExponent)} + \frac{\marginExponent}{\sourceTailExponent(2+\marginExponent)} + 1; \quad b_1 \equiv b_{1,\theta} := \frac{2^{5\rho_1}\contractionKappaSource^{\sourceDimension}}{8^{\sourceDimension}\cdot 6^{\targetDimension/\holderExponent} \cdot 2^{5 + \sourceDimension - \targetDimension}}; \\ 
\lambda \equiv \lambda_\theta &:= \frac{\marginExponent + 2\targetTailExponent + \targetDimension \targetTailExponent/\holderExponent}{2+\marginExponent}; \quad a_2 \equiv a_{2,\theta} := 2^{5(\lambda-\targetTailExponent)} \cdot 2^{-3\targetDimension\targetTailExponent} \cdot \contractionKappaSource^{\targetDimension\targetTailExponent} \cdot 6^{-\targetDimension\targetTailExponent/\holderExponent}.
\end{align*}
Now let $a \equiv a_{\theta} := \min\bigl\{(a_1b_1)^{1/\rho_1}  , 2^5 a_1^{(2+\marginExponent)/\marginExponent},a_2^{1/\lambda},2^{-(1+3\alpha)},2^{4 - 2/\alpha}\bigr\} > 0$.
This allows us to define 
\[
q =\big\lfloor \min\bigl(a\numTarget^{\rho},L^*\bigr)^{1/\targetDimension}\big\rfloor.
\]
Observe that $q \geq 1$ whenever $\numTarget \geq a^{-1/\rho}$, and we will therefore first prove the desired lower bound in this case.  Now let $m = q^{\targetDimension}$, let $\epsilon_{\sourceDistribution} = 0$, let 
\begin{equation}
\label{Eq:EpsilonQ}
\epsilon \equiv \epsilon_{\targetDistribution} = \min\biggl\{\Bigl(\frac{m}{2^5\numTarget}\Bigr)^{1/(2+\marginExponent)},\frac{1-\expansivityConstant}{4}\biggr\},
\end{equation}
let $\wTarget = \epsilon^{\marginExponent}$, let $\uTarget = \wTarget/m$, let $r = (6\epsilon)^{1/\holderExponent} q/\contractionKappaSource$, let $\wSource = (8r)^{\sourceDimension}N_{q,r,\sourceDimension}q^{-\sourceDimension}\wTarget^{1/\sourceTailExponent}$ and let $\uSource = \wSource/N_{q,r,\sourceDimension}$.   Set  $x_t = x_t^{q,r}$ for $t \in [m]$, where $x_t^{q,r}$ is defined at the beginning of Section~\ref{Sec:Marginal}.  Further, let $\vSource = 1/2 + \epsilon$ and $\vTarget = 1/2$.  Recalling~\eqref{Eq:GeneralMarginal}, we will take the marginal distributions to be $\sourceMarginalDistribution=\mu_{q,r,\wSource,\sourceDimension}$ and $\targetMarginalDistribution=\mu_{q,r,\wTarget,\targetDimension}$, noting that by our choice of the first three terms in the minimum defining $a$, the conditions of Corollary~\ref{lemma:constructedMuMeasureSatisfiesTailAssumption} hold (this uses the hypothesis that $\sourceTailExponent(1 - \targetTailExponent) \leq \targetTailExponent$), and this corollary then tells us that Assumption~\ref{sourceDensityAssumption} is satisfied.  For $\sigma \in \Sigma$, let $\targetRegressionFunction^\sigma = \regressionFunction_{\epsilon,q,r,\sigma}$ as defined in Corollary~\ref{lemma:existenceConstructedRegressionFunctionIsHolder}, noting that the fourth term in the minimum defining $a$ ensures that the conditions of this corollary hold, and therefore that each $\targetRegressionFunction^\sigma$ satisfies Assumption~\ref{holderContinuityAssumption}.  Moreover, the final term in the minimum defining~$a$, together with the hypotheses of the current lemma, guarantees that the conditions of Lemma~\ref{checkingTheMarginConditionForConstructionLemma} hold, so the distribution $\targetDistribution^\sigma$ on $\R^\ambientDimension \times \{0,1\}$ with marginal distribution $\targetMarginalDistribution$ and regression function $\targetRegressionFunction^\sigma$ satisfies Assumption~\ref{marginAssumption}.

It remains to define $\sourceRegressionFunction^\sigma$ for $\sigma \in \Sigma$, and to do this, we first define a decision tree partition and a family of transfer functions.  Recalling the definition of $\mathcal{T}_{q,r,\targetDimension}$ from the beginning of Section~\ref{Sec:Marginal}, let $\{\X_1^*,\ldots,\X_{L^*}^*\} \in \thresholdPartitions_{L^*}$ be such that $\X_{\ell}^* \cap \mathcal{T}_{q,r,\targetDimension} = \{x_{\ell}^{q,r}\}$ for each $\ell \in [m] \subseteq [L^*]$ (the fact that $m \leq L^*$ follows from our definition of $q$).  Define $h:[0,1]\rightarrow [0,1]$ by
\begin{align*}
h(z) :=\begin{cases}z&\text{ if }z \in \left[0,1/2-2\epsilon\right]\\3z+4\epsilon-1
&\text{ if }z \in \left[1/2-2\epsilon,1/2-\epsilon\right]\\\frac{(1-2\epsilon)z+4\epsilon}{1+2\epsilon}
&\text{ if }z \in \left[1/2-\epsilon,1\right].
\end{cases}
\end{align*}
Observe that 
\begin{equation}
\label{Eq:hlowerbound}
\frac{h(z) - 1/2}{z - 1/2} \geq \frac{1-2\epsilon}{1+2\epsilon} \geq 1 - 4\epsilon \geq \phi
\end{equation}
for $z \in [0,1] \setminus {1/2}$, where the final bound follows from the second term in the minimum defining~$\epsilon$.  For $\sigma =(\sigma_1,\ldots,\sigma_m) \in \Sigma$ and $\ell \in [m]$, define $g_{\ell}^\sigma: [0,1] \rightarrow [0,1]$ by
\[
g_\ell^\sigma(z) := \left\{ \begin{array}{ll} z & \quad \mbox{if $\sigma_\ell=1$} \\
h(z) & \quad \mbox{if $\sigma_\ell=-1$,} \end{array}  \right.
\]
and for $\ell \in \{m+1,\ldots,L^*\}$, let $g_\ell^\sigma(z) := z$.  We can now set $\sourceRegressionFunction^\sigma = g_\ell^\sigma \circ \targetRegressionFunction^\sigma$ on $\X_\ell^*$, and note that by~\eqref{Eq:hlowerbound}, Assumption~\ref{transferAssumption} holds for each transfer function $g_\ell^\sigma$.  

We are now in a position to verify that our constructed marginals and family of source and target regression functions satisfy the conditions of Lemma~\ref{lemma:assouadTypeLemmaForTransferClassification} with $\mathcal{P} = \mathcal{P}_{\theta^\sharp}$.  Condition~(i) holds because $\epsilon_P = 0$ and $2^5\numTarget \uTarget \epsilonTarget^2 = 2^5\numTarget \epsilon_Q^{2+\alpha} \leq 1$ by definition of $\epsilon_Q$ in~\eqref{Eq:EpsilonQ}.  The verification of Condition~(ii) again uses the fact that $\epsilon_P = 0$, and also that $v_Q = 1/2$.  Condition~(iii) follows immediately by definition of $\mu_P$, $\mu_Q$, $x_t$, $u_P$ and $u_Q$.  The second part of Condition~(iv) holds by definition of $\eta_Q^\sigma$, together with the definitions of $\eta_{\epsilon,q,r,\sigma}$ in~\eqref{lemma:existenceConstructedRegressionFunctionIsHolder}, $v_Q$ and $\epsilon_Q$.  The first part of this condition uses this second part, together with the facts that $v_P = 1/2 + \epsilon$ and $h(1/2-\epsilon)=1/2+\epsilon$.  Finally, Condition~(v) holds because the restriction of $\eta_{\epsilon,q,r,\sigma}$ in~\eqref{eq:defOfConstructedRegressionFunction} to $[-\kappa_Q(1+r),-r\kappa_Q]^{d_Q} \times \{0\}^{d-d_Q}$ does not depend on $\sigma$, and because~$g_\ell^\sigma$ is the identity function for $\ell \in \{m+1,\ldots,L^*\}$.

Writing $c_{\theta,0}' := a^{\frac{1+\marginExponent}{2+\marginExponent}}/2^{(6 + \targetDimension)(1+\marginExponent)}$, we conclude from Lemma~\ref{lemma:assouadTypeLemmaForTransferClassification} that 
\begin{equation}
\label{Eq:nlarge2}
\inf_{\classifierEst \in \setofDDClassifiers}\sup_{(P,Q) \in \mathcal{P}} \mathbb{E}\bigl\{\excessRisk(\classifierEst)\bigr\} \geq  \frac{m\uTarget \epsilonTarget}{2} \geq c_{\theta,0}' \biggl\{ \biggl(  \frac{{L^*}}{ \numTarget}\biggr)^{\frac{1+\marginExponent}{2+\marginExponent}} \wedge  B_{\numTarget}^{\mathrm{L}} \wedge (1-\expansivityConstant)^{1+\marginExponent}\biggr\}
\end{equation}
for $\numTarget \geq a^{-1/\rho}$.  But the left-hand side of~\eqref{Eq:nlarge2} is decreasing in $\numTarget$, so the full result holds on setting $c_{\theta,0} := c_{\theta,0}' \cdot 2^{-(1+\marginExponent)}a^{\frac{1+\marginExponent}{\rho(2+\marginExponent)}}$.
\end{proof}

\subsection{Difficulty of estimating the source regression function and completion of the proof of the lower bound in Theorem~\ref{Thm:Main}}
\label{Sec:SRF}

\begin{lemma}
\label{Lemma:EstimateEtaP}
Fix $\theta^\sharp = (\Delta,\expansivityConstant,L^*,\theta) \in \Theta^\sharp$ with $\marginExponent\holderExponent \leq \targetDimension$, $\sourceTailExponent(1 - \targetTailExponent) \leq \targetTailExponent$ and $\marginConstant\geq 1 + 2^{2\targetDimension/\holderExponent}\targetDimension^{\targetDimension/2}V_{\targetDimension}$.  Then there exists $c_{\theta,1} > 0$, depending only on $\theta$, such that 
\[
 \inf_{\classifierEst \in \setofDDClassifiers}\sup_{(P,Q) \in \mathcal{P}_{\theta^\sharp}} \mathbb{E}\bigl\{\excessRisk(\classifierEst)\bigr\} \geq  c_{\theta,1}\min\biggl\{\biggl(\frac{1}{\expansivityConstant^2 \cdot \numSource}\biggr)^{\frac{\holderExponent \sourceTailExponent (1+\marginExponent) }{\sourceTailExponent(2\holderExponent+\sourceDimension)+\marginExponent \holderExponent }} + \biggl(\frac{\Delta}{\expansivityConstant}\biggr)^{1+\marginExponent},B_{\numTarget}^{\mathrm{L}},1\biggr\}.
\]
\end{lemma}
\begin{proof} Recalling the definition of $ a_1 = 2^{-3\sourceDimension(\sourceTailExponent \vee \targetTailExponent)}\wedge (1-\TailConstant^{-1/(\sourceTailExponent \wedge\targetTailExponent)})$ from the proof of Lemma~\ref{lemma:thresholdEstimationLowerBound}, we define $a_3:=a_1^{1/\targetDimension} \cdot 6^{-1/\holderExponent} \cdot (\contractionKappaSource/16)^{\sourceDimension/\targetDimension}$ and let
\begin{align*}
{a}_4&:=\min\biggl\{ a_3^{\frac{\holderExponent\targetDimension \targetTailExponent}{\targetTailExponent(\targetDimension-\marginExponent \holderExponent)+\marginExponent\holderExponent}},\frac{1}{8},a_1^{1/\marginExponent},2^{-1/\marginExponent},\left(\contractionKappaSource^{\sourceDimension} \cdot 2^{-(6+3\sourceDimension)}\cdot 6^{-\sourceDimension/\holderExponent} \right)^{\frac{\holderExponent \sourceTailExponent }{\sourceTailExponent(2\holderExponent+\sourceDimension)+\marginExponent \holderExponent }}, \nonumber \\
&\hspace{8cm} (a_3^{\targetDimension}\cdot 2^{-(6+\targetDimension)})^{\frac{\holderExponent \targetTailExponent}{\targetTailExponent(2\holderExponent+\targetDimension)+\marginExponent \holderExponent}} \biggr\}, \nonumber \\
\epsilon\equiv \epsilonTarget &:= {a}_4\cdot \min\biggl(\max\biggl\{\biggl(\frac{1}{\expansivityConstant^2 \cdot \numSource}\biggr)^{\frac{\holderExponent \sourceTailExponent }{\sourceTailExponent(2\holderExponent+\sourceDimension)+\marginExponent \holderExponent }} , \frac{\Delta}{\expansivityConstant}\biggr\},\biggl(\frac{1}{\numTarget}\biggr)^{\frac{\holderExponent \targetTailExponent}{\targetTailExponent(2\holderExponent+\targetDimension)+\marginExponent \holderExponent}}\biggr).
\end{align*}
Take $q:=\Big\lfloor a_3 \cdot \epsilon^{-\frac{\targetTailExponent(\targetDimension-\marginExponent \holderExponent) + \marginExponent \holderExponent}{\holderExponent \targetDimension \targetTailExponent}}\Bigr\rfloor$.  We will initially assume that $\numTarget \geq 1$, which means that $\epsilon \leq a_3^{\frac{\holderExponent\targetDimension \targetTailExponent}{\targetTailExponent(\targetDimension-\marginExponent \holderExponent)+\marginExponent\holderExponent}}$ so $q \geq 1$. Further define $m = q^{\targetDimension}$, let $\epsilonSource=(\expansivityConstant \cdot \epsilon-\Delta)\vee 0$, let $\wTarget = \epsilon^{\marginExponent}$, let $\uTarget = \wTarget/m$, let $r = (6\epsilon)^{1/\holderExponent} q/\contractionKappaSource$, let $\wSource = (8r)^{\sourceDimension}N_{q,r,\sourceDimension}q^{-\sourceDimension}\wTarget^{1/\sourceTailExponent}$, let $\uSource = \wSource/N_{q,r,\sourceDimension}$ and let $\vSource =\vTarget = 1/2$.   Set  $x_t = x_t^{q,r}$ for $t \in [m]$, where $x_t^{q,r}$ is defined at the beginning of Section~\ref{Sec:Marginal}.  We will take the marginal distributions to be $\sourceMarginalDistribution=\mu_{q,r,\wSource,\sourceDimension}$ and $\targetMarginalDistribution=\mu_{q,r,\wTarget,\targetDimension}$, which, as in the proof of Lemma~\ref{lemma:thresholdEstimationLowerBound}, satisfy the conditions of Corollary~\ref{lemma:constructedMuMeasureSatisfiesTailAssumption}, and hence Assumption~\ref{sourceDensityAssumption}.  For $\sigma \in \Sigma = \{-1,1\}^m$, let $\targetRegressionFunction^\sigma = \regressionFunction_{\epsilon,q,r,\sigma}$.  Then, as in the proof of Lemma~\ref{lemma:thresholdEstimationLowerBound}, the conditions of Corollary~\ref{lemma:existenceConstructedRegressionFunctionIsHolder} and Lemma~\ref{checkingTheMarginConditionForConstructionLemma} hold, so Assumptions~\ref{holderContinuityAssumption} and~\ref{marginAssumption} are also satisfied.  For $\delta \in (0,\expansivityConstant/2]$, define $h_{\expansivityConstant,\delta}:[0,1]\rightarrow [0,1]$ by
\begin{align*}
h_{\expansivityConstant,\delta}(z):=\begin{cases} \expansivityConstant\cdot (z-1/2)+1/2+\delta&\text{ if } z \in \bigl[0,1/2-\delta/\expansivityConstant\bigr]\\ 1/2 & \text{ if } z \in \bigl[1/2- \delta/\expansivityConstant,1/2+ \delta/\expansivityConstant\bigr] \\
\expansivityConstant\cdot (z-1/2)+ 1/2 - \delta&\text{ if } z \in \bigl[1/2+ \delta/\expansivityConstant,1\bigr],
\end{cases}
\end{align*}
and for $\delta > \expansivityConstant/2$, let $h_{\expansivityConstant,\delta}(\cdot) := 1/2$.  For $\sigma \in \Sigma$, we take $\sourceRegressionFunction^{\sigma} =h_{\expansivityConstant,\Delta} \circ \targetRegressionFunction^{\sigma}$, and $g_{\ell}:=h_{\expansivityConstant,0}$ for $\ell \in [L^*]$.  Note that these definitions ensure that each $g_\ell$ satisfies~\eqref{Eq:gell}, and $\bigl|\sourceRegressionFunction^{\sigma}(x) - g_\ell\bigl(\targetRegressionFunction^\sigma(x)\bigr)\bigr| \leq \|h_{\expansivityConstant,\Delta} - h_{\expansivityConstant,0}\|_{\infty} \leq \Delta$ for $x \in \X_{\ell}$, so Assumption~\ref{transferAssumption} holds.  

Finally, similar (but slightly simpler) arguments to those used in the proof of Lemma~\ref{lemma:thresholdEstimationLowerBound} verify that the assumptions of Lemma~\ref{lemma:assouadTypeLemmaForTransferClassification} hold with $\mathcal{P} = \mathcal{P}_{\theta^\sharp}$, so writing $c_{\theta,1} := a_4^{1+\marginExponent}/4$,
we conclude from Lemma~\ref{lemma:assouadTypeLemmaForTransferClassification} that 
\begin{align}
\label{Eq:nlarge}
\inf_{\classifierEst \in \setofDDClassifiers}\sup_{(P,Q) \in \mathcal{P}} \mathbb{E}\bigl\{\excessRisk(\classifierEst)\bigr\} &\geq  \frac{m\uTarget \epsilonTarget}{2} \nonumber \\
&\geq c_{\theta,1} \min\biggl\{\biggl(\frac{1}{\expansivityConstant^2 \cdot \numSource}\biggr)^{\frac{\holderExponent \sourceTailExponent (1+\marginExponent) }{\sourceTailExponent(2\holderExponent+\sourceDimension)+\marginExponent \holderExponent }} + \biggl(\frac{\Delta}{\expansivityConstant}\biggr)^{1+\marginExponent},B_{\numTarget}^{\mathrm{L}}\biggr\}
\end{align}
whenever $\numTarget \geq 1$.  But the left-hand side of~\eqref{Eq:nlarge} is decreasing in $\numTarget$, so the full result follows.
\end{proof}

\begin{proof}[Proof of the lower bound in Theorem~\ref{Thm:Main}]
This follows immediately from Lemmas~\ref{lemma:thresholdEstimationLowerBound} and~\ref{Lemma:EstimateEtaP}.
\end{proof}

\appendix

\section{Auxiliary results}

\subsection{Understanding the first part of Assumption~\ref{sourceDensityAssumption}}

The goal of this subsection is to provide several results to understand the first part of the tail assumption (Assumption~\ref{sourceDensityAssumption}).  We begin with settings where the super-level sets of our marginal distribution are regular \citep{audibert2007fast}, and then consider cases where this marginal satisfies a strong minimal mass assumption \citep{gadat2016classification}; in both cases we are able to relate the lower density $\omega_{\mu,d}$ of~\eqref{Eq:omega} to properties of the density of $\mu$.   Other settings for which we are able to verify the first part of Assumption~\ref{sourceDensityAssumption} include distributions satisfying a moment condition on a separable metric space and (mixtures of) log-concave distributions.  

\begin{definition}[Regular sets]
  \label{Def:Regular} Given $c_0, r_0 > 0$, a set $S\subseteq \R^d$ is said to be $(c_0,r_0)$-regular if $\Lebesgue\bigl(S \cap B_r(x)\bigr) \geq c_0 \cdot   \Lebesgue\bigl(B_r(x)\bigr)$ for all $r \in (0,r_0]$ and $x \in S$.
\end{definition}

\begin{example} Suppose that $S\subseteq \R^d$ is a convex set with $x_0 \in S$ such that $B_{\lambda_0}(x_0)\subseteq S \subseteq B_{\lambda_1}(x_0)$ for some $0<\lambda_0 \leq \lambda_1$. Then $S$ is $(c_0,r_0)$-regular with $c_0=\bigl(\lambda_0/(\lambda_0+\lambda_1)\bigr)^d$ and $r_0=\lambda_0$. Indeed, if we take $x\in S$ then by convexity we have $B_{\rho \cdot \lambda_0}(\rho\cdot x_0+(1-\rho)\cdot x)\subseteq S$ for any $\rho \in [0,1]$. Given $r \in (0,r_0]$, we may take $\rho_{x,r}:=r/(\lambda_0+\|x-x_0\|) \in (0, 1]$, so that $B_{\rho_{x,r} \cdot \lambda_0}\bigl(\rho_{x,r}\cdot x_0+(1-\rho_{x,r})\cdot x\bigr)\subseteq B_r(x)$. Hence,
\begin{align*}
\Lebesgue\bigl(S \cap B_r(x)\bigr) &\geq \Lebesgue\bigl(B_{\rho_{x,r} \cdot \lambda_0}(\rho_{x,r}\cdot x_0+(1-\rho_{x,r})\cdot x)\bigr) \\
&\geq  \left(\frac{\rho_{x,r} \cdot \lambda_0}{r}\right)^d  \cdot  \Lebesgue\bigl(B_r(x)\bigr) \geq  c_0 \cdot \Lebesgue\bigl(B_r(x)\bigr),
\end{align*}
as claimed.
\end{example}
We also remark that regular sets behave well under unions: if $S_0\subseteq \R^d$ is $(c_0,r_0)$-regular and $S_1\subseteq \R^d$ is $(c_1,r_1)$-regular, then $S_0 \cup S_1$ is $(c_0 \wedge c_1,r_0 \wedge r_1)$-regular.

\begin{lemma}\label{regularSuperLevelSetsLemma} Suppose that $\mu$ is a probability measure on $\R^d$ with density $f$ with respect to~$\Lebesgue$. Given any $\epsilon>0$ we let $S_{\epsilon}:=\{x \in \R^d: f(x) \geq \epsilon\}$. Suppose that there exist $c_0 \in (0,1]$ and $\epsilon_0>0$, such that for every $\epsilon \in (0,\epsilon_0]$, the super-level set $S_{\epsilon}$ is a $\bigl(c_0,\diam(S_{\epsilon})\bigr)$-regular set. Then 
\begin{align*}
\mu\bigl(\bigl\{ x \in \R^d : \omega_{\mu,d}(x) < \xi \bigr\}\bigr) \leq 2 \cdot \mu\bigl( S_{(c_0 \cdot V_d)^{-1} \cdot \xi}^c\bigr)
\end{align*}
for all $\xi \in \bigl(0,1/2 \wedge (c_0 \cdot V_d) \cdot \epsilon_0\bigr]$.
\end{lemma}
From Lemma~\ref{regularSuperLevelSetsLemma}, we see that whenever $\mu_Q$ has $d_0$-dimensional, regular support, with a density that is bounded away from zero on this support, we may take $d_Q = d_0$ and $\gamma_Q$ to be arbitrarily large in the first part~\eqref{Eq:targetMarginal} of Assumption~\ref{sourceDensityAssumption}.
\begin{proof}
Take $\xi \in \bigl(0,1/2 \wedge (c_0 \cdot V_d) \cdot \epsilon_0\bigr]$ and let $\epsilon= (c_0 \cdot V_d)^{-1} \cdot \xi \leq \epsilon_0$. We may assume that $\mu( S_{\epsilon}^c) \leq 1/2$, because otherwise the statement follows immediately from the fact that $\mu$ is a probability measure. Now take any $x \in S_{\epsilon}$ and any $r \in (0,1)$. If $r > \diam(S_{\epsilon})$, then $S_{\epsilon} \subseteq B_{r}(x)$,
so
\begin{align*}
\mu\bigl(B_r(x)\bigr) \geq \mu\left(S_{\epsilon}\right) \geq \frac{1}{2} \geq \xi \cdot r^d.
\end{align*}
On the other hand, if $r \in (0, \diam(S_{\epsilon})]$ then 
\begin{align*}
\mu\bigl(B_r(x)\bigr) =\int_{B_r(x)} f \geq \int_{B_r(x)\cap S_{\epsilon}} f  &\geq \epsilon \cdot \Lebesgue\bigl(B_r(x)\cap S_{\epsilon}\bigr) \geq \epsilon \cdot c_0 \cdot \Lebesgue\bigl(B_r(x)\bigr)\\
& = \epsilon \cdot c_0 \cdot V_d \cdot r^d = \xi \cdot r^d.
\end{align*}
Hence, in general, for $x \in S_{\epsilon}$ we have $\omega_{\mu,d}(x) \geq \xi$, so the result follows. \end{proof}
\begin{example}
\label{Ex:gammafamily}
To illustrate how Lemma~\ref{regularSuperLevelSetsLemma} may be applied to verify~\eqref{Eq:targetMarginal} in Assumption~\ref{sourceDensityAssumption}, consider the family of univariate densities $\{f_\gamma : \gamma > 0\}$ given by 
\begin{align*}
  f_{\gamma}(x) := \begin{cases} \gamma \bigl\{1 + (1-\gamma)x\bigr\}^{-1/(1-\gamma)}\one_{\{x \geq 0\}} &\text{ if }\gamma < 1 \\
e^{-x}\mathbbm{1}_{\{x \geq 0\}} &\text{ if } \gamma=1\\
\gamma \bigl\{1 - (\gamma-1)x\bigr\}^{1/(\gamma-1)}\one_{\{0 \leq x \leq 1/(\gamma-1)\}} &\text{ if }\gamma>1.
\end{cases}
\end{align*}
Writing $\mu_{\gamma}$ for the probability measure with density $f_{\gamma}$, we claim that for each $\gamma>0$, we have
\begin{align}\label{claimForConcreteGammaFamily}
\mu_{\gamma}\bigl( \bigl\{ x \in \R: \omega_{\mu_\gamma,1}(x)<\xi\bigr\} \bigr)\leq \biggl(\frac{2}{\gamma^\gamma}\vee 2^\gamma\biggr) \cdot \xi^{\gamma}
\end{align}
for all $\xi >0$. To see this, let
\[
  x_\gamma := \left\{ \begin{array}{ll} \frac{(\gamma/\xi)^{1-\gamma}-1}{1-\gamma} \vee 0 & \quad \mbox{if $\gamma < 1$} \\
                                                                                                     \log(1/\xi)\vee 0 & \quad \mbox{if $\gamma=1$} \\
                        \frac{1- (\xi/\gamma)^{\gamma-1}}{\gamma-1} \vee 0 & \quad \mbox{if $\gamma>1$.}
                      \end{array} \right.
\]
Then, for each $\gamma, \xi>0$, we have
\begin{align*}
  \mu_{\gamma}\bigl( \bigl\{ x \in \R: f_{\gamma}(x)<\xi\bigr\} \bigr)= \int_{x_\gamma}^\infty f_\gamma = (\xi/\gamma)^\gamma \wedge 1.                 
\end{align*}
Moreover, in each case, the super level set $S_{\epsilon}:=\{x \in \R: f_{\gamma}(x) \geq \epsilon\}$ is a compact interval for every $\epsilon > 0$, so is $\bigl(1/2,\diam(S_{\epsilon})\bigr)$-regular.  We deduce from Lemma~\ref{regularSuperLevelSetsLemma} that
\[
\mu_{\gamma}\bigl( \bigl\{ x \in \R: \omega_{\mu_\gamma,1}(x)<\xi\bigr\} \bigr)\leq \frac{2}{\gamma^\gamma} \cdot \xi^{\gamma}
\]
for all $\xi \in (0,1/2]$. Moreover, for $\xi > 1/2$, we have
\[
\mu_{\gamma}\bigl( \bigl\{ x \in \R: \omega_{\mu_\gamma,1}(x)<\xi\bigr\} \bigr)\leq 1 \leq 2^\gamma \cdot \xi^{\gamma},
\]
which establishes the claim~\eqref{claimForConcreteGammaFamily}.
\end{example}

Returning to more general settings, we recall \citep{gadat2016classification} that, given $c_0,r_0 > 0$, a probability measure $\mu$ on $\R^d$ that is absolutely continuous with respect to~$\Lebesgue$, having density $f$, is said to satisfy the $(c_0,r_0)$-\emph{strong minimal mass assumption} if $\mu\bigl(B_r(x)\bigr) \geq c_0 \cdot f(x) \cdot r^d$ for every $r \in (0,r_0]$ and $\Lebesgue$-almost every $x \in \R^d$.
\begin{lemma}
\label{Lemma:SMM}
Let $\mu$ be a probability measure on $\R^d$ that is absolutely continuous with respect to~$\Lebesgue$, having density $f$, and that satisfies the $(c_0,r_0)$-strong minimal mass assumption.  Then $\omega_{\mu,d}(\cdot) \geq c_0 \cdot (r_0^d \wedge 1) \cdot f(\cdot)$.
\end{lemma}
\begin{proof}
This follows immediately by considering the two cases $r_0 < 1$ and $r_0 \geq 1$.
\end{proof}
We now move on to consider distributions on a general metric space $(\X,\rho)$, noting that the definition of $\omega_{\mu,d} \equiv \omega_{\mu,d}^\rho$ remains well-defined, provided that we reinterpret $B_r(x)$ as open balls in the metric space.  For a Borel measure $\nu$ on a separable metric space $(\X,\distX)$, and for~$s \in (0,\infty)$, we define
\begin{align*}
\Theta_s(\nu):= \sup\biggl\{ \frac{\nu\bigl(B_R(x)\bigr)}{\nu\bigl(B_r(x)\bigr)} \cdot \Bigl(\frac{r}{R}\Bigr)^s: x \in \supp(\nu) ,\hspace{2mm}r, R \in (0,\infty)\hspace{2mm}\text{ with }0<r\leq R \biggr\}.
\end{align*}
Moreover, in a slight abuse of notation, we let $\Theta_s(\X):= \inf\bigl\{ \Theta_s(\nu): \supp\left(\nu\right) = \X\bigr\}$, with the convention that $\inf \emptyset := \infty$.

\begin{example}
\label{Ex:Short}
  We have $\Theta_d(\mathcal{L}_d)=1$, so $\Theta_d(\R^d)\leq 1$. 
\end{example}

We shall make use of Vitali's covering lemma.

\begin{theorem}[Vitali's covering lemma: \citet{evans2015measure}, Theorem~1]\label{vitalisCoveringLemma} Let $\mathcal{X}$ be a separable metric space, let $\mathcal{X}_0 \subseteq \mathcal{X}$, and suppose that $\{r(x):x \in \X_0\}$ is a bounded collection of positive real numbers. Then there exists a countable set $\mathcal{X}_1 \subseteq \mathcal{X}_0$ such that $\{B_{r(x)}(x):x \in \mathcal{X}_1\}$ are disjoint, and satisfy
\begin{align*}
\bigcup_{x \in \mathcal{X}_0} B_{r(x)}(x) \subseteq \bigcup_{x \in \mathcal{X}_1} B_{5r(x)}(x).
\end{align*}
\end{theorem}
Lemma~\ref{boundedSetTailAssumptionSepMetricSpace} below is a general result that reveals in particular that if $\mu_Q$ is supported on a bounded, $d_0$-dimensional subset of~$\mathbb{R}^d$, then we may take $d_Q = d_0$ and $\gamma_Q = 1$ in the first part~\eqref{Eq:targetMarginal} of Assumption~\ref{sourceDensityAssumption}.  We will also apply Lemma~\ref{boundedSetTailAssumptionSepMetricSpace} when we study distributions satisfying a moment condition~(Lemma~\ref{momentTailAssumptionSepMetricSpace}) as well as ones satisfying Weibull-type tails~(Lemma~\ref{exponentialTailsLemma}).
\begin{lemma}\label{boundedSetTailAssumptionSepMetricSpace}  Given a Borel probability measure $\mu$ on $(\X,\distX)$, as well as a Borel measurable set $A \subseteq \X$, $s \in (0,\infty)$ and $\xi>0$, we have 
\begin{align*}
\mu\bigl(\bigl\{ x \in A: \omega_{\mu,s}(x)<\xi\bigr\}\bigr) \leq 5^s \cdot \bigl\{\diam(A)+1\bigr\}^s \cdot \Theta_s(\X) \cdot \xi.
\end{align*}
\end{lemma}
\begin{proof}  For $\xi>0$, we let $A_{\xi}:=\bigl\{x \in A: \omega_{\mu,s}(x)<\xi\bigr\}$. By definition, for each $x\in A_{\xi}$, there exists $r(x) \in (0,1)$ such that $\mu\bigl(B_{r(x)}(x)\bigr) < \xi \cdot r(x)^s$. By applying Vitali's covering lemma to the set of balls $\{B_{r(x)/5}(x):x \in A_{\xi}\}$ we see that there is a countable set $\mathcal{J}$ and $\{x_j\}_{j \in \J}\subseteq A_{\xi}$ such that $\{B_{r(x_j)/5}(x_j):j \in \J\}$ are disjoint and 
\begin{align*}
A_{\xi}\subseteq \bigcup_{x \in A_{\xi}}B_{r(x)/5}(x) \subseteq \bigcup_{j \in \J}B_{r(x_j)}(x_j).
\end{align*}
Since the result is vacuously true if $\Theta_s(\mathcal{X}) = \infty$, we may assume that $\Theta_s(\mathcal{X}) < \infty$.  Fix $\epsilon>0$ and choose a Borel measure $\nu$ on $\X$ with $\supp(\nu)=\X$ and $\Theta_s(\nu) \leq \Theta_s(\X)+\epsilon$.  Define $A^{(1)} := \bigcup_{x \in A}B_1(x)$. Then $A^{(1)} \subseteq B_{\diam(A)+1}(x_j)$ for every $j \in \J$, so
\begin{align*}
\nu(A^{(1)})\leq \nu\left( B_{\diam(A)+1}(x_j)\right) \leq \bigl\{\Theta_s(\X)+\epsilon\bigr\}\cdot \frac{5^s\cdot \bigl\{ \diam(A)+1\bigr\}^s}{r(x_j)^s} \cdot \nu\left(B_{r(x_j)/5}(x_j)\right).
\end{align*}
Putting the above together we have
\begin{align*}
\mu(A_{\xi}) &\leq \mu\biggl(  \bigcup_{j \in \J}B_{r(x_j)}(x_j)\biggr) \leq \sum_{j \in \J} \mu\left(B_{r(x_j)}(x_j)\right) \leq \xi \cdot \sum_{j \in \J}r(x_j)^s\\
& \leq \xi \cdot 5^s \cdot \bigl\{\diam(A)+1\bigr\}^s \cdot \frac{\Theta_s(\X)+\epsilon}{\nu(A^{(1)})}\cdot \sum_{j \in \J}   \nu\bigl(B_{r(x_j)/5}(x_j)\bigr)\\
&= \xi \cdot 5^s \cdot \bigl\{\diam(A)+1\bigr\}^s \cdot \frac{\Theta_s(\X)+\epsilon}{\nu(A^{(1)})}\cdot    \nu\biggl(\bigcup_{j \in \J}B_{r(x_j)/5}(x_j)\biggr)\\
& \leq 5^s \cdot \bigl\{\diam(A)+1\bigr\}^s \cdot (\Theta_s(\X)+\epsilon) \cdot \xi,
\end{align*}
where the final inequality uses $\bigcup_{j \in \J}B_{r(x_j)/5}(x_j)\subseteq A^{(1)}$. Letting $\epsilon \rightarrow 0$ completes the proof of the lemma.
\end{proof}
As mentioned above, we now provide two applications of Lemma~\ref{boundedSetTailAssumptionSepMetricSpace} that reveal feasible choices of $d_Q$ and $\gamma_Q$ in the first part~\eqref{Eq:targetMarginal} of Assumption~\ref{sourceDensityAssumption} under different tail decay assumptions.
\begin{lemma}\label{momentTailAssumptionSepMetricSpace} Let $\mu$ be a Borel probability measure on a separable metric space $(\X,\distX)$. Suppose further that for some $x_0 \in \X$, $\rho>0$ and $M_{\rho} \geq 1$, we have 
\begin{align*}
\mu\bigl( B_t(x_0)^c\bigr) \leq M_{\rho} \cdot t^{-\rho}
\end{align*} for all $t>0$.  Then, for every $s,\xi >0$, we have
\begin{align}
\label{Eq:Conclusion}
\mu\bigl(\bigl\{ x \in \X: \omega_{\mu,s}(x)<\xi\bigr\}\bigr) \leq  \bigl\{15^s \cdot \Theta_s(\X) +1 \bigr\} \cdot M_{\rho}^{\frac{s}{s+\rho}}\cdot \xi^{\frac{\rho}{s+\rho}}. 
\end{align}
In particular, if $\int_{\X} \distX(x,x_0)^\rho \, d\mu(x) \leq M_\rho$, then the conclusion~\eqref{Eq:Conclusion} holds.
\end{lemma}
\begin{remark}
\label{Rem:Moment}
We believe that many if not most applications of this result will be concerned with the setting where $\X = \mathbb{R}^d$, where $\mu$ has a density $f$ with respect to Lebesgue measure, and where $\int_{\R^d} \|x\|^\rho f(x) \, dx \leq M_\rho$.  With $X \sim f$, the desired conclusion, which follows immediately from~Lemma~\ref{momentTailAssumptionSepMetricSpace}, will often be that
\[
\mathbb{P}\bigl\{f(X) < \xi/V_d\bigr\} \leq (15^d +1) \cdot M_{\rho}^{\frac{d}{d+\rho}}\cdot \xi^{\frac{\rho}{d+\rho}}.
\]  
\end{remark}
\begin{proof} 
Since the conclusion is clear when $\xi \geq 1$, we fix $\xi \in (0,1)$.  By Lemma~\ref{boundedSetTailAssumptionSepMetricSpace} and Markov's inequality, for any $R \geq 1$ we have
\begin{align*}
\mu\bigl(\bigl\{ x \in \X: \omega_{\mu,s}(x)<\xi \bigr\}\bigr)&\leq \mu\bigl(\bigl\{ x \in B_R(x_0): \omega_{\mu,s}(x)<\xi\bigr\}\bigr)+\mu\bigl(B_R(x_0)^c\bigr)\\
& \leq   \Theta_s(\X)\cdot \left({15 \cdot R}\right)^s  \cdot \xi+M_{\rho} \cdot R^{-\rho}.
\end{align*}
Taking $R = \left(M_{\rho} / \xi \right)^{\frac{1}{s+\rho}} \geq 1$ completes the proof of the first statement of the lemma.  The second follows from the first, together with Markov's inequality.
\end{proof}
\begin{lemma}\label{exponentialTailsLemma} Let $A,a,q > 0$ and suppose that $\mu$ is a Borel probability measure on a separable metric space $(\X,\distX)$ with $\mu\bigl(B_t(x_0)^c\bigr) \leq A e^{-at^q}$ for all $t\geq 0$ and some $x_0 \in \X$. Then for all $s,\xi >0$ we have
\begin{align*}
  \mu\bigl(\{x \in \mathcal{X}: \omega_{\mu,s}(x)<\xi\}\bigr) \leq \biggl[5^s \cdot \biggl\{ \frac{2}{a^{1/q}} \cdot \log_+^{1/q}\biggl(\frac{A}{\xi}\biggr)+1\biggr\}^s \cdot \Theta_s(\X) + 1\biggr] \cdot \xi.
\end{align*}
\end{lemma}
\begin{proof} Let $t:= a^{-1/q}\log_+^{1/q}(A/\xi)$.  We apply Lemma S4 with $A=B_t(x_0)$ as follows:
\begin{align*}
\mu\bigl(\{x \in \mathcal{X}: \omega_{\mu,s}(x)<\xi\}\bigr) &\leq \mu\bigl(\{x \in B_t(x_0): \omega_{\mu,s}(x)<\xi\}\bigr)+\mu\bigl(B_t(x_0)^c\bigr) \\
& \leq 5^s \cdot (2t+1)^s\cdot \Theta_s(\X)\cdot \xi+A e^{-at^q}\\
& \leq 5^s \cdot \biggl\{ \frac{2}{a^{1/q}} \cdot \log_+^{1/q}\biggl(\frac{A}{\xi}\biggr)+1\biggr\}^s \cdot \Theta_s(\X) \cdot \xi+ \xi,
\end{align*}
as required.
\end{proof}
Our final results in a general metric space setting concern mixtures and products:
\begin{proposition}\label{lemma:mixtureTails} Let $\X$ be a metric space, and let $\mu_1,\ldots,\mu_J$ be Borel probability measures on $\X$.  Further, let $\bar{\mu} = \sum_{j =1}^Jp_j \mu_{j}$, where $p_j \in (0,1]$ with $\sum_{j=1}^Jp_j=1$.  Then for all $s,\xi>0$, we have
\begin{align*}
\bar{\mu}\Bigl(\bigl\{ x \in \X: \omega_{\bar{\mu},s}(x)\leq \xi \bigr\}\Bigr) \leq  \sum_{j=1}^J p_j \cdot {\mu}_j\biggl(\biggl\{ x \in \X: \omega_{{\mu}_j,s}(x)\leq \frac{\xi}{p_j} \biggr\}\biggr).
\end{align*}
\end{proposition}
\begin{proof} First note that for all $x \in \X$, we have
\begin{align*}
\omega_{\bar{\mu},s}(x)&=\inf_{r \in (0,1)} \sum_{j=1}^Jp_j\cdot \frac{\mu_j\bigl(B_r(x)\bigr)}{r^{s}} \\ & \geq  \sum_{j=1}^Jp_j\cdot\inf_{r \in (0,1)} \frac{\mu_j\left(B_r(x)\right)}{r^{s}} \geq \max_{j\in [J]} \left\lbrace p_j \cdot \omega_{{\mu}_j,s}(x)\right\rbrace. 
\end{align*}
Hence, we have
\begin{align*}
\bar{\mu}\Bigl(\bigl\{ x \in \X: \omega_{\bar{\mu},s}(x)\leq \xi \bigr\}\Bigr) &= \sum_{j=1}^Jp_j \cdot \mu_j\Bigl(\bigl\{ x \in \X: \omega_{\bar{\mu},s}(x)\leq \xi \bigr\}\Bigr)\\
 &\leq \sum_{j=1}^Jp_j \cdot \mu_j\biggl(\biggl\{ x \in \X: \omega_{{\mu}_j,s}(x)\leq \frac{\xi}{p_j} \biggr\}\biggr),
\end{align*}
as required.
\end{proof}
\begin{proposition}
\label{Prop:Products}
Let $(\X_1,\dist_1),\ldots,(\X_J,\dist_J)$ be metric spaces, and for $j \in [J]$, let $\mu_j$ and $\nu_j$ be Borel probability measures on $\X_j$.  Let $\X :=\times_{j=1}^J\X_j$, and suppose that $\dist$ is a metric on $\X$ with
  \[
   \dist(x,\tilde{x})\leq \mathrm{C_{met}} \cdot \max_{j=1 \in [J]}\dist_j(x_j,\tilde{x}_j)  
\]
for all $x=(x_j)_{j=1}^J$, $\tilde{x}=(\tilde{x}_j)_{j=1}^J \in \X$ and some $\mathrm{C_{met}} \geq 1$.  Let $\mu := \times_{j=1}^J \mu_j$ and $\nu := \times_{j=1}^J \nu_j$.  For each $j \in [J]$, suppose that $d_j \geq 0$, $\gamma_j>0$ and $\mathrm{C}_j \geq 1$ are such that for every $\xi>0$ we have
\begin{align*}
\nu_j\bigl(\bigl\{ x_j \in \X_j:  \omega_{\mu_j,d_j}^{\dist_j}(x_j)<\xi              \bigr\}\bigr) \leq \mathrm{C}_j \cdot \xi^{\gamma_j}.
\end{align*}
Then letting $\gamma_{\min} :=\min_{j \in [J]} \gamma_j$, $d := \sum_{j=1}^Jd_j$ and $J_0=|\{j \in [J]:\gamma_j= \gamma_{\min}\}|$, there exists $C\geq 1$, depending only on $\gamma_1,\ldots,\gamma_J$, $\mathrm{C}_1,\ldots,\mathrm{C}_J$ and $\mathrm{C_{met}}$ such that for all $\xi>0$,
\begin{align*}
\nu\bigl(\bigl\{ x \in \X:  \omega_{\mu,d}^{\dist}(x)<\xi              \bigr\}\bigr) \leq C  \cdot \log_+^{J_0-1}(1/\xi) \cdot \xi^{\gamma_{\min}}.
\end{align*}
\end{proposition}
\begin{proof} Without loss of generality, we may assume that $\gamma_1 \leq \gamma_2 \leq \ldots \leq \gamma_J$.  By induction, it suffices to consider the case $J=2$ and show that if, for some $\gamma_1 \leq \gamma_2$ and $K_0 \geq 0$, we have
\begin{align*}
\nu_1&\bigl(\bigl\{ x_1 \in \X_1:  \omega_{\mu_1,d_1}^{\dist_1}(x_1)<\xi              \bigr\}\bigr) \leq \mathrm{C}_1 \cdot \log_+^{K_0}(1/\xi) \cdot \xi^{\gamma_1}\\
\nu_2&\bigl(\bigl\{ x_2 \in \X_2:  \omega_{\mu_2,d_2}^{\dist_2}(x_2)<\xi              \bigr\}\bigr) \leq \mathrm{C}_2 \cdot \xi^{\gamma_2},
\end{align*}
for all $\xi>0$, then there exists $\tilde{C} \geq 1$ depending only upon  $\gamma_1,\gamma_2$, $\mathrm{C}_1,\mathrm{C}_2$ and $\mathrm{C_{met}}$ such that
\begin{align}\label{claimInProofOfProductsTailProp}
\nu\bigl(\bigl\{ x \in {\X}:  \omega_{{\mu},{d}}^{{\dist}}(x)<\xi              \bigr\}\bigr) \leq \left\{ \begin{array}{ll} \tilde{C}  \cdot \log_+^{K_0+1}(1/\xi) \cdot \xi^{\gamma_1} & \quad \mbox{if $\gamma_1=\gamma_2$} \\
\tilde{C}  \cdot \log_+^{K_0}(1/\xi) \cdot \xi^{\gamma_1} & \quad \mbox{if $\gamma_1<\gamma_2$}, \end{array} \right.
\end{align}
for all $\xi>0$. To prove the claim (\ref{claimInProofOfProductsTailProp}) we first observe that given $x=(x_1,x_2) \in \X$ and $r > 0$, we have
\begin{align*}
B_{r/\mathrm{C_{met}}}(x_1)\times B_{r/\mathrm{C_{met}} }(x_2) \subseteq B_r(x).
\end{align*}
Hence, for each $x=(x_1,x_2)\in \X$ we have
\begin{align*}
\omega^{{\dist}}_{\mu,d}(x) &= \inf_{r\in (0,1)}\frac{\mu\bigl(B_r(x)\bigr)}{r^{d}} = \inf_{r\in (0,1)}\frac{(\mu_1\times \mu_2)\bigl(B_r(x)\bigr)}{r^{d_1+d_2}}\\
& \geq \frac{1}{ \mathrm{C}_{\mathrm{met}}^{d}}\cdot \inf_{r\in (0,1)} \biggl\{ \frac{\mu_1\bigl(B_{r/\mathrm{C_{met}}}(x_1)\bigr)}{(r/\mathrm{C}_{\mathrm{met}})^{d_1}} \cdot \frac{\mu_2\bigl(B_{r/\mathrm{C_{met}}}(x_2)\bigr)}{(r/\mathrm{C}_{\mathrm{met}})^{d_2}} \biggr\} \geq \frac{\omega^{{\dist}_1}_{{\mu}_1,{d}_1}(x_1) \cdot \omega^{{\dist}_2}_{{\mu}_2,{d}_2}(x_2)}{ \mathrm{C}_{\mathrm{met}}^{d}} .
\end{align*}
Using the fact that $\omega^{{\dist}_1}_{{\mu}_1,{d}_1}(x_1)\leq 1$, it follows that for $\xi>0$, 
\begin{align*}
\bigl\{& x \in \X: \omega^{{\dist}}_{\mu,d}(x)<\xi \bigr\} \\ &\subseteq  \bigl\{ x \in \X: \omega^{{\dist}_1}_{{\mu}_1,{d}_1}(x_1)< \xi \bigr\} \cup  \bigcup_{\ell=0}^{\infty}  \bigl\{ x \in \X: e^{\ell}\cdot \xi \leq \omega^{{\dist}_1}_{{\mu}_1,{d}_1}(x_1)< e^{\ell+1}\cdot \xi\text{ and } \omega^{{\dist}}_{\mu,d}(x)<\xi \bigr\}\\ &\subseteq  \{ x \in \X: \omega^{{\dist}_1}_{{\mu}_1,{d}_1}(x_1)< \xi \} \cup \\
&\hspace{0.5cm}\bigcup_{\ell=0}^{\lfloor \log_+(1/\xi)\rfloor+1}  \{ x \in \X: e^{\ell}\cdot \xi \leq \omega^{{\dist}_1}_{{\mu}_1,{d}_1}(x_1)< e^{\ell+1}\cdot \xi\text{ and } \omega^{{\dist}_2}_{{\mu}_2,{d}_2}(x_2)< \mathrm{C}_{\mathrm{met}}^{d} \cdot e^{-\ell}  \}\\
& \subseteq \bigl(\bigl\{x_1 \in \X_1: \omega^{{\dist}_1}_{{\mu}_1,{d}_1}(x_1)< \xi\bigr\} \times \X_2\bigr) \hspace{2mm} \cup \\
&\hspace{0.5cm}\bigcup_{\ell=0}^{\lfloor \log_+(1/\xi)\rfloor+1} \! \! \! \bigl(\bigl\{ x_1 \in \X_1:  \omega^{{\dist}_1}_{{\mu}_1,{d}_1}(x_1)< e^{\ell+1}\cdot \xi\bigr\}\times \bigl\{ x_2 \in \X_2:  \omega^{{\dist}_2}_{{\mu}_2,{d}_2}(x_2)< \mathrm{C}_{\mathrm{met}}^{d} \cdot e^{-\ell}  \bigr\}\bigr).
\end{align*}
We deduce that 
\begin{align*}
\nu&\bigl(\bigl\{ x \in \X: \omega^{{\dist}}_{\mu,d}(x)<\xi \bigr\}\bigr)\\
&\leq \nu_1\bigl( \bigl\{x_1 \in \X_1: \omega^{{\dist}_1}_{{\mu}_1,{d}_1}(x_1)< \xi\bigr\}\bigr)+ \\
&\hspace{2mm}\sum_{\ell=0}^{\lfloor \log_+(1/\xi)\rfloor+1} \! \! \! \!\nu_1\bigl( \bigl\{ x_1 \in \X_1:  \omega^{{\dist}_1}_{{\mu}_1,{d}_1}(x_1)< e^{\ell+1}\xi\bigr\}\bigr) \cdot \nu_2\bigl(\bigl\{ x_2 \in \X_2:  \omega^{{\dist}_2}_{{\mu}_2,{d}_2}(x_2)< \mathrm{C}_{\mathrm{met}}^{d}e^{-\ell}  \bigr\}\bigr) \\
& \leq  C_1 \log_+^{K_0}\left(\frac{1}{\xi}\right) \cdot  \xi^{\gamma_1}+ C_1C_2 \! \sum_{\ell=0}^{\lfloor \log_+(1/\xi)\rfloor+1}  \! \log_+^{K_0}\left(\frac{1}{e^{\ell+1} \cdot \xi}\right) \cdot (e^{\ell+1} \cdot \xi)^{\gamma_1} \cdot  \bigl( \mathrm{C}_{\mathrm{met}}^{d} \cdot e^{-\ell} \bigr)^{\gamma_2} \\
& \leq C_1 C_2 \cdot  \log_+^{K_0}\left(\frac{1}{\xi}\right) \cdot  \xi^{\gamma_1}  \biggl( 1+e^{\gamma_1} \mathrm{C}_{\mathrm{met}}^{d\gamma_2} \sum_{\ell=0}^{\lfloor \log_+(1/\xi)\rfloor+1} e^{\ell(\gamma_1-\gamma_2)} \biggr).
\end{align*}
By considering separately the cases $\gamma_1 = \gamma_2$ and $\gamma_1 < \gamma_2$, we conclude that~\eqref{claimInProofOfProductsTailProp} holds with
\[
  \tilde{C} = C_1C_2\max\biggl\{1 + 2e^{\gamma_1}\mathrm{C}_{\mathrm{met}}^{d\gamma_2} \, , \, 1 + \frac{e^{\gamma_2}\mathrm{C}_{\mathrm{met}}^{d\gamma_2}}{e^{\gamma_2-\gamma_1}-1}\biggr\},
\]
as required.
\end{proof}

The following lemma essentially shows that Proposition~\ref{Prop:Products} cannot be improved beyond logarithmic factors.
\begin{lemma} Let $\mu_1$ be a Borel probability measure on $\R^{d_1}$ and $\mu_2$ be a Borel probability measure on $\R^{d_2}$. Suppose that there exists an open set $U \subseteq \R^{d_2}$ with $\mu_2(U)>0$ such that the restriction of $\mu_2$ to $U$ is absolutely continuous with respect to $\mathcal{L}_{d_2}$. Let $d:=d_1+d_2$ and let $\mu=\mu_1\times \mu_2$ denote the product measure on $\R^{d}$. Then there exist $c_A, c_B>0$, depending only on $\mu_2$, such that for any $\xi>0$,
\begin{align*}
\mu\bigl( \bigl\{ x \in \R^d: \omega_{\mu,d}(x)<\xi\bigr\} \bigr) \geq c_A \cdot \mu_1\bigl( \bigl\{ x_1 \in \R^{d_1}: \omega_{{\mu}_1,d_1}(x_1)< c_B\cdot \xi \bigr\}\bigr).
\end{align*}
\end{lemma}
\begin{proof}
By the Radon--Nikodym theorem, there exists a Lebesgue integrable function $f :\R^{d_2} \rightarrow [0,\infty)$ such that $\mu_2(A)= \int_A f$ for all Borel sets $A \subseteq U$. By the Lebesgue differentiation theorem \citep[e.g.,][Theorem 7.7]{rudin2006real} there exists a Borel set $U_0\subseteq U$ with $\mu_2(U_0)=\mu_2(U)>0$ such that for all $z \in U_0$ we have
\begin{align}\label{lebDiffConsequence}
\lim_{r \rightarrow 0}\frac{\int_{B_r(z)}f}{r^{d_2}}= V_{d_2}\cdot \lim_{r \rightarrow 0}\frac{\int_{B_r(z)}f}{\mathcal{L}_{d_2}\bigl(B_r(z)\bigr)}= V_{d_2} \cdot f(z).
\end{align}
We define an increasing family of Borel subsets $(A_{m})_{m \in \N}$ of $U_0$ by
\begin{align*}
A_{m}:=\biggl\{ z \in U_0:\hspace{2mm}B_{2^{-m}}(z) \subseteq U \hspace{2mm}\text{ and }\hspace{2mm} \int_{B_r(z)}f <m\cdot r^{d_2}\hspace{2mm}\text{ for all } r \leq 2^{-m}\biggr\},
\end{align*}
and claim that $U_0 = \bigcup_{m=1}^\infty A_m$. Indeed, given $z \in U_0$ there must exist $m_1 \in \N$ such that $B_{2^{-m_1}}(z) \subseteq U$ since $U$ is open. In addition, by \eqref{lebDiffConsequence} there exists $m_2 \in \N$ such that $\int_{B_r(z)}f < 2V_d\cdot f(z) \cdot {r^{d_2}}$ for all $r \leq 2^{-m_2}$. Hence, we have $z \in A_m$ with $m= m_1 \vee m_2 \vee \lceil 2V_d\cdot f(z) \rceil$, which proves the claim.  We may therefore take $m_0 \in \N$ with $\mu_2(A_{m_0})\geq \mu_2(U_0)/2 >0$ and set $c_A := \mu_2(A_{m_0})$. 

To complete the proof it suffices to show that with $c_B:=m_0^{-1}\cdot 2^{-m_0d_1}$ we have
\begin{align*}
\bigl\{ x_1 \in \R^{d_1}: \omega_{{\mu}_1,d_1}(x_1)< c_B\cdot \xi \bigr\} \times A_{m_0} \subseteq  \bigl\{ x \in \R^d: \omega_{\mu,d}(x)<\xi\bigr\}.
\end{align*}
Indeed, given $x_1 \in \R^{d_1}$ with $\omega_{{\mu}_1,d_1}(x_1)< c_B\cdot \xi$, there exists $r_0 \equiv r_0(x_1) \in (0,1)$ such that $\mu_1\bigl(B_{r_0}(x_1)\bigr)<c_B \cdot \xi \cdot r_0^{d_1}$. In addition, taking $x_2 \in A_{m_0}$ and $r_1 \equiv r_1(x_1) := 2^{-m_0}\wedge r_0$, we have $B_{r_1}(x_2) \subseteq B_{2^{-m_0}}(x_2) \subseteq U$ so $\mu_2\bigl(B_{r_1}(x_2)\bigr) = \int_{B_{r_1}(x_2)}f <m_0\cdot r_1^{d_2}$. Hence, letting $x=(x_1,x_2)$, we have $B_{r_1}(x) \subseteq  B_{r_1}(x_1) \times B_{r_1}(x_2)$, and so 
\begin{align*}
\omega_{\mu,d}(x) \leq \frac{\mu_1\bigl(B_{r_1}(x_1)\bigr)}{r_1^{d_1}}\cdot \frac{\mu_2\bigl(B_{r_1}(x_2)\bigr)}{r_1^{d_2}}\leq 2^{m_0d_1} \cdot \frac{\mu_1\bigl(B_{r_0}(x_1)\bigr)}{r_0^{d_1}} \cdot m_0< \xi,
\end{align*}
as required.
\end{proof}

We now return to the Euclidean setting.  Recall that a probability measure $\mu$ on $\R^d$ is \emph{log-concave} if $\mu\bigl(\lambda \cdot A + (1-\lambda) \cdot B\bigr) \geq \mu(A)^{\lambda}\mu(B)^{1-\lambda}$ for all Borel sets $A, B \subseteq \R^d$ and $\lambda \in [0,1]$.  Recall further that when $\mu$ has $d$-dimensional support, it is log-concave if and only if it has a log-concave density \citep[e.g.,][Theorem~2.8]{dharmadhikari1988unimodality}.  Our first result shows that $\omega_{\mu,d}$ inherits log-concavity from $\mu$.
\begin{lemma}
\label{Lemma:InheritLC}
If $\mu$ is a log-concave measure on a $d_0$-dimensional subset of $\mathbb{R}^d$, then $\omega_{\mu,d_0}$ is log-concave.
\end{lemma}
\begin{proof}
  If $d_0 = 0$, then $\omega_{\mu,0}(x) = \mathbbm{1}_{\{x = 0\}}$, which is log-concave, so we may assume that $d_0 \geq 1$.  Now consider the case $d_0 = d$.  Let $f$ denote the (log-concave) density on $\R^d$ of $\mu$ with respect to $\mathcal{L}_d$.  For each $r \in (0,1)$, the function 
\[
x \mapsto \frac{\mu\bigl(B_r(x)\bigr)}{r^d} = \frac{1}{r^d}\int_{\mathbb{R}^d} f(y)\mathbbm{1}_{\{y-x \in B_r(0)\}} \, dy = \frac{1}{r^{d_0}} (f \ast \mathbbm{1}_{B_r(0)})(x)
\]
is log-concave on $\R^d$, because the convolution of two log-concave functions is log-concave; see, e.g.,~\citet{prekopa1973contributions,prekopa1980logarithmic} or \citet[][Corollary~2.4]{samworth2018recent}.  Since the infimum of a collection of log-concave functions is also log-concave, we deduce that $\omega_{\mu,d}$ is log-concave on~$\R^d$.  If $1 \leq d_0 < d$, then since log-concavity is preserved under affine transformations, we may assume without loss of generality that $\supp(\mu) = \bigl\{(x_1,\ldots,x_{d_0},0,\ldots,0) \in \R^d:x_j \in \R \text{ for all } j \in [d_0]\bigr\}$.  Moreover, we may then define a log-concave measure $\tilde{\mu}$ on $\R^{d_0}$ by $\tilde{\mu}(B) := \mu\bigl(B \times \{0\} \times \ldots \times \{0\}\bigr)$ for Borel subsets $B$ of $\R^{d_0}$.  The function $\omega_{\tilde{\mu},d_0}$ is log-concave by the argument above, and
\[
  \omega_{\mu,d_0}(x_1,\ldots,x_d) = \left\{ \begin{array}{ll} \omega_{\tilde{\mu},d_0}(x_1,\ldots,x_{d_0}) & \quad \mbox{if $x_{d_0+1} = \cdots = x_d = 0$} \\
0 & \quad \mbox{otherwise,}
\end{array} \right.
\]
so $\omega_{\mu,d_0}$ is log-concave.
\end{proof}
For log-concave $\mu$, we can let $\nu := \int_{\mathbb{R}^d} x \, d\mu(x) \in \R^d$ and $\Sigma := \int_{\mathbb{R}^d} (x-\nu)(x-\nu)^\top \, d\mu(x) \in \R^{d \times d}$ denote its mean and covariance matrix respectively; these are both finite by \citet[][Lemma~1]{cule2010theoretical}, and $\Sigma$ is positive definite by \citet[][Lemma~2.1]{dumbgen2011approximation}.

Let us now consider~\eqref{Eq:targetMarginal} of Assumption~\ref{sourceDensityAssumption} in the context of log-concave measures supported on a $d_0$-dimensional subset of $\R^d$.  When $d_0 = 0$, such a log-concave probability measure $\mu$ is a Dirac point mass on some $x_0 \in \mathbb{R}^d$, in which case $\omega_{\mu,0}(x) := \mathbbm{1}_{\{x=x_0\}}$, and $\mu\bigl(\{x \in \mathbb{R}^d: \omega_{\mu,0}(x) \leq \xi\}\bigr) = \mathbbm{1}_{\{\xi \geq 1\}}$, so we may take $\gamma_Q$ in~\eqref{Eq:targetMarginal} to be arbitrarily large for $d_0 = 0$.  Henceforth we will therefore consider $d_0 \in [d]$.  We will treat multivariate log-concave measures in Proposition~\ref{Prop:LCGenerald} below, but it turns out that a slightly sharper bound is available when $d_0 = 1$:
\begin{proposition}
\label{Prop:LCd1}
Let $\mu$ be a log-concave probability measure on $\R^d$ with univariate support and density $f$, let $\sigma^2$ denote the non-zero eigenvalue of its covariance matrix.  Then for all $\xi>0$,
\begin{align*}
\mu\bigl( \bigl\{ x \in \R^d: \omega_{\mu,1}(x)<\xi\bigr\} \bigr) \leq \max(16\sigma,2) \cdot \xi.
\end{align*}
\end{proposition}
\begin{remark*}
  This proposition tells us that we can take $d_Q = 1$ and $\gamma_Q=1$ in~\eqref{Eq:targetMarginal} of Assumption~\ref{sourceDensityAssumption} whenever $\mu_Q$ is log-concave on $\R^d$ with univariate support.
 \end{remark*}
 \begin{proof} 
First consider the case $d=1$.  We may assume without loss of generality that $f$ is upper semi-continuous.  By \cite[Lemma~5.5(b)]{lovasz2007geometry}, we have $M:= \sup_{x \in \R}f(x) \geq 1/(8\sigma)$.  For $\epsilon>0$, let $S_{\epsilon}:=\{x \in \R :f(x)\geq \epsilon\}$ as in Lemma~\ref{regularSuperLevelSetsLemma}, which is a compact interval.  Then, for all $\epsilon>0$, the set $S_{\epsilon}$ is $\bigl(1/2,\diam(S_{\epsilon})\bigr)$-regular. By applying Lemma~\ref{regularSuperLevelSetsLemma} followed by \citet[Lemma~5.6(a)]{lovasz2007geometry}, we see that for all $\xi \in (0,1/2]$ we have
\begin{align*}
\mu\left(\{x \in \R: \omega_{\mu,1}(x)<\xi\}\right) &\leq 2 \cdot \mu(S_{\xi}^c) \leq \frac{2}{M}\cdot \xi\leq 16\sigma\cdot \xi.
\end{align*}
On the other hand, for $\xi > 1/2$ we have $
\mu\bigl(\{x \in \R: \omega_{\mu,1}(x)<\xi\}\bigr) \leq 1 \leq 2 \cdot \xi$.  The result extends to general $d$ as in the proof of Lemma~\ref{Lemma:InheritLC}.
\end{proof}
\begin{proposition}
  \label{Prop:LCGenerald}
Let $d_0 \in [d]$, and let $\mu$ be a log-concave probability measure supported on a $d_0$-dimensional subset of $\R^d$ with covariance matrix $\Sigma$.  Then for all $\xi >0$, we have 
\begin{align*}
  \mu\bigl(\{x \in \R^d: \omega_{\mu,d_0}(x)<\xi\}\bigr) \leq \biggl[5^{d_0} \cdot \biggl\{ 2\tr^{1/2}(\Sigma) \cdot \log_+\biggl(\frac{e}{\xi}\biggr)+1\biggr\}^{d_0} + 1\biggr] \cdot \xi.
\end{align*}
\end{proposition}
\begin{remark*}
From Proposition~\ref{Prop:LCGenerald} we see that we may take $d_Q = d_0$ and any $\gamma_Q<1$ in~\eqref{Eq:targetMarginal} of Assumption~\ref{sourceDensityAssumption} whenever~$\mu_Q$ is a log-concave probability measure supported on a $d_0$-dimensional subset of $\R^d$, with $d_0 \geq 2$.  A nice aspect of this bound is the fact that it depends on a certain average of the singular values of $\Sigma$, as opposed to the largest of these singular values, or the condition number.
\end{remark*}
\begin{proof}
First consider the case $d_0 = d$.  Without loss of generality, we may assume that $\int_{\R^d} x \, d\mu(x) = 0$.  By \citet[Lemma~5.17]{lovasz2007geometry}, we have $\mu(\{x \in \R^d:\|x\|>t\}) \leq e^{1-t/\tr^{1/2}(\Sigma)}$ for $t > \tr^{1/2}(\Sigma)$.  But for $t \in \bigl[0,\tr^{1/2}(\Sigma)\bigr]$, we have $\mu(\{x \in \R^d:\|x\|>t\}) \leq 1 \leq e^{1-t/\tr^{1/2}(\Sigma)}$, so the bound holds for all $t \geq 0$.  The result therefore follows from Lemma~\ref{exponentialTailsLemma} with $A = e, a = \tr^{-1/2}(\Sigma)$ and $q=1$.

Arguing as in the proof of Lemma~\ref{Lemma:InheritLC}, we can then extend this result to $1 \leq d_0 < d$.
\end{proof}
Our final result in this subsection concerns mixtures of log-concave distributions.
\begin{proposition}
\label{Lemma:MixtureLC}
Fix $d,J \in \mathbb{N}$ and $d_{0,1},\ldots,d_{0,J} \in [d]$.  Suppose that $\mu_1,\ldots,\mu_J$ are log-concave probability measures on $\R^d$, where $\mu_j$ has $d_{0,j}$-dimensional support, and let $\sigma_j^2$ denote the trace of the covariance matrix of $\mu_j$.  Let $p_1,\ldots,p_J \in (0,1]$ satisfy $\sum_{j=1}^J p_j = 1$, and define the mixture distribution $\bar{\mu} := \sum_{j =1}^Jp_j \mu_j$.  Define $\psi:[d] \times (0,\infty)^2 \rightarrow (0,\infty)$ by
\[
\psi(s,\sigma,\xi) := \left\{ \begin{array}{ll} \max(16\sigma,2) & \quad \mbox{for $s=1$} \\ 5^s \cdot \bigl\{ 2\sigma \cdot \log_+(e/\xi)+1\bigr\}^s + 1 & \quad \mbox{for $s \geq 2$.} \end{array} \right.
\]
Then, writing $d_0 := \max_{j \in [d]} d_{0,j}$, we have for every $\xi > 0$ that
\begin{align*}
  \bar{\mu}\Bigl(\bigl\{ x \in \R^d: \omega_{\bar{\mu},d_0}(x)\leq \xi \bigr\}\Bigr) \leq \sum_{j=1}^J \psi(d_{0,j},\sigma_j,\xi/p_j) \cdot \xi.
\end{align*}
\end{proposition}
\begin{remark*}
From Proposition~\ref{Lemma:MixtureLC}, we see that if $d_{0,j} = 1$ for all $j$, then the mixture of log-concave distributions satisfies~\eqref{Eq:targetMarginal} in Assumption~\ref{sourceDensityAssumption} with $d_Q = 1$ and $\gamma_Q = 1$.  If $d_{0,j} > 1$ for some $j$, then the mixture satisfies~\eqref{Eq:targetMarginal} with $d_Q = \max_{j \in [d]} d_{0,j}$ and any $\gamma_Q < 1$.
\end{remark*}
\begin{proof} By Proposition~\ref{lemma:mixtureTails}, the fact that $d_0 \mapsto \omega_{\mu,d_0}(x)$ is increasing, and Propositions~\ref{Prop:LCd1} and~\ref{Prop:LCGenerald}, we have that for every $\xi > 0$,
\begin{align*}
  \bar{\mu}\Bigl(\bigl\{ x \in \R^d: \omega_{\bar{\mu},d_0}(x)\leq \xi \bigr\}\Bigr) &\leq  \sum_{j=1}^Jp_j \cdot \mu_j\biggl(\biggl\{ x \in \R^d: \omega_{{\mu}_j,d_0}(x)\leq \frac{\xi}{p_j} \biggr\}\biggr) \\
                                                                                     &\leq  \sum_{j=1}^Jp_j \cdot \mu_j\biggl(\biggl\{ x \in \R^d: \omega_{{\mu}_j,d_{0,j}}(x)\leq \frac{\xi}{p_j} \biggr\}\biggr) \\
  &\leq \sum_{j=1}^J \psi(d_{0,j},\sigma,\xi/p_j) \cdot \xi,
\end{align*}
as required.
\end{proof}

\subsection{Understanding the second part of Assumption~\ref{sourceDensityAssumption}}

In this subsection, we begin by providing two examples to aid understanding of the second part of Assumption~\ref{sourceDensityAssumption}.  They illustrate an attraction of the condition, in that both examples involve Gaussian measures, so apply to unbounded feature spaces.  We then show that~\eqref{Eq:sourceMarginal} generalises the transfer-exponent assumption of \citet{kpotufe2018marginal}, and complete this subsection by providing an example of compactly-supported measures where Assumption~\ref{sourceDensityAssumption} allows us to obtain faster rates than would be the case if we were to replace it with the transfer-exponent assumption.

\begin{example}
\label{Ex:Gaussianscale}
Let $\mu_Q = N(0,1)$ and $\mu_P = N(0,\sigma^2)$ for some $\sigma > 0$.  Then, the density $\phi_\sigma$ of $\mu_P$ satisfies $\phi_\sigma''(x) \geq 0$ for $|x| \geq \sigma$, so for $|x| \geq \sigma+1$, we have
\[
  \omega_{\mu_P,1}(x) = \inf_{r \in (0,1)} \frac{1}{r} \int_{x-r}^{x+r} \phi_\sigma(y) \, dy = 2\phi_\sigma(x).
\]
Write $\xi_0 := 2\phi_\sigma(\sigma+1)$.  Since $\omega_{\mu_P,1}$ is also log-concave by Lemma~\ref{Lemma:InheritLC}, it follows that for $\xi \in (0,\xi_0)$,
\begin{align*}
  \mu_Q\bigl(\bigl\{x \in \mathbb{R}:\omega_{\mu_P,1}(x) < \xi\bigr\}\bigr) &= \mu_Q\bigl(\bigl\{x \in \mathbb{R}: 2\phi_\sigma(x) < \xi \bigr\}\bigr) = 2\bigl\{1 - \Phi(a)\bigr\},
\end{align*}
where $a := \sqrt{2\sigma^2 \log\bigl(\frac{\sqrt{2}}{\sqrt{\pi\sigma^2} \xi}\bigr)}$.  Thus for $\xi \in (0,\xi_0)$, we have
\[
\mu_Q\bigl(\bigl\{x \in \mathbb{R}:\omega_{\mu_P,1}(x) < \xi\bigr\}\bigr) \leq e^{-a^2/2} = \biggl(\frac{\sqrt{\pi\sigma^2}}{\sqrt{2}} \cdot \xi\biggr)^{\sigma^2}. 
\]
On the other hand, for $\xi \geq \xi_0$, we have
\[
  \mu_Q\bigl(\bigl\{x:\omega_{\mu_P,1}(x) < \xi\bigr\}\bigr) \leq 1 \leq \frac{\xi^{\sigma^2}}{\xi_0^{\sigma^2}}.
\]
We deduce that we may take any $\gamma_P = \sigma^2$ in the second part of Assumption~\ref{sourceDensityAssumption}.
\end{example}
\begin{example}
\label{Ex:GaussianLocation}
Now suppose that $\mu_Q = N(0,1)$ and $\mu_P = N(a,1)$ for some $a > 0$.  Then, since the standard normal density $\phi$ satisfies $\phi''(x) \geq 0$ for $|x| \geq 1$, we have for $|x-a| \geq 2$ that 
\[
 \omega_{\mu_P,1}(x) = \inf_{r \in (0,1)} \frac{1}{r} \int_{x-r}^{x+r} \phi(y-a) \, dy = 2\phi(x-a). 
\] 
Write $\xi_0 := 2\phi(2)$.  It follows that for $\xi \in (0,\xi_0)$,
\begin{align*}
  \mu_Q\bigl(\bigl\{x:\omega_{\mu_P,1}(x) < \xi\bigr\}\bigr) &\leq \mu_Q\bigl(\bigl\{x:2\phi(x-a) < \xi\bigr\}\bigr) = 1 - \Phi(a+b) + 1 - \Phi(b-a),
\end{align*}
where $b := \sqrt{2\log\bigl(\frac{\sqrt{2}}{\xi\sqrt{\pi}}\bigr)}$.  Thus for $\xi \leq \min\bigl\{\xi_0,2\phi(a)\bigr\} =: \xi_1$, we have
\begin{align*}
  \mu_Q\bigl(\bigl\{x:\omega_{\mu_P,1}(x) < \xi\bigr\}\bigr) &\leq \frac{1}{2}e^{-(b+a)^2/2} + \frac{1}{2}e^{-(b-a)^2/2} \leq e^{-b^2/2 + ab} \\
                                            &= \frac{\xi\sqrt{\pi}}{\sqrt{2}} \cdot \exp\biggl\{a\sqrt{2\log\biggl(\frac{\sqrt{2}}{\xi\sqrt{\pi}}\biggr)}\biggr\}. 
\end{align*}
On the other hand, if $\xi > \xi_1$, and $\gamma_P < 1$, then 
\[
\mu_Q\bigl(\bigl\{x:\omega_{\mu_P,1}(x) < \xi\bigr\}\bigr) \leq 1 \leq \frac{\xi^{\gamma_P}}{\xi_1^{\gamma_P}}.
\]  
We deduce that we may take any $\gamma_P < 1$ in the second part of Assumption~\ref{sourceDensityAssumption}.
\end{example}
We now relate Assumption~\ref{sourceDensityAssumption} to the existing literature.  We recall from \citet{kpotufe2018marginal} that a pair of distributions $(P,Q)$, each on $\R^d \times \{0,1\}$, is said to have transfer-exponent\footnote{In fact, our definition differs slightly from that of \citet{kpotufe2018marginal}, whose covariates take values in a bounded set $\mathcal{X}$, and who therefore take $r_0 = \mathrm{diam}(\mathcal{X})$.} $\kappa \in [0,\infty]$ if there exist $c_0 \in (0,1]$, $r_0 > 0$ and a Borel subset $A$ of $\R^d$ with $\mu_Q(A) = 1$, such that $\mu_P\bigl(B_r(x)\bigr) \geq c_0 \cdot \mu_Q\bigl(B_r(x)\bigr) \cdot (r/r_0)^\kappa$ for all $x \in A$ and $r \in (0,r_0]$.  Lemma~\ref{Lemma:TransferTransferExponents} below shows the second part of Assumption~\ref{sourceDensityAssumption} generalises the notion of a transfer-exponent, in that if a pair of distributions has transfer-exponent $\kappa$, and if the first part of Assumption~\ref{sourceDensityAssumption} holds with parameters $\gamma_Q$ and $d_Q$, then the second part of Assumption~\ref{sourceDensityAssumption} holds with $\gamma_P = \gamma_Q$ and any $d_P \geq d_Q + \kappa$.
\begin{lemma}
\label{Lemma:TransferTransferExponents}
Let $(P,Q)$ be a pair of distributions, each on $\R^d \times \{0,1\}$, that have transfer-exponent $\kappa$.  If~\eqref{Eq:targetMarginal} holds and $d_P \geq d_Q + \kappa$, then
\[
\mu_Q\bigl(\bigl\{x \in \R^d:\omega_{\mu_P,d_P}(x) < \xi\bigr\}\bigr) \leq C_{P,Q} \cdot \biggl(\frac{r_0^\kappa \vee 1}{c_0(r_0^{d_Q} \wedge 1)}\biggr)^{\gamma_Q} \cdot \xi^{\gamma_Q}
\]
for all $\xi > 0$.
\end{lemma}
\begin{remark}
The conclusion of the lemma tells us that~\eqref{Eq:sourceMarginal} in Assumption~\ref{sourceDensityAssumption} holds with $C_{P,Q}$ there replaced with $C_{P,Q} \cdot \Bigl(\frac{r_0^\kappa \vee 1}{c_0(r_0^{d_Q} \wedge 1)}\Bigr)^{\gamma_Q}$, and with $\gamma_P = \gamma_Q$.
\end{remark}
\begin{proof}
For any $\xi > 0$, we have
\begin{align*}
\mu_Q&\bigl(\bigl\{x\in \R^d:\omega_{\mu_P,d_P}(x) < \xi\bigr\}\bigr) \\
&\leq \mu_Q\biggl(\biggl\{x \in \R^d: \min\biggl(\inf_{r \in (0,r_0 \wedge 1)} \frac{\mu_Q\bigl(B_r(x)\bigr)}{r^{d_Q}} \cdot \frac{c_0 r^{d_Q + \kappa - d_P}}{r_0^\kappa} \, , \, c_0 \mu_Q\bigl(B_{r_0\wedge 1}(x)\bigr)\biggr) < \xi\biggr\}\biggr) \\
&\leq \mu_Q\biggl(\biggl\{x \in \R^d: \min\biggl(\frac{c_0}{r_0^\kappa} \cdot \omega_{\mu_Q,d_Q}(x)  \, , \, c_0 \cdot (r_0^{d_Q} \wedge 1) \cdot\omega_{\mu_Q,d_Q}(x) \biggr) < \xi\biggr\}\biggr) \\
&\leq C_{P,Q} \cdot \biggl(\frac{r_0^\kappa \vee 1}{c_0(r_0^{d_Q} \wedge 1)}\biggr)^{\gamma_Q} \cdot \xi^{\gamma_Q},
\end{align*}
as required.
\end{proof}

\begin{example}
\label{Ex:Prototype}
Suppose that $\targetMarginalDistribution$ is the uniform measure on $[0,1]^d$ and that $\sourceMarginalDistribution$ is the probability measure with density $f_{\sourceDistribution}:\R^d \rightarrow [0,\infty)$ given by
\begin{align*}
f_{\sourceDistribution}(x)=\begin{cases} {A \cdot \|x\|_{\infty}^{\kappa}}&\text{ for } x \in [0,1]^d\\
0&\text{ otherwise},
\end{cases}
\end{align*}
where $A:= 1/\int_{[0,1]^d} \|z\|_{\infty}^{\kappa} \, dz$, and where $\|\cdot\|_\infty$ denotes the $\ell_\infty$-norm on $\R^d$. This example corresponds to \citet[][Example 3]{kpotufe2018marginal} in $d$ dimensions.  We claim that Assumption~\ref{sourceDensityAssumption} holds with $\sourceDimension=\targetDimension=d$, $\sourceTailExponent=d/\kappa$ and arbitrarily large~$\targetTailExponent$.

To demonstrate the condition~\eqref{Eq:targetMarginal}, we first note that for any $x \in \supp(\targetMarginalDistribution)=[0,1]^d$ and any $r \in (0,1)$, the set $B_r(x) \cap [0,1]^d$ contains a hyper-cube of side length at least $r/(2\sqrt{d})$, which yields $\targetMarginalDistribution\bigl(B_r(x)\bigr) \geq \bigl(r/(2\sqrt{d})\bigr)^d$. Thus $\omega_{\targetMarginalDistribution,d}(x) \geq (2\sqrt{d})^{-d}$ for all $x \in \supp(\targetMarginalDistribution)$. It follows that given any $\targetTailExponent>0$, if we take $\TailConstant \geq (2\sqrt{d})^{d \targetTailExponent}$ then condition~\eqref{Eq:targetMarginal} holds with $d_Q=d$.

Next we turn to condition~\eqref{Eq:sourceMarginal}. We begin by showing that with $c :=  A \cdot 2^{-(2d+\kappa)} \cdot d^{-d/2}$ we have $\omega_{\sourceMarginalDistribution,d}(x) \geq c \cdot \|x\|_{\infty}^{\kappa}$ for all $x \in [0,1]^d$.  To this end, first consider the case where $x \in [0,1]^d \setminus [0,1/2]^d$, so that $f_{\sourceDistribution}(x) \geq A \cdot 2^{-\kappa}$. It follows that for any $r \in (0,1)$, the set $B_r(x) \cap \{{z} \in \R^d: f_{\sourceDistribution}({z}) \geq A \cdot 2^{-\kappa}\}$ contains a hyper-cube of side length at least $r/(4\sqrt{d})$, and hence  $\sourceMarginalDistribution\bigl(B_r(x)\bigr) \geq  A \cdot 2^{-\kappa} \cdot \bigl(r/(4\sqrt{d})\bigr)^d$. Thus $\omega_{\sourceMarginalDistribution,d}(x) \geq A \cdot 2^{-(2d+\kappa)} \cdot d^{-d/2} \geq c \cdot \|x\|_{\infty}^{\kappa}$ for all $x \in [0,1]^d \setminus [0,1/2]^d$.

Now suppose that $x \in [0,1/2]^d$.  For any $r \in (0,1)$, the set $B_r(x) \cap \{{z} \in \R^d: f_{\sourceDistribution}(z)\geq f_{\sourceDistribution}(x)\}$ contains a hyper-cube of side length at least $r/(2\sqrt{d})$ and so $\sourceMarginalDistribution\bigl(B_r(x)\bigr) \geq f_{\sourceDistribution}(x) \cdot \bigl(r/(2\sqrt{d})\bigr)^d=\bigl(A /(2\sqrt{d})^d\bigr)\cdot \|x\|_{\infty}^{\kappa}  \cdot r^d$. Hence $\omega_{\sourceMarginalDistribution,d}(x) \geq \bigl(A /(2\sqrt{d})^d\bigr)\cdot \|x\|_{\infty}^{\kappa} \geq c \cdot \|x\|_{\infty}^{\kappa}$ for all $x \in [0,1/2]^d$. 

We deduce that, given any $\xi>0$, 
\begin{align*}
\targetMarginalDistribution\bigl(\bigl\{ x \in \R^d: \omega_{\mu_\sourceDistribution,d}(x)<\xi\bigr\} \bigr) &\leq \Lebesgue(\bigl\{ x \in [0,1]^d: \|x\|_{\infty}< (\xi/c)^{1/\kappa} \bigr\} \bigr)=(\xi/c)^{d/\kappa},
\end{align*}
so~\eqref{Eq:sourceMarginal} holds with $d_P = d$ and $\gamma_P = d/\kappa$, provided that $\TailConstant \geq c^{-d/\kappa}$, and our claim about Assumption~\ref{sourceDensityAssumption} is established.

Now suppose that $\targetRegressionFunction:\R^d \rightarrow [0,1]$ is chosen such that Assumptions~\ref{marginAssumption} and~\ref{holderContinuityAssumption} hold for some $\alpha >0$ and $\beta \in (0,1]$. Suppose also, for simplicity, that $\sourceRegressionFunction=\targetRegressionFunction$ and $\numSource \geq \numTarget$. Then Theorem~\ref{Thm:Main} yields that there exists a data-dependent classifier $\hat{f}$ satisfying
\begin{align}
\E\bigl\{ \mathcal{E}(\hat{f})\bigr\} \leq C_{\theta}\cdot \left(\frac{\log_+(\numSource)}{\numSource}\right)^{\frac{\holderExponent(1+\marginExponent)}{(2\holderExponent+d)+\kappa \cdot (\marginExponent \holderExponent/d)}}.
\end{align}

Viewing this example from the perspective of the conditions in \cite{kpotufe2018marginal}, it can be shown that the transfer exponent is~$\kappa$.  Indeed, the transfer exponent can certainly be no larger, because for all $r > 0$, we have
\begin{align*}
  \sourceMarginalDistribution\bigl(B_r(0)\bigr) \leq \sourceMarginalDistribution\bigl([0,r \wedge 1]^d\bigr) &\leq  \Lebesgue([0,r \wedge 1]^d)\cdot \sup_{x \in [0,r \wedge 1]^d} f_{\sourceDistribution}(x) \\
  &= A(r \wedge 1)^{d+\kappa} \leq A d^{d/2}\cdot r^{\kappa} \cdot \targetMarginalDistribution\bigl(B_r(0)\bigr).
\end{align*}
\citet[Theorem 2]{kpotufe2018marginal} then guarantees that there exists a data-dependent classifier $\tilde{f}$ satisfying
\begin{align*}
\E\bigl\{ \mathcal{E}(\tilde{f})\bigr\} \leq \tilde{C}_\theta \cdot \biggl(\frac{\log_+(\numSource)}{\numSource}\biggr)^{\frac{\holderExponent(1+\marginExponent)}{(2\holderExponent+d)+\kappa}}.
\end{align*}
Hence the bound from Theorem~\ref{Thm:Main}, which relies on Assumption~\ref{sourceDensityAssumption}, gives a faster rate whenever ${\marginExponent \holderExponent}< d$. But ${\marginExponent \holderExponent}\leq{d}$ whenever $\targetRegressionFunction(x_0)=1/2$ for some $x_0 \in (0,1)^d$, by Lemma~\ref{Lemma:TBD} below.
\end{example}

\subsection{Constraint on the margin and H\"older exponents}

\begin{lemma}
\label{Lemma:TBD}
Suppose that a distribution $Q$ on $\R^d \times \{0,1\}$ satisfies Assumptions~\ref{marginAssumption} and~\ref{holderContinuityAssumption}.  If there exists $x_0 \in \R^d$ with $\targetRegressionFunction(x_0) = 1/2$ and $\omega_{\targetMarginalDistribution,\targetDimension}(x_0) > 0$, then $\marginExponent \holderExponent \leq \targetDimension$.
\end{lemma}
\begin{proof}
Fix $\xi < \holderConstant$, and note that if $x \in B_{{(\xi/\holderConstant)}^{1/\holderExponent}}(x_0)$, then
\[
|\targetRegressionFunction(x) - 1/2|  = |\targetRegressionFunction(x) - \targetRegressionFunction(x_0)|\leq \holderConstant \|x -x_0\|^{\holderExponent} \leq \xi.
\]
We deduce that 
\begin{align*}
\marginConstant \cdot \xi^{\marginExponent} \geq \targetMarginalDistribution\bigl(\bigl\{x \in \R^d:\left|\targetRegressionFunction(x)-1/2\right|<\xi\bigr\}\bigr) &\geq \targetMarginalDistribution\bigl(B_{(\xi/\holderConstant)^{1/\holderExponent}}(x_0)\bigr) \\
&\geq \Bigl(\frac{\xi}{\holderConstant}\Bigr)^{\targetDimension/\holderExponent} \omega_{\targetMarginalDistribution,\targetDimension}(x_0),
\end{align*}
and the result follows.
\end{proof}

\section{Empirical results}
\label{Sec:Empirical}

In order to give a preliminary indication of the potential practical benefits of transfer learning and the ATL algorithm, we present the results of a small-scale simulation study.  The settings we considered were as follows: let $d =2$, $\mu_P = \mu_Q = U([0,1]^2)$, and, for $x = (x_1,x_2) \in [0,1]^2$, let $\eta_Q(x) = \{1 + \sin(4\pi x_1)\}/2$. The Bayes risk in this problem is $18.2\%$.  We considered two different transfer learning problems, by specifying two different decision tree partitions and transfer functions:
\begin{enumerate} 
\item Let $L^* = 1$, so that $\mathcal{X}^*_1 = \mathbb{R}^2$, and, for $z \in [0,1]$, let $g_1(z) = (1 + 4z)/5$.  Now set $\eta_P = g_1 \circ \eta_Q$.  In this case, $\Delta = 0$, $\phi = 0.8$, $d_Q = 2$, $\gamma_Q = \infty$, $d_P=2$, $\gamma_P = \infty$, $\alpha = \beta = 1$.
\item Let $L^* = 2$, and let $(\mathcal{X}^*_1,\mathcal{X}^*_2)  = \bigl(\mathbb{R} \times (-\infty,1/2), \mathbb{R} \times [1/2,\infty) \bigr)$. Now, for $z \in [0,1]$, let $g_1(z) = \max(0,z - 1/4)$ and $g_2(z) = \min(z+1/4,1)$.  Finally, set 
\[
\eta_P(x) = \left\{  \begin{array}{ll}
g_1 \bigl(\eta_Q(x)\bigr) &\mbox{for $x \in \mathcal{X}_1^*$} \\ 
g_2 \bigl(\eta_Q(x)\bigr) &\mbox{for $x \in \mathcal{X}_2^*$.} \end{array} \right.
\]
In this case, we have $\Delta = 0$, $\phi = 0.5$, $d_Q = 2$, $\gamma_Q = \infty$, $d_P=2$, $\gamma_P = \infty$, $\alpha = \beta = 1$.
\end{enumerate} 

For each setting and for each of 50 repetitions, we generated $n_Q = 100$ independent target data pairs from $Q$, generated $n_P \in \{0,100,200,500,1000\}$ independent source data pairs from $P$, and a further $n_{\mathrm{test}} = 1000$ independent test pairs from $Q$.  In Table~\ref{tab:results} we present, for our ATL method, the average percentage of the test data pairs that were incorrectly classified, along with the corresponding standard errors.  For comparison, we also present the corresponding errors for the algorithm that pools the source and target data, and then applies our ATL algorithm as if all of the data had come from $Q$.

We mention that, for computational reasons, we ran a Monte Carlo approximation to the ATL algorithm stated in the paper.  Specifically, instead of searching over all decision tree partitions, we only considered those with $L\in \{1,2\}$ leaves (regardless of $L^*$).  For $L=2$, we generated 100 random splits by choosing one of the two axes uniformly at random, and then chose a data point at which to split, again uniformly at random.  

\begin{table}[htbp!] \caption{\label{tab:results}Average error of the ATL algorithm and the pooled data algorithm.  The Bayes error rate is 18.2\%.}
\begin{center}
\begin{tabular}{ l r | c |  c }
& $n_P$ & ATL Error (Std. Error.) & Pooled Error (Std. Error.)\\
\hline 
Setting 1 & 0 & 30.0 (0.6) & NA (NA)\\
&100 & 27.4 (0.5) &  29.1 (0.4) \\
&200 & 25.6 (0.4) &  27.5 (0.5) \\
&500 & 23.4 (0.5) &  25.8 (0.4) \\
&1000 & 22.4 (0.4) &  24.0 (0.3) \\
\hline
Setting 2 & 0 & 30.3 (0.5) &NA  (NA) \\
&100 & 28.5 (0.5) & 29.4 (0.5) \\
&200 & 26.7 (0.5)  & 29.0 (0.5)  \\
&500 & 24.7 (0.4)  & 29.0 (0.5)  \\
&1000 & 24.2 (0.4)  & 27.6 (0.4)  \\  
\end{tabular}
\end{center}
\end{table}
Table~\ref{tab:results} reveals two notable features.  First, the performance of the ATL algorithm improves significantly as $n_P$ increases.  This illustrates the potential of transfer learning.  Second, the ATL algorithm is consistently better than the pooled data algorithm, with an excess risk ratio of around 0.6--0.9.  In other words, there is benefit to be had by attempting to learn the underlying structure as we do in the ATL method; one should not simply ignore the fact that the $P$ data were generated from a different distribution to $Q$.

As an alternative approach to approximating the decision tree partition~$\hat{h}$ in~\eqref{ermOverDecisionTreesChoice} for larger-scale problems, one could consider a greedy strategy that would proceed along similar lines to CART \citep{breiman1984classification}.  In other words, at stage $\ell_0 \in \mathbb{N}$ of our iterative process, where we have $\{\mathcal{X}_1,\ldots,\mathcal{X}_{\ell_0}\} \in \mathbb{T}_{\ell_0}$, we select $\ell \in [\ell_0]$, as well as $j \in [d]$ and $s$ belonging to the set of $j$th coordinates of $\sourceSample$ (or a random subset thereof) so that the refined partition $\{\mathcal{X}_1,\ldots,\mathcal{X}_{\ell-1},\mathcal{X}_\ell \cap H_{j,s},\mathcal{X}_{\ell} \setminus H_{j,s},\mathcal{X}_{\ell+1},\ldots,\mathcal{X}_{\ell_0}\} \in \mathbb{T}_{\ell_0+1}$ minimises the objective in~\eqref{ermOverDecisionTreesChoice} over this restricted class.  We defer detailed exploration of such a method to future work.

\bigskip

\textbf{Acknowledgements}: The research of TIC was supported by Engineering and Physical Sciences Research Council (EPSRC) New Investigator Award EP/V002694/1.  The research of RJS was supported by EPSRC Programme grant EP/N031938/1 and EPSRC Fellowship EP/P031447/1.  The authors are grateful for the constructive feedback from the anonymous reviewers, which helped to improve the paper.

\bibliographystyle{apalike}
\bibliography{mybib}

\end{document}